%% file: OPDL_v11.tex
\documentclass[journal, 10pt, twocolumn]{IEEEtran}

\pdfoutput=1

\usepackage{cite}

\usepackage[pdftex]{graphicx}

\usepackage[cmex10]{amsmath}

\usepackage{algorithmic}

\usepackage{amssymb}

\usepackage{array}

\usepackage{mdwmath}
\usepackage{mdwtab}

\usepackage{eqparbox}

\usepackage{multirow}

\usepackage{fixltx2e}

\usepackage{enumerate}

\usepackage[hyphens]{url}
\usepackage{hyperref}
\usepackage{breakurl}

\usepackage{tabularx}

\usepackage{color}

\usepackage{bm}
\usepackage{bbm}

\usepackage{multibib}
\newcites{Supp}{References}
\bstctlcite[@auxoutSupp]{BSTcontrol2}

\begin{document}
%
\title{Efficient Sum of Outer Products Dictionary Learning (SOUP-DIL) and Its Application to Inverse Problems}

\author{Saiprasad~Ravishankar,~\IEEEmembership{Member,~IEEE,}~Raj~Rao~Nadakuditi,~\IEEEmembership{Member,~IEEE,}\\ and~Jeffrey~A.~Fessler,~\IEEEmembership{Fellow,~IEEE}

\thanks{DOI: 10.1109/TCI.2017.2697206. Copyright (c) 2016 IEEE. Personal use of this material is permitted. However, permission to use this material for any other purposes must be obtained from the IEEE by sending a request to pubs-permissions@ieee.org.}

\thanks{This work was supported in part by the following grants: ONR grant N00014-15-1-2141, DARPA Young Faculty Award  D14AP00086, ARO MURI grants  W911NF-11-1-0391 and 2015-05174-05, NIH grant U01 EB018753, and a UM-SJTU seed grant.}

\thanks{S. Ravishankar, R. R. Nadakuditi,  and J. A. Fessler are with the Department of Electrical Engineering and Computer Science, University of Michigan, Ann Arbor, MI, 48109 USA emails: (ravisha, rajnrao, fessler)@umich.edu.}
}

\maketitle

\begin{abstract}
The sparsity of signals in a transform domain or dictionary has been exploited in applications such as compression, denoising and inverse problems. More recently, data-driven adaptation of synthesis dictionaries has shown promise compared to analytical dictionary models. However, dictionary learning problems are typically non-convex and NP-hard, and the usual alternating minimization approaches for these problems are often computationally expensive, with the computations dominated by the NP-hard synthesis sparse coding step. This paper exploits the ideas that drive algorithms such as K-SVD, and investigates in detail efficient methods for aggregate sparsity penalized dictionary learning by first approximating the data with a sum of sparse rank-one matrices (outer products) and then using a block coordinate descent approach to estimate the unknowns. The resulting block coordinate descent algorithms involve efficient closed-form solutions. Furthermore, we consider the problem of dictionary-blind image reconstruction, and propose novel and efficient algorithms for adaptive image reconstruction using block coordinate descent and sum of outer products methodologies. We provide a convergence study of the algorithms for dictionary learning and dictionary-blind image reconstruction.  Our numerical experiments show the promising performance and speed-ups provided by the proposed methods over previous schemes in sparse data representation and compressed sensing-based image reconstruction.
\end{abstract}


\IEEEpeerreviewmaketitle

\input{main}

\bibliographystyle{./IEEEtran}
\bibliography{./IEEEabrv,./OPDL_v11}


\newpage
\clearpage



{
\twocolumn[
\begin{center}
 \Huge Efficient Sum of Outer Products Dictionary Learning (SOUP-DIL) and Its Application to Inverse Problems: Supplementary Material
\vspace{0.2in}
\end{center}]
}

\input{supplement}

\bibliographystyle{./IEEEtran}
\bibliographySupp{./IEEEabrv,./OPDL_v11}

\end{document}

%% file: main.tex
\section{Introduction} \label{sec1}

The sparsity of natural signals and images in a transform domain or dictionary has been exploited in applications such as compression, denoising, and inverse problems.  
Well-known models for sparsity include the synthesis, analysis \cite{elmiru, Cand201159}, and transform \cite{tfcode, sabres} (or generalized analysis) models. Alternative signal models include the balanced sparse model for tight frames \cite{liuqu2}, where the signal is sparse in a synthesis dictionary and also approximately sparse in the corresponding transform (transpose of the dictionary) domain, with a common sparse representation in both domains. 
These various models have been exploited in inverse problem settings such as in compressed sensing-based magnetic resonance imaging \cite{lustig, Liu7448, liuqu2}.
More recently, the data-driven adaptation of sparse signal models has benefited many applications \cite{elad2, elad3, elad4, irami, Mai, bresai, akd, sabres, doubsp2l, saiwen, Cai201489, zhan33} compared to fixed or analytical models.
This paper focuses on data-driven adaptation of the synthesis model and investigates highly efficient methods with convergence analysis and applications, particularly inverse problems. 
In the following, we first briefly review the topic of synthesis dictionary learning before summarizing the contributions of this work.

\subsection{Dictionary Learning}\label{sec1a} 

The well-known synthesis model approximates a signal $ \mathbf{y}  \in \mathbb{C}^{n} $ by a linear combination of a small subset of atoms or columns of a dictionary $\mathbf{D}  \in \mathbb{C}^{n \times J} $, i.e., $\mathbf{y} \approx \mathbf{D} \mathbf{x}$ with $\mathbf{x} \in \mathbb{C}^{J} $ sparse, i.e., $\left \| \mathbf{x} \right \|_{0}\ll n$.
Here, the $\ell_0$ ``norm" counts the number of non-zero entries in a vector, and we assume $\left \| \mathbf{x} \right \|_{0}$ is much lower than the signal dimension $n$.
Since different signals may be approximated using different subsets of columns in the dictionary $\mathbf{D}$, the synthesis model is also known as a union of subspaces model \cite{vidal2, elhamm}. When $n=J$ and $\mathbf{D}$ is full rank, it is a basis. Else when $J>n$, $\mathbf{D}$ is called an overcomplete dictionary. Because of their richness, overcomplete dictionaries can provide highly sparse (i.e., with few non-zeros) representations of data and are popular.

For a given signal $\mathbf{y}$ and dictionary $\mathbf{D}$, finding a sparse coefficient vector $\mathbf{x}$ involves solving the well-known synthesis \emph{sparse coding} problem. Often this problem is to  minimize $\left \| \mathbf{y}-\mathbf{D}\mathbf{x} \right \|_{2}^{2}$ subject to  $\left \| \mathbf{x} \right \|_{0}\leq s$, where $s$ is a set sparsity level. The synthesis sparse coding problem is NP-hard (Non-deterministic Polynomial-time hard) in general \cite{npb}. Numerous algorithms \cite{pati, mp2, chen2,  befro, Needell2, wei} including greedy and relaxation algorithms have been proposed for such problems. While some of these algorithms are guaranteed to provide the correct solution under certain conditions, these conditions are often restrictive and violated in applications. Moreover, these algorithms typically tend to be computationally expensive for large-scale problems.

More recently, data-driven adaptation of synthesis dictionaries, called dictionary learning, has been investigated \cite{ols, eng, elad, Yagh, Mai}. 
Dictionary learning provides promising results in several applications, including in inverse problems \cite{elad2, elad3, bresai, Kongwang}.
Given a collection of signals $\left \{ \mathbf{y}_{i} \right \}_{i=1}^{N}$ (e.g., patches extracted from some images) that are represented as columns of the matrix $ \mathbf{Y} \in \mathbb{C}^{n \times N} $, the dictionary learning problem is often formulated as follows \cite{elad}:
\begin{align} 
\nonumber (\mathrm{P0})\; & \min_{\mathbf{D},\mathbf{X}}\: \left \| \mathbf{Y}-\mathbf{D}\mathbf{X} \right \|_{F}^{2}\; \: \mathrm{s.t.}\; \:  \left \| \mathbf{x} _{i} \right \|_{0}\leq s\; \: \forall \,  i, \, \left \| \mathbf{d}_{j} \right \|_2 =1 \,\forall \, j.
\end{align}
Here, $\mathbf{d}_{j}$ and $\mathbf{x}_{i}$ denote the columns of the dictionary $\mathbf{D} \in \mathbb{C}^{n \times J}$ and sparse code matrix $\mathbf{X} \in \mathbb{C}^{J \times N}$, respectively, and $s$ denotes the maximum sparsity level (number of non-zeros in representations $\mathbf{x}_i$) allowed for each signal. Constraining the columns of the dictionary to have unit norm eliminates the scaling ambiguity \cite{kar}. Variants of Problem (P0) include replacing the $\ell_0$ ``norm" for sparsity with an $\ell_1$ norm or an alternative sparsity criterion, or enforcing additional properties (e.g., incoherence \cite{barchi1, irami}) for the dictionary $\mathbf{D}$, or solving an online version (where the dictionary is updated sequentially as new signals arrive) of the problem \cite{Mai}.

Algorithms for Problem (P0) or its variants \cite{eng, elad, Yagh, zibul, skret, Mai, ophel4, smith1, sadeg2, segh2, bao1} typically alternate in some form between a \emph{sparse coding step} (updating $\mathbf{X}$), and a \emph{dictionary update step} (updating $\mathbf{D}$). Some of these algorithms (e.g., \cite{elad, smith1, segh2}) also partially update $\mathbf{X}$ in the dictionary update step. A few recent methods update $\mathbf{D}$ and $\mathbf{X}$ jointly in an iterative fashion \cite{directdl11, hawesep22}.
The K-SVD method \cite{elad} has been particularly popular \cite{elad2, elad3, bresai}.
Problem (P0) is highly non-convex and NP-hard, and most dictionary learning approaches lack proven convergence guarantees. Moreover, existing algorithms for (P0) tend to be computationally expensive (particularly alternating-type algorithms), with the computations usually dominated by the sparse coding step.

Some recent works  \cite{spel2b, agra1, arora1, yint3, bao1, agra2} have studied the convergence of (specific) dictionary learning algorithms.
However, these dictionary learning methods have not been demonstrated to be useful in applications such as inverse problems. Bao et al. \cite{bao1} find that their proximal scheme denoises less effectively than K-SVD \cite{elad2}. Many prior works use restrictive assumptions (e.g., noiseless data, etc.) for their convergence results.

Dictionary learning has been demonstrated to be useful in inverse problems such as in tomography \cite{gusup} and magnetic resonance imaging (MRI) \cite{bresai, wangying}.
The goal in inverse problems is to estimate an unknown signal or image from its (typically corrupted) measurements. We consider the following general regularized linear inverse problem:
\begin{equation}\label{reginveq1}
\min_{\mathbf{y} \in \mathbb{C}^{p}}\:\left \| \mathbf{A}\mathbf{y}-\mathbf{z} \right \|_{2}^{2}+ \zeta (\mathbf{y})
\end{equation}
where $\mathbf{y} \in \mathbb{C}^{p}$ is a vectorized version of a signal or image (or volume) to be reconstructed, $\mathbf{z} \in \mathbb{C}^{m}$ denotes the observed measurements, and $\mathbf{A} \in \mathbb{C}^{m \times p}$ is the associated measurement matrix for the application.  For example, in the classical denoising application (assuming i.i.d. gaussian noise), the operator $\mathbf{A}$ is the identity matrix, whereas in inpainting (i.e., missing pixels case), $\mathbf{A}$ is a diagonal matrix of zeros and ones. In medical imaging applications such as computed tomography or magnetic resonance imaging, the system operator takes on other forms such as a Radon transform, or a Fourier encoding, respectively.
A regularizer $ \zeta (\mathbf{y})$ is used in \eqref{reginveq1} to capture assumed properties of the underlying image $\mathbf{y}$ and to help compensate for noisy or incomplete data $\mathbf{z}$. 
For example, $ \zeta (\mathbf{y})$ could encourage the sparsity of $\mathbf{y}$ in some fixed or known sparsifying transform or dictionary, or alternatively, it could be an adaptive dictionary-type regularizer such as one based on (P0) \cite{bresai}.
The latter case corresponds to dictionary-blind image reconstruction, where the dictionary for the underlying image patches is unknown a priori. The goal is then to reconstruct both the image $\mathbf{y}$ as well as the dictionary $\mathbf{D}$ (for image patches) from the observed measurements $\mathbf{z}$.
Such an approach allows the dictionary to adapt to the underlying image \cite{bresai}.

\subsection{Contributions} \label{sec1b} 

This work focuses on dictionary learning using a general overall sparsity penalty instead of column-wise constraints like in (P0).
We focus on $\ell_{0}$ ``norm'' penalized dictionary learning, but also consider alternatives.
Similar to recent works \cite{elad, sadeg33}, we approximate the data ($\mathbf{Y}$) by a sum of sparse rank-one matrices or outer products.
The constraints and penalties in the learning problems are separable in terms of the dictionary columns and their corresponding coefficients, which enables efficient optimization.
In particular, we use simple and exact block coordinate descent approaches to estimate the factors of the various rank-one matrices in the dictionary learning problem. 
Importantly, we consider the application of such sparsity penalized dictionary learning in inverse problem settings, and investigate the problem of overall sparsity penalized dictionary-blind image reconstruction.  
We propose novel methods for image reconstruction that exploit the proposed efficient dictionary learning methods.
We provide a novel convergence analysis of the algorithms for overcomplete dictionary learning and dictionary-blind image reconstruction for both $\ell_{0}$ and $\ell_{1}$ norm-based settings.
Our experiments illustrate the empirical convergence behavior of our methods, and demonstrate their promising performance and speed-ups over some recent related schemes in sparse data representation and compressed sensing-based \cite{tao1, don} image reconstruction. These experimental results illustrate the benefits of aggregate sparsity penalized dictionary learning, and the proposed $\ell_{0}$ ``norm''-based methods.

\subsection{Relation to Recent Works} \label{sec1c} 

The sum of outer products approximation to data has been exploited in recent works \cite{elad, sadeg33} for developing dictionary learning algorithms. 
Sadeghi et al. \cite{sadeg33} considered a variation of the Approximate K-SVD algorithm \cite{ge35b} by including an $\ell_{1}$ penalty for coefficients in the dictionary update step of Approximate K-SVD. However, a formal and rigorous description of the formulations and various methods for overall sparsity penalized dictionary learning, and their extensions, was not developed in that work.
 In this work, we investigate in detail Sum of OUter Products (SOUP) based learning methodologies in a variety of novel problem settings. We focus mainly on $\ell_{0}$ ``norm'' penalized dictionary learning. While Bao et al. \cite{bao2, bao1} proposed proximal alternating schemes for $\ell_{0}$ dictionary learning, we show superior performance (both in terms of data representation quality and runtime) with the proposed simpler direct block coordinate descent methods for sparse data representation.
Importantly, we investigate the novel extensions of SOUP learning methodologies to inverse problem settings. 
We provide a detailed convergence analysis and empirical convergence studies for the various efficient algorithms for both dictionary learning and dictionary-blind image reconstruction. Our methods work better than classical overcomplete dictionary learning-based schemes (using K-SVD) in applications such as sparse data representation and magnetic resonance image reconstruction from undersampled data.
We also show some benefits of the proposed $\ell_{0}$ ``norm''-based adaptive methods over corresponding $\ell_{1}$ methods in applications.

\subsection{Organization} \label{sec1d} 

The rest of this paper is organized as follows. Section \ref{sec2} discusses the formulation for $\ell_{0}$ ``norm''-based dictionary learning, along with potential alternatives. Section \ref{sec3} presents the dictionary learning algorithms and their computational properties. Section \ref{sec7} discusses the formulations for dictionary-blind image reconstruction, along with the corresponding algorithms.
Section \ref{sec4} presents a convergence analysis for various algorithms. Section \ref{sec5} illustrates the empirical convergence behavior of various methods and demonstrates their usefulness for sparse data representation and inverse problems (compressed sensing). Section \ref{sec6} concludes with proposals for future work.

\section{Dictionary Learning Problem Formulations}
\label{sec2}

This section and the next focus on the ``classical'' problem of dictionary learning for sparse signal representation. Section \ref{sec7} generalizes these methods to inverse problems.

\subsection{$\ell_{0}$ Penalized Formulation} \label{sec2a} 

Following  \cite{bao1}, we consider a sparsity penalized variant of Problem (P0). Specifically, replacing the sparsity constraints in (P0) with an $\ell_0$ penalty $\sum_{i=1}^{N} \left \| \mathbf{x}_{i} \right \|_{0}$ and introducing a variable $\mathbf{C} = \mathbf{X}^{H} \in \mathbb{C}^{N \times J}$, where $(\cdot)^{H}$ denotes matrix Hermitian (conjugate transpose), leads to the following formulation:
\begin{align}  \label{eqop4}
& \min_{\mathbf{D},\mathbf{C}}\: \left \| \mathbf{Y}-\mathbf{D}\mathbf{C}^{H} \right \|_{F}^{2} + \lambda^{2} \left \| \mathbf{C} \right \|_{0} \; \: \mathrm{s.t.}\; \: \left \| \mathbf{d}_{j} \right \|_2 =1 \,\forall \, j.
\end{align}
where $\left \| \mathbf{C} \right \|_{0}$ counts the number of non-zeros in matrix $\mathbf{C}$, and $\lambda^{2}$ with $\lambda >0$, is a weight to control the overall sparsity.

Next, following previous work like \cite{elad, sadeg33}, we express the matrix $\mathbf{D}\mathbf{C}^{H}$ in \eqref{eqop4} as a sum of (sparse) rank-one matrices or outer products $\sum_{j=1}^{J} \mathbf{d}_{j}\mathbf{c}_{j}^{H}$, where $\mathbf{c}_{j}$ is the $j$th column of $\mathbf{C}$. This SOUP representation of the data $\mathbf{Y}$ is natural because it separates out the contributions of the various atoms in representing the data.
For example, atoms of a dictionary whose contributions to the data ($\mathbf{Y}$) representation error or modeling error are small could be dropped. 
With this model, \eqref{eqop4} becomes (P1) as follows, where $\left \| \mathbf{C} \right \|_{0} = \sum_{j=1}^{J} \left \| \mathbf{c}_{j} \right \|_{0}$:
\begin{align} 
\nonumber (\mathrm{P1})\; & \min_{\left \{ \mathbf{d}_{j},\mathbf{c}_{j} \right \}}\: \begin{Vmatrix}
\mathbf{Y}- \sum_{j=1}^{J} \mathbf{d}_{j}\mathbf{c}_{j}^{H}
\end{Vmatrix}_{F}^{2} + \lambda^{2} \sum_{j=1}^{J} \left \| \mathbf{c}_{j} \right \|_{0} \\
\nonumber & \; \; \; \mathrm{s.t.}\; \: \left \| \mathbf{d}_{j} \right \|_2 =1, \, \left \| \mathbf{c}_{j} \right \|_{\infty} \leq L \,\forall \, j.
\end{align}
As in Problem (P0), the matrix $\mathbf{d}_{j}\mathbf{c}_{j}^{H}$ in (P1) is invariant to joint scaling of $\mathbf{d}_{j}$ and $\mathbf{c}_{j}$ as $\alpha \mathbf{d}_{j}$ and $(1/\alpha) \mathbf{c}_{j}$, for $\alpha \neq 0$. The constraint $\left \| \mathbf{d}_{j} \right \|_2 =1$ helps in removing this scaling ambiguity. 
We also enforce the constraint $\left \| \mathbf{c}_{j} \right \|_{\infty} \leq L$, with $L>0$, in (P1) \cite{bao1} (e.g., $L=\left \| \mathbf{Y} \right \|_{F}$).
This is because the objective in (P1) is non-coercive. In particular, consider a dictionary $\mathbf{D}$ that has a column $\mathbf{d}_{j}$ that repeats. Then, in this case, the SOUP approximation for $\mathbf{Y}$ in (P1) could have both the terms $\mathbf{d}_{j}\mathbf{c}_{j}^{H}$ and $-\mathbf{d}_{j}\mathbf{c}_{j}^{H}$ with $\mathbf{c}_{j}$ that is highly sparse (and non-zero), and the objective would be invariant\footnote{Such degenerate representations for $\mathbf{Y}$, however, cannot be minimizers in the problem because they simply increase the $\ell_{0}$ sparsity penalty without affecting the fitting error (the first term) in the cost.} to (arbitrarily) large scalings of $\mathbf{c}_{j}$ (i.e., non-coercive objective).
The $\ell_{\infty}$ constraints on the columns of $\mathbf{C}$ (that constrain the magnitudes of entries of $\mathbf{C}$) alleviate possible problems (e.g., unbounded iterates in algorithms) due to such a non-coercive objective.

Problem (P1) aims to learn the factors $\left \{ \mathbf{d}_{j} \right \}_{j=1}^{J}$ and $\left \{ \mathbf{c}_{j} \right \}_{j=1}^{J}$ that enable the best SOUP sparse representation of $\mathbf{Y}$. However, (P1), like (P0), is non-convex, even if one replaces the $\ell_{0}$ ``norm'' with a convex penalty.

Unlike the sparsity constraints in (P0), the term $\left \| \mathbf{C} \right \|_{0}=\sum_{j=1}^{J}\left \| \mathbf{c}_{j} \right \|_{0}=\sum_{i=1}^{N}\left \| \mathbf{x}_{i} \right \|_{0}=\left \| \mathbf{X} \right \|_{0}$ in Problem (P1) (or \eqref{eqop4}) penalizes the number of non-zeros in the (entire) coefficient matrix (i.e., the number of non-zeros used to represent a collection of signals), allowing variable sparsity levels across the signals. This flexibility could enable better data representation error versus sparsity trade-offs
than with a fixed column sparsity constraint (as in (P0)). 
For example, in imaging or image processing applications, the dictionary is usually learned for image patches. Patches from different regions of an image typically contain different amounts of information\footnote{Here, the emphasis is on the required sparsity levels for encoding different patches. This is different from the motivation for multi-class models such as in \cite{saiwen, sravTCI1} (or \cite{irami, zhan33}), where patches from different regions of an image are assumed to contain different ``types'' of features or textures or edges, and thus common sub-dictionaries or sub-transforms are learned for groups of patches with similar features.}, and thus enforcing a common sparsity bound
for various patches does not reflect typical image properties (i.e., is restrictive) and usually leads to sub-optimal
performance in applications.
In contrast, Problem (P1) encourages a more general and flexible image model, and leads to promising performance in the experiments of this work.
Additionally, we have observed that the different columns of $\mathbf{C}$ (or rows of $ \mathbf{X}$) learned by the proposed algorithm (in Section \ref{sec3}) for (P1) typically have widely different sparsity levels or number of non-zeros in practice.

\subsection{Alternative Formulations} \label{sec2b} 

Several variants of Problem (P1) could be constructed that also involve the SOUP representation. For example, the $\ell_0$ ``norm" for sparsity could be replaced by the $\ell_1$ norm \cite{sadeg33} resulting in the following formulation: 
\begin{align} 
\nonumber (\mathrm{P2})\; & \min_{\left \{ \mathbf{d}_{j},\mathbf{c}_{j} \right \}}\: \begin{Vmatrix}
\mathbf{Y}- \sum_{j=1}^{J} \mathbf{d}_{j}\mathbf{c}_{j}^{H}
\end{Vmatrix}_{F}^{2} + \mu \sum_{j=1}^{J} \left \| \mathbf{c}_{j} \right \|_{1} \\
\nonumber & \; \; \; \mathrm{s.t.}\; \: \left \| \mathbf{d}_{j} \right \|_2 =1 \,\forall \, j.
\end{align}
Here, $\mu>0$, and the objective is coercive with respect to $\mathbf{C}$  because of the $\ell_{1}$ penalty.
Another alternative to (P1) enforces $p$-block-orthogonality constraints on $\mathbf{D}$. The dictionary in this case is split into blocks (instead of individual atoms), each of which has $p$ (unit norm) atoms that are orthogonal to each other. For $p=2$, we would have (added) constraints such as  $\mathbf{d}_{2j-1}^{H}\mathbf{d}_{2j} = 0$, $1\leq j \leq J/2$. In the extreme (more constrained) case of $p=n$, the dictionary would be made of several square unitary\footnote{Recent works have shown the promise of learned orthonormal (or unitary) dictionaries or sparsifying transforms in applications such as image denoising \cite{sabres3, Cai201489}. Learned multi-class unitary models have been shown to work well in inverse problem settings such as in MRI \cite{sravTCI1, zhan33}.} blocks (cf. \cite{lesa1}).
For tensor-type data, (P1) can be modified by enforcing the dictionary atoms to be in the form of a Kronecker product.
The algorithm proposed in Section \ref{sec3} can be easily extended to accommodate several such variants of Problem (P1).
We do not explore all such alternatives in this work due to space constraints, and a more detailed investigation of these is left for future work.

\newtheorem{theorem}{Theorem}
\newtheorem{cor}{Corollary}
\newtheorem{lem}{Lemma}
\newtheorem{prop}{Proposition}

\section{Learning Algorithm and Properties}
\label{sec3}  

\subsection{Algorithm} \label{sec3a} 

We apply a block coordinate descent method to estimate the unknown variables in Problem (P1). For each $j$ ($1 \leq j \leq J$), the algorithm has two steps.
First, we solve (P1) with respect to $\mathbf{c}_{j}$ keeping all the other variables fixed. We refer to this step as the \emph{sparse coding step} in our method. Once $\mathbf{c}_{j}$ is updated, we solve (P1) with respect to $\mathbf{d}_{j}$ keeping all other variables fixed. This step is referred to as the \emph{dictionary atom update step} or simply \emph{dictionary update step}. The algorithm thus updates the factors of the various rank-one matrices one-by-one.
The approach for (P2) is similar and is a simple extension of the OS-DL method in \cite{sadeg33} to the complex-valued setting.
We next describe the sparse coding and dictionary atom update steps of the methods for (P1) and (P2).

\subsubsection{Sparse Coding Step for (P1)} \label{sec3a1}
Minimizing (P1) with respect to $\mathbf{c}_{j}$ leads to the following non-convex problem, where $\mathbf{E}_{j} \triangleq \mathbf{Y} - \sum_{k\neq j} \mathbf{d}_{k}\mathbf{c}_{k}^{H}$ is a fixed matrix based on the most recent values of all other atoms and coefficients:
\begin{equation} \label{eqop5}
\min_{\mathbf{c}_{j} \in \mathbb{C}^{N}} \; \begin{Vmatrix}
\mathbf{E}_{j} - \mathbf{d}_{j}\mathbf{c}_{j}^{H}
\end{Vmatrix}_{F}^{2} + \lambda^{2} \left \| \mathbf{c}_{j} \right \|_{0}  \;\; \mathrm{s.t.}\; \: \left \| \mathbf{c}_{j} \right \|_{\infty} \leq L.
\end{equation}
The following proposition provides the solution to Problem \eqref{eqop5}, where the hard-thresholding operator $H_{\lambda} (\cdot)$ is defined as
\begin{equation} \label{equ88ch4}
\left ( H_{\lambda} (\mathbf{b}) \right )_{i}=\left\{\begin{matrix}
 0,& \;\;\left | b_{i} \right | < \lambda \\
b_{i},  & \;\;\left | b_{i} \right | \geq \lambda 
\end{matrix}\right.
\end{equation}
with $\mathbf{b} \in \mathbb{C}^{N}$, and the subscript $i$ above indexes vector entries. We use $b_{i}$ (without bold font) to denote the $i$th (scalar) element of a vector $\mathbf{b}$.
We assume that the bound $L > \lambda$ holds and let $\mathbf{1}_{N}$ denote a vector of ones of length $N$. The operation ``$\odot$" denotes element-wise multiplication,
and $\mathbf{z}= \min(\mathbf{a}, \mathbf{u})$ for vectors $\mathbf{a}, \mathbf{u} \in \mathbb{R}^{N}$ denotes the element-wise minimum operation, i.e., $z_{i}= \min(a_{i}, b_{i})$, $1 \leq i \leq N$.
For a vector $\mathbf{c} \in \mathbb{C}^{N}$, $e^{j \angle \mathbf{c}} \in \mathbb{C}^{N}$ is computed element-wise, with ``$\angle$'' denoting the phase.

\begin{prop}\label{prop1} \vspace{0.02in}
Given $\mathbf{E}_{j} \in \mathbb{C}^{n \times N}$ and $\mathbf{d}_{j} \in \mathbb{C}^{n}$, and assuming $L > \lambda$,  a global minimizer of the sparse coding problem \eqref{eqop5} is obtained by the following truncated hard-thresholding operation:
\begin{equation} \label{tru1ch4}
\hat{\mathbf{c}}_{j} =  \min\left ( \begin{vmatrix}
H_{\lambda} \left ( \mathbf{E}_{j}^{H}\mathbf{d}_{j} \right )
\end{vmatrix}, L \mathbf{1}_{N} \right ) \, \odot \, e^{j \angle \,  \mathbf{E}_{j}^{H}\mathbf{d}_{j}  }.
\end{equation}
The minimizer of \eqref{eqop5} is unique if and only if the vector $\mathbf{E}_{j}^{H}\mathbf{d}_{j}$ has no entry with a magnitude of $\lambda$. 
\end{prop}

The proof of Proposition \ref{prop1} is provided in the supplementary material.

\subsubsection{Sparse Coding Step for (P2)} \label{sec3a1b}

The sparse coding step of (P2) involves solving the following problem: 
\begin{equation} \label{eqop5bb}
\min_{\mathbf{c}_{j}  \in \mathbb{C}^{N}} \; \begin{Vmatrix}
\mathbf{E}_{j} - \mathbf{d}_{j}\mathbf{c}_{j}^{H}
\end{Vmatrix}_{F}^{2} + \mu \left \| \mathbf{c}_{j} \right \|_{1}.
\end{equation}
The solution is given by the following proposition (proof in the supplement), and was previously discussed in \cite{sadeg33} for the case of real-valued data.

\begin{prop}\label{prop1b} \vspace{0.02in}
Given $\mathbf{E}_{j} \in \mathbb{C}^{n \times N}$ and $\mathbf{d}_{j} \in \mathbb{C}^{n}$,  the unique global minimizer of the sparse coding problem \eqref{eqop5bb} is
\begin{equation} \label{tru1ch4bnmn}
\hat{\mathbf{c}}_{j} =  \max \left (\begin{vmatrix}
  \mathbf{E}_{j}^{H}\mathbf{d}_{j}
\end{vmatrix} - \frac{\mu}{2} \mathbf{1}_{N}, \, 0 \right ) \, \odot \, e^{j \angle \, \mathbf{E}_{j}^{H}\mathbf{d}_{j} }.
\end{equation}
\end{prop}

\subsubsection{Dictionary Atom Update Step} \label{sec3a2}

Minimizing (P1) or (P2) with respect to $\mathbf{d}_{j}$ leads to the following problem:
\begin{equation} \label{eqop6}
 \min_{\mathbf{d}_{j} \in \mathbb{C}^{n}} \; \begin{Vmatrix}
\mathbf{E}_{j} - \mathbf{d}_{j}\mathbf{c}_{j}^{H}
\end{Vmatrix}_{F}^{2}  \;\:\; \mathrm{s.t.}\; \: \left \| \mathbf{d}_{j} \right \|_2 =1.
\end{equation}
Proposition \ref{prop2} provides the closed-form solution for \eqref{eqop6}. The solution takes the form given in \cite{ge35b}. We briefly derive the solution in the supplementary material considering issues such as uniqueness.

\begin{prop}\label{prop2} \vspace{0.02in}
Given $\mathbf{E}_{j} \in \mathbb{C}^{n \times N}$ and $\mathbf{c}_{j} \in \mathbb{C}^{N}$, a global minimizer of the dictionary atom update problem \eqref{eqop6} is
\begin{equation} \label{tru1ch4g}
\hat{\mathbf{d}}_{j} =  \left\{\begin{matrix}
\frac{\mathbf{E}_{j}\mathbf{c}_{j}}{\left \| \mathbf{E}_{j}\mathbf{c}_{j} \right \|_{2}}, & \mathrm{if}\,\, \mathbf{c}_{j}\neq 0 \\ 
\mathbf{v}, & \mathrm{if}\,\, \mathbf{c}_{j}= 0 
\end{matrix}\right.
\end{equation}
where $\mathbf{v}$ can be any unit $\ell_{2}$ norm vector (i.e., on the unit sphere). In particular, here, we set $\mathbf{v}$ to be the first column of the $n \times n $ identity matrix.
The solution is unique if and only if $\mathbf{c}_{j}\neq 0$. 
\end{prop}

\begin{figure}
\begin{tabular}{p{8.3cm}}
\hline
SOUP-DILLO Algorithm\\
\hline
 \textbf{Inputs\;:} \:\:\: Data $ \mathbf{Y} \in \mathbb{C}^{n \times N}$, weight $\lambda$, upper bound $L$, and number of iterations $K$.\\
 \textbf{Outputs\;:} \:\:\: Columns $ \left \{ \mathbf{d}_{j}^{K} \right \}_{j=1}^{J}$ of the learned dictionary, and the learned sparse coefficients $\left \{\mathbf{c}_{j}^{K} \right \}_{j=1}^{J}$. \\
\textbf{Initial Estimates:} $\left \{ \mathbf{d}_{j}^{0}, \mathbf{c}_{j}^{0} \right \}_{j=1}^{J}$.  (Often $\mathbf{c}_{j}^{0} = \mathbf{0}$ $\forall$ $j$.)\\
\textbf{For \;$t$ = $1:$ $K$ repeat}\\
\vspace{-0.09in}\hspace{0.01in} \textbf{For \;$j$ = $1:$ $J$ repeat}
\begin{enumerate}
\item $\mathbf{C}=\left [ \mathbf{c}_{1}^{t},...,\mathbf{c}_{j-1}^{t},\mathbf{c}_{j}^{t-1},...,\mathbf{c}_{J}^{t-1} \right ]$. \newline
$\mathbf{D}=\left [ \mathbf{d}_{1}^{t},...,\mathbf{d}_{j-1}^{t},\mathbf{d}_{j}^{t-1},...,\mathbf{d}_{J}^{t-1} \right ]$. \vspace{0.04in}
\item \textbf{Sparse coding:} 
\begin{equation} \label{trree1}
\mathbf{b}^{t} =   \mathbf{Y}^{H} \mathbf{d}_{j}^{t-1} - \mathbf{C}\mathbf{D}^{H}\mathbf{d}_{j}^{t-1} + \mathbf{c}_{j}^{t-1}
\end{equation}
\begin{equation} \label{tru1ch4g88bn}
\mathbf{c}_{j}^{t} =  \min\left ( \begin{vmatrix}
H_{\lambda} \left ( \mathbf{b}^{t}  \right )
\end{vmatrix}, L \mathbf{1}_{N} \right ) \, \odot \,  e^{j \angle  \mathbf{b}^{t}}
\end{equation}
\item \textbf{Dictionary atom update:} 
\begin{equation} \label{rbb}
\mathbf{h}^{t} =  \mathbf{Y} \mathbf{c}_{j}^{t} - \mathbf{D}\mathbf{C}^{H}\mathbf{c}_{j}^{t} + \mathbf{d}_{j}^{t-1}\left ( \mathbf{c}_{j}^{t-1} \right )^{H}\mathbf{c}_{j}^{t}
\end{equation}
\begin{equation} \label{tru1ch4g88}
\mathbf{d}_{j}^{t} =  \left\{\begin{matrix}
\frac{\mathbf{h}^{t}}{\left \| \mathbf{h}^{t} \right \|_{2}}, & \mathrm{if}\,\, \mathbf{c}_{j}^{t}\neq 0 \\ 
\mathbf{v}, & \mathrm{if}\,\, \mathbf{c}_{j}^{t}= 0 
\end{matrix}\right.
\end{equation}
\end{enumerate}\\
\hspace{0.01in} \textbf{End} \\
\textbf{End} \\
\hline
\end{tabular}
\caption{The SOUP-DILLO Algorithm (due to the $\ell_{0}$ ``norm") for Problem (P1). Superscript $t$ denotes the iterates in the algorithm. The vectors $\mathbf{b}^{t}$ and $\mathbf{h}^{t}$ above are computed efficiently via sparse operations.} \label{im5p}
\end{figure}

\subsubsection{Overall Algorithms} \label{sec3a2hhhh}

Fig. \ref{im5p} shows the Sum of OUter Products DIctionary Learning (SOUP-DIL) Algorithm for Problem (P1), dubbed SOUP-DILLO in this case, due to the $\ell_{0}$ ``norm". 
The algorithm needs initial estimates $\left \{ \mathbf{d}_{j}^{0}, \mathbf{c}_{j}^{0} \right \}_{j=1}^{J}$ for the variables. For example, the initial sparse coefficients could be set to zero, and the initial dictionary could be a known analytical dictionary such as the overcomplete DCT \cite{elad2}.
When $\mathbf{c}_{j}^{t}= \mathbf{0}$, setting $\mathbf{d}_{j}^{t}$ to be the first column of the identity matrix in the algorithm could also be replaced with other (equivalent) settings such as $\mathbf{d}_{j}^{t} = \mathbf{d}_{j}^{t-1}$ or setting $\mathbf{d}_{j}^{t}$ to a random unit norm vector. All of these settings have been observed to work well in practice.
A random ordering of the atom/sparse coefficient updates in Fig. \ref{im5p}, i.e., random $j$ sequence, also works in practice in place of cycling in the same order $1$ through $J$ every iteration.
One could also alternate several times between the sparse coding and dictionary atom update steps for each $j$. However, this variation would increase computation.

The method for (P2) differs from SOUP-DILLO in the sparse coding step (Proposition \ref{prop1b}). From prior work \cite{sadeg33}, we refer to this method (for (P2)) as OS-DL. We implement this method in a similar manner as in Fig. \ref{im5p} (for complex-valued data); unlike OS-DL in \cite{sadeg33}, our implementation does not compute the matrix $\mathbf{E}_{j}$ for each $j$.

Finally, while we interleave the sparse coefficient ($\mathbf{c}_{j}$) and atom ($\mathbf{d}_{j}$) updates in Fig.~\ref{im5p}, one could also cycle first through all the columns of $\mathbf{C}$ and then through the columns of $\mathbf{D}$ in the block coordinate descent (SOUP) methods. Such an approach was adopted recently in \cite{zli2} for $\ell_{1}$ penalized dictionary learning. We have observed similar performance with such an alternative update ordering strategy compared to an interleaved update order. Although the convergence results in Section \ref{sec4} are for the ordering in Fig. \ref{im5p}, similar results can be shown to hold with alternative (deterministic) orderings in various settings.

\subsection{Computational Cost Analysis} \label{sec3b} 

For each iteration $t$ in Fig. \ref{im5p}, SOUP-DILLO involves $J$ sparse code and dictionary atom updates. The sparse coding and atom update steps involve matrix-vector products for computing $\mathbf{b}^{t}$ and $\mathbf{h}^{t}$, respectively.

\textbf{Memory Usage:} An alternative approach to the one in Fig.~\ref{im5p} involves computing $\mathbf{E}_{j}^{t} = \mathbf{Y} - \sum_{k<j} \mathbf{d}_{k}^{t} \left ( \mathbf{c}_{k}^{t} \right )^{H} - \sum_{k>j} \mathbf{d}_{k}^{t-1} \left ( \mathbf{c}_{k}^{t-1} \right )^{H}$ (as in Propositions~\ref{prop1} and \ref{prop2}) directly at the beginning of each inner $j$ iteration.
This matrix could be updated sequentially and efficiently for each $j$ by adding and subtracting appropriate sparse rank-one matrices, as done in OS-DL in \cite{sadeg33} for the $\ell_{1}$ case.
However, this alternative approach requires storing and updating $\mathbf{E}_{j}^{t} \in \mathbb{C}^{n \times N}$, which is a large matrix for large $N$ and $n$. The procedure in Fig.~\ref{im5p} avoids this overhead (similar to the Approximate K-SVD approach \cite{zibul}), and is faster and saves memory usage.

\textbf{Computational Cost:} We now discuss the cost of each sparse coding and atom update step in the SOUP-DILLO method of Fig. \ref{im5p} (a similar discussion holds for the method for (P2)).
Consider the $t$th iteration and the $j$th inner iteration in  Fig.~\ref{im5p}, consisting of the update of the $j$th dictionary atom $\mathbf{d}_{j}$ and its corresponding sparse coefficients $\mathbf{c}_{j}$. As in Fig. \ref{im5p}, let $\mathbf{D} \in \mathbb{C}^{n \times J}$ be the dictionary whose columns are the current estimates of the atoms (at the start of the $j$th inner iteration), and let $\mathbf{C} \in \mathbb{C}^{N \times J}$ be the corresponding sparse coefficients matrix. (The index $t$ on $\mathbf{D}$ and $\mathbf{C}$ is dropped to keep the notation simple.) Assume that the matrix $\mathbf{C}$ has $\alpha N n$ non-zeros, with $\alpha \ll 1$ typically. This translates to an average of $\alpha n$ non-zeros per row of $\mathbf{C}$ or $\alpha N n / J$ non-zeros per column of $\mathbf{C}$. We refer to $\alpha$ as the sparsity factor of $\mathbf{C}$.

The sparse coding step involves computing the right hand side of \eqref{trree1}.
While computing $\mathbf{Y}^{H} \mathbf{d}_{j}^{t-1} $ requires $N n$ multiply-add\footnote{In the case of complex-valued data, this would be the complex-valued multiply-accumulate (CMAC) operation (cf. \cite{wef1}) that requires 4 real-valued multiplications and 4 real-valued additions.} operations, computing $\mathbf{C}\mathbf{D}^{H}\mathbf{d}_{j}^{t-1}$ using matrix-vector products requires $J n + \alpha N n$ multiply-add operations.
The remainder of the operations in \eqref{trree1} and \eqref{tru1ch4g88bn} have $O(N)$ cost.

Next, when $\mathbf{c}_{j}^{t} \neq \mathbf{0}$, the dictionary atom update is as per \eqref{rbb} and  \eqref{tru1ch4g88}.
Since $\mathbf{c}_{j}^{t}$ is sparse with say $r_{j}$ non-zeros, computing $ \mathbf{Y} \mathbf{c}_{j}^{t}$ in \eqref{rbb} requires $n r_{j}$ multiply-add operations, and computing $\mathbf{D}\mathbf{C}^{H}\mathbf{c}_{j}^{t}$  requires less than $J n + \alpha N n$ multiply-add operations. The cost of the remaining operations in \eqref{rbb} and \eqref{tru1ch4g88} is negligible.

Thus, the net cost of the $J \geq n$ inner iterations in iteration $t$ in Fig. \ref{im5p} is dominated (for $N \gg J, n$) by $N J n + 2 \alpha_{m} N J n + \beta  N n^{2}$, where $\alpha_{m}$ is the maximum sparsity factor of the estimated $\mathbf{C}$'s during the inner iterations, and $\beta$ is the sparsity factor of the estimated $\mathbf{C}$ at the end of iteration $t$. Thus, the cost per iteration of the block coordinate descent SOUP-DILLO Algorithm is about $(1 + \alpha') N J n$, with $\alpha' \ll 1$ typically. On the other hand, the proximal alternating algorithm proposed recently by Bao et al. for (P1) \cite{bao2, bao1} (Algorithm 2 in \cite{bao2}) has a per-iteration cost of at least $2 N J n + 6 \alpha N J n + 4 \alpha N n^{2}$. This is clearly more computation\footnote{Bao et al. also proposed another proximal alternating scheme (Algorithm 3 in \cite{bao2}) for discriminative incoherent dictionary learning. However, this method, when applied to (P1) (as a special case of discriminative incoherent learning), has been shown in \cite{bao2} to be much slower than the proximal Algorithm 2 \cite{bao2} for (P1).} than SOUP-DILLO.
The proximal methods \cite{bao2} also involve more parameters than direct block coordinate descent schemes.

Assuming $J \propto n$, the cost per iteration of the SOUP-DILLO Algorithm scales as $O(N n^{2})$.
This is lower than the per-iteration cost of learning an $n \times J$ synthesis dictionary $\mathbf{D}$ using K-SVD \cite{elad}, which scales\footnote{When $s \propto n$ and $J \propto n$, the per-iteration computational cost of the efficient implementation of K-SVD \cite{ge35b} also scales similarly as $O(Nn^{3})$.} (assuming that the synthesis sparsity level $s \propto n$ and $J \propto n$ in K-SVD) as $O(Nn^{3})$.
SOUP-DILLO converges in few iterations in practice (cf. supplement). Therefore, the per-iteration computational advantages may also translate to net computational advantages in practice. This low cost could be particularly useful for big data applications, or higher dimensional (3D or 4D) applications.

\section{Dictionary-Blind Image Reconstruction} \label{sec7}

\subsection{Problem Formulations} \label{sec7a}

Here, we consider the application of sparsity penalized dictionary learning to inverse problems. In particular, we use the following $\ell_{0}$ aggregate sparsity penalized dictionary learning regularizer that is based on (P1)
\begin{align*}
\zeta(\mathbf{y}) = & \frac{1}{\nu} \min_{\mathbf{D}, \mathbf{X}}\:  \sum_{i=1}^{N} 
\left \| \mathbf{P}_{i}\mathbf{y}- \mathbf{D} \mathbf{x}_{i} \right \|_{2}^{2}
 + \lambda^{2}  \left \| \mathbf{X} \right \|_{0} \\
&\;\;\;\;\; \mathrm{s.t.}\; \: \left \| \mathbf{d}_{j} \right \|_2 =1, \, \left \| \mathbf{x}_{i} \right \|_{\infty} \leq L \,\forall \, i,j
\end{align*}
in \eqref{reginveq1} to arrive at the following dictionary-blind image reconstruction problem:
\begin{align}
\nonumber (\mathrm{P3})\:\: & \min_{\mathbf{y},\mathbf{D}, \mathbf{X}}\:  \nu \left \| \mathbf{A}\mathbf{y}-\mathbf{z} \right \|_{2}^{2} + \sum_{i=1}^{N} 
\left \| \mathbf{P}_{i}\mathbf{y}- \mathbf{D} \mathbf{x}_{i} \right \|_{2}^{2}
 + \lambda^{2}  \left \| \mathbf{X} \right \|_{0}\\
\nonumber & \;\; \mathrm{s.t.}\; \: \left \| \mathbf{d}_{j} \right \|_2 =1, \, \left \| \mathbf{x}_{i} \right \|_{\infty} \leq L \,\forall \, i,j.
\end{align}
Here, $ \mathbf{P}_{i} \in \mathbb{R}^{n \times p} $ is an operator that extracts a $\sqrt{n} \times \sqrt{n}$ patch (for a 2D image) of $\mathbf{y}$ as a vector $ \mathbf{P}_{i} \mathbf{y}$, and $\mathbf{D}  \in \mathbb{C}^{n \times J} $ is a (unknown) synthesis dictionary for the image patches.
A total of $N$ overlapping image patches are assumed, and $\nu >0$ is a weight in (P3). We use $\mathbf{Y}$ to denote the matrix whose columns are the patches $\mathbf{P}_{i}\mathbf{y}$, and $\mathbf{X}$ (with columns $\mathbf{x}_{i}$) denotes the corresponding dictionary-sparse representation of $\mathbf{Y}$. All other notations are as before.
Similarly as in (P1), we approximate the (unknown) patch matrix $\mathbf{Y}$ using a sum of outer products representation.

An alternative to Problem (P3) uses a regularizer $\zeta(\mathbf{y})$ based on Problem (P2) rather than (P1). In this case, we have the following $\ell_{1}$ sparsity penalized dictionary-blind image reconstruction problem, where $\left \| \mathbf{X} \right \|_{1} = \sum_{i=1}^{N}\left \| \mathbf{x}_{i} \right \|_{1}$:
\begin{align}
\nonumber (\mathrm{P4})\:\: & \min_{\mathbf{y},\mathbf{D}, \mathbf{X}}\:  \nu \left \| \mathbf{A}\mathbf{y}-\mathbf{z} \right \|_{2}^{2} + \sum_{i=1}^{N} 
\left \| \mathbf{P}_{i}\mathbf{y}- \mathbf{D} \mathbf{x}_{i} \right \|_{2}^{2}
 + \mu  \left \| \mathbf{X} \right \|_{1}\\
\nonumber & \;\; \mathrm{s.t.}\; \: \left \| \mathbf{d}_{j} \right \|_2 =1 \,\forall \, j.
\end{align}

Similar to (P1) and (P2), the dictionary-blind image reconstruction problems (P3) and (P4) are non-convex. The goal in these problems is to  learn a dictionary and sparse coefficients, and reconstruct the image using only the measurements $\mathbf{z}$.

\subsection{Algorithms and Properties} \label{sec7b}

We adopt iterative block coordinate descent methods for (P3) and (P4) that lead to highly efficient solutions for the corresponding subproblems. In the \emph{dictionary learning step}, we minimize (P3) or (P4) with respect to $(\mathbf{D}, \mathbf{X})$ keeping $\mathbf{y}$ fixed. In the \emph{image update step}, we solve (P3) or (P4) for the image $\mathbf{y}$ keeping the other variables fixed. We describe these steps below.

\subsubsection{Dictionary Learning Step} \label{sec7b1}

Minimizing (P3) with respect to $(\mathbf{D},\mathbf{X})$ involves the following problem:
\begin{align}
\nonumber & \min_{\mathbf{D}, \mathbf{X}}\:
\left \| \mathbf{Y}- \mathbf{D} \mathbf{X} \right \|_{F}^{2}
 + \lambda^{2}  \left \| \mathbf{X} \right \|_{0} \\
&\;\, \mathrm{s.t.}\; \: \left \| \mathbf{d}_{j} \right \|_2 =1, \, \left \| \mathbf{x}_{i} \right \|_{\infty} \leq L \,\forall \, i,j.  \label{rop1}
\end{align}
By using the substitutions $\mathbf{X} = \mathbf{C}^{H}$ and $\mathbf{D}\mathbf{C}^{H}=\sum_{j=1}^{J}\mathbf{d}_{j}\mathbf{c}_{j}^{H}$, Problem \eqref{rop1} becomes (P1) \footnote{The $\ell_{\infty}$ constraints on the columns of $\mathbf{X}$ translate to identical constraints on the columns of $\mathbf{C}$.}. We then apply the SOUP-DILLO algorithm in Fig. \ref{im5p} to update the dictionary $\mathbf{D}$ and sparse coefficients $\mathbf{C}$.
In the case of (P4), when minimizing with respect to $(\mathbf{D},\mathbf{X})$, we again set $\mathbf{X}=\mathbf{C}^{H}$ and use the SOUP representation to recast the resulting problem in the form of (P2). The dictionary and coefficients are then updated using the OS-DL method.

\subsubsection{Image Update Step} \label{sec7b2}

Minimizing (P3) or (P4) with respect to $\mathbf{y}$ involves the following optimization problem:
\begin{align}
& \min_{\mathbf{y}}\:  \nu \left \| \mathbf{A}\mathbf{y}-\mathbf{z} \right \|_{2}^{2} + \sum_{i=1}^{N} 
\left \| \mathbf{P}_{i}\mathbf{y}- \mathbf{D} \mathbf{x}_{i} \right \|_{2}^{2} \label{rop4}
\end{align}
This is a least squares problem whose solution satisfies the following normal equation:
\begin{align} 
 & \left (  \sum_{i=1}^{N} \mathbf{P}_{i}^{T} \mathbf{P}_{i} \; +\;\nu\: \mathbf{A}^{H}\mathbf{A} \right )\mathbf{y}=  \sum_{i=1}^{N} \mathbf{P}_{i}^{T}\mathbf{D} \mathbf{x}_{i} + \nu \: \mathbf{A}^{H}\mathbf{z} \label{bcs10}
\end{align}
When periodically positioned, overlapping patches (patch overlap stride \cite{bresai} denoted by $r$) are used, and the patches that overlap the image boundaries `wrap around' on the opposite side of the image \cite{bresai}, then $ \sum_{i=1}^{N} \mathbf{P}_{i}^{T} \mathbf{P}_{i}$ is a diagonal matrix. Moreover, when the patch stride $r=1$, $ \sum_{i=1}^{N} \mathbf{P}_{i}^{T} \mathbf{P}_{i}=  \beta I$, with $\beta=n$.
In general, the unique solution to \eqref{bcs10} can be found using techniques such as conjugate gradients (CG). In several applications, the  matrix $ \mathbf{A}^{H}\mathbf{A}$ in \eqref{bcs10} is diagonal (e.g., in denoising or in inpainting) or readily diagonalizable. In such cases, the solution to \eqref{bcs10} can be found efficiently \cite{elad2, bresai}.
Here, we consider single coil compressed sensing MRI \cite{lustig}, where $\mathbf{A} = \mathbf{F}_{\mathrm{u}} \in \mathbb{C}^{m \times p}$ ($m \ll p$), the undersampled Fourier encoding matrix.
Here, the measurements $\mathbf{z}$ are samples in Fourier space (or k-space) of an object $\mathbf{y}$, and we assume for simplicity that $\mathbf{z}$ is obtained by subsampling on a uniform Cartesian (k-space) grid. 
Denoting by $  \mathbf{F}  \in \mathbb{C}^{p \times p} $ the full Fourier encoding matrix with $  \mathbf{F}^{H} \mathbf{F} =  \mathbf{I}$ (normalized), we get $  \mathbf{F} \mathbf{F}_{\mathrm{u}}^{H} \mathbf{F}_{\mathrm{u}} \mathbf{F}^{H} $ is a diagonal matrix of ones and zeros, with ones at entries correspond to sampled k-space locations. Using this in \eqref{bcs10} yields the following solution in Fourier space \cite{bresai} with $\mathbf{S} \triangleq \mathbf{F} \sum_{i=1}^{N} \mathbf{P}_{i}^{T}\mathbf{D} \mathbf{x}_{i}$, $\mathbf{S}_{0} \triangleq \mathbf{F} \mathbf{F}_{\mathrm{u}}^{H}\mathbf{z}$, and $\beta =n$ (i.e., assuming $r=1$):
\begin{equation}\label{bcs13}
\mathbf{F}\mathbf{y} \:(k_{1},k_{2})=\left\{\begin{matrix}
\frac{\mathbf{S}(k_{1},k_{2})}{\beta} &,\:(k_{1},k_{2})\notin \Omega  \\
\frac{\mathbf{S}(k_{1},k_{2})+\nu\, \mathbf{S}_{0}(k_{1},k_{2})}{\beta+\nu}   & ,\:(k_{1},k_{2})\in \Omega
\end{matrix}\right.
\end{equation}
where $(k_{1},k_{2})$ indexes k-space or frequency locations (2D coordinates), and $ \Omega $ is the subset of k-space sampled. The $\mathbf{y}$ solving \eqref{bcs10} is obtained by an inverse FFT of $\mathbf{F}\mathbf{y}$ in \eqref{bcs13}.

\subsubsection{Overall Algorithms and Computational Costs} \label{sec7b3}

\begin{figure}
\begin{tabular}{p{8.3cm}}
\hline
Algorithms for (P3) and (P4)\\
\hline
 \textbf{Inputs\;:} \:\:\: measurements $ \mathbf{z} \in \mathbb{C}^{m}$, weights $\lambda$, $\mu$, and $\nu$, upper bound $L$, number of learning iterations $K$, and number of outer iterations $M$.\\
 \textbf{Outputs\;:} \:\:\: reconstructed image $\mathbf{y}^{M}$, learned dictionary $\mathbf{D}^{M}$, and learned coefficients of patches $\mathbf{X}^{M}$. \\
\textbf{Initial Estimates:} $\left ( \mathbf{y}^{0}, \mathbf{D}^{0}, \mathbf{X}^{0} \right )$, with $\mathbf{C}^{0} = \left ( \mathbf{X}^{0} \right )^{H}$.\\
\textbf{For \;$t$ = $1:$ $M$ repeat}\\
\begin{enumerate}
\item Form $\mathbf{Y}^{t-1}=\left [ \mathbf{P}_{1}\mathbf{y}^{t-1} \mid \mathbf{P}_{2}\mathbf{y}^{t-1} \mid ... \mid \mathbf{P}_{N}\mathbf{y}^{t-1}  \right ]$. \vspace{0.04in}
\item \textbf{Dictionary Learning:}  Set $\left ( \mathbf{D}^{t}, \mathbf{C}^{t} \right )$ to be the output after $K$ iterations of the SOUP-DILLO (for (P3)) or OS-DL (for (P4)) methods, with training data $\mathbf{Y}^{t-1}$ and initialization $\left ( \mathbf{D}^{t-1}, \mathbf{C}^{t-1} \right )$. Set $\mathbf{X}^{t} = \left ( \mathbf{C}^{t} \right )^{H}$.
\item \textbf{Image Update:} Update $\mathbf{y}^{t}$ by solving \eqref{bcs10} using a direct approach (e.g., \eqref{bcs13}) or using CG.
\end{enumerate}\\
\textbf{End} \\
\hline
\end{tabular}
\caption{The SOUP-DILLO and SOUP-DILLI image reconstruction algorithms for Problems (P3) and (P4), respectively. Superscript $t$ denotes the iterates. Parameter $L$ can be set very large in practice (e.g., $L \propto \left \| \mathbf{A}^{\dagger}\mathbf{z} \right \|_{2}$).} \label{im6p}
\end{figure}

Fig. \ref{im6p} shows the algorithms for (P3) and (P4), which we refer to as the SOUP-DILLO and SOUP-DILLI image reconstruction algorithms, respectively.
The algorithms start with an initial $\left ( \mathbf{y}^{0}, \mathbf{D}^{0}, \mathbf{X}^{0} \right )$ (e.g., $\mathbf{y}^{0}=\mathbf{A}^{\dagger} \mathbf{z}$, and the other variables initialized as in Section \ref{sec3a2hhhh}).
In applications such as inpainting or single coil MRI, the cost per outer ($t$) iteration of the algorithms is typically dominated by the dictionary learning step, for which (assuming $J \propto n$) the cost scales as $O(K N n^{2})$, with $K$ being the number of inner iterations of dictionary learning.
On the other hand, recent image reconstruction methods involving K-SVD (e.g., DLMRI \cite{bresai}) have a worse corresponding cost per outer iteration of $O(K N n^{3})$. 

\section{Convergence Analysis} \label{sec4}

This section presents a convergence analysis of the algorithms for the non-convex Problems (P1)-(P4). 
Problem (P1) involves the non-convex $\ell_{0}$ penalty for sparsity, the unit $\ell_{2}$ norm constraints on atoms of $\mathbf{D}$, and the term $\begin{Vmatrix}
\mathbf{Y}- \sum_{j=1}^{J} \mathbf{d}_{j}\mathbf{c}_{j}^{H}
\end{Vmatrix}_{F}^{2}$ that is a non-convex function involving the products of multiple unknown vectors. 
The various algorithms discussed in Sections \ref{sec3} and \ref{sec7} are exact block coordinate descent methods for (P1)-(P4).
Due to the high degree of non-convexity involved, recent results on convergence of (exact) block coordinate descent methods \cite{tseng6} do not immediately apply (e.g., the assumptions in \cite{tseng6} such as block-wise quasiconvexity or other conditions do not hold here). 
More recent works \cite{xu222} on the convergence of block coordinate descent schemes also use assumptions (such as multi-convexity, etc.) that do not hold here.
While there have been recent works \cite{Attouchaa, Bolte33, emlie12b, abb12, hess2} studying the convergence of alternating proximal-type methods for non-convex problems, we focus on the exact block coordinate descent schemes of Sections \ref{sec3} and \ref{sec7} due to their simplicity.
We discuss the convergence of these algorithms to the critical points (or generalized stationary points \cite{vari1}) in the problems.
In the following, we present some definitions and notations, before stating the main results.

\subsection{Definitions and Notations} \label{sec4a}

A sequence $ \left \{ \mathbf{a}^{t} \right \}  \subset \mathbb{C}^{p}$ has an accumulation point $\mathbf{a}$, if there is a subsequence that converges to $\mathbf{a}$.
The constraints $\left \| \mathbf{d}_{j} \right \|_2 =1$, $1 \leq j \leq J$, in (P1) can instead be added as penalties in the cost by using barrier functions $\chi (\mathbf{d}_{j})$ (taking the value $+ \infty$ when the norm constraint is violated, and zero otherwise). The constraints $\left \| \mathbf{c}_{j} \right \|_{\infty} \leq L$, $1 \leq j \leq J$, in (P1), can also be similarly replaced with barrier penalties $\psi(\mathbf{c}_{j})$ $\forall$ $j$.
Then, we rewrite (P1) in unconstrained form with the following objective:
\begin{align} 
\nonumber & f(\mathbf{C}, \mathbf{D}) = f\left ( \mathbf{c}_{1}, \mathbf{c}_{2},..., \mathbf{c}_{J}, \mathbf{d}_{1}, \mathbf{d}_{2},..., \mathbf{d}_{J} \right ) =
\lambda^{2} \sum_{j=1}^{J} \left \| \mathbf{c}_{j} \right \|_{0}\\
& \;\; + \begin{Vmatrix}
\mathbf{Y}- \sum_{j=1}^{J} \mathbf{d}_{j}\mathbf{c}_{j}^{H}
\end{Vmatrix}_{F}^{2}  + \sum_{j=1}^{J} \chi (\mathbf{d}_{j})  + \sum_{j=1}^{J} \psi(\mathbf{c}_{j}).  \label{eqop32}
\end{align}
We rewrite (P2) similarly with an objective $\tilde{f}(\mathbf{C}, \mathbf{D})$ obtained by replacing the $\ell_{0}$ ``norm'' above with the $\ell_{1}$ norm, and dropping the penalties $\psi(\mathbf{c}_{j}) $.
We also rewrite (P3) and (P4) in terms of the variable $\mathbf{C} = \mathbf{X}^{H}$, and denote the corresponding unconstrained objectives (involving barrier functions) as $g(\mathbf{C}, \mathbf{D}, \mathbf{y})$ and $\tilde{g}(\mathbf{C}, \mathbf{D}, \mathbf{y})$, respectively.

The iterates computed in the $t$th outer iteration of SOUP-DILLO (or alternatively in OS-DL) are denoted 
by the pair of matrices $\left (  \mathbf{C}^{t}, \mathbf{D}^{t} \right )$.

\subsection{Results for (P1) and (P2)} \label{sec4b}

First, we present a convergence result for the SOUP-DILLO algorithm for (P1) in Theorem \ref{theorem2}.
Assume that the initial $(\mathbf{C}^{0}, \mathbf{D}^{0})$ satisfies the constraints in (P1).  

\begin{theorem}\label{theorem2} \vspace{0.02in}
Let $\left \{ \mathbf{C}^{t}, \mathbf{D}^{t} \right \}$ denote the bounded iterate sequence generated by the SOUP-DILLO Algorithm with data $\mathbf{Y}  \in \mathbb{C}^{n \times N}$ and initial $(\mathbf{C}^{0}, \mathbf{D}^{0})$. 
Then, the following results hold:
\begin{enumerate}[(i)]
\item The objective sequence  $\left \{ f^{t} \right \}$ with $f^{t} \triangleq f\left ( \mathbf{C}^{t}, \mathbf{D}^{t} \right )$ is monotone decreasing, and converges to a finite value, say $f^{*}=f^{*}(\mathbf{C}^{0}, \mathbf{D}^{0})$.
\item All the accumulation points of the iterate sequence are equivalent in the sense that they achieve the exact same value $f^{*}$ of the objective.
\item Suppose each accumulation point $\left ( \mathbf{C}, \mathbf{D}\right )$ of the iterate sequence is such that the matrix $\mathbf{B}$ with columns $\mathbf{b}_{j} = \mathbf{E}_{j}^{H}\mathbf{d}_{j}$ and $\mathbf{E}_{j} = \mathbf{Y} - \mathbf{D}\mathbf{C}^{H} + \mathbf{d}_{j}\mathbf{c}_{j}^{H}$, has no entry with magnitude $\lambda$. Then every accumulation point of the iterate sequence is a critical point of the objective $f(\mathbf{C}, \mathbf{D})$. Moreover, the two sequences with terms $ \begin{Vmatrix}
\mathbf{D}^{t} - \mathbf{D}^{t-1}
\end{Vmatrix}_{F}$ and $ \begin{Vmatrix}
\mathbf{C}^{t} - \mathbf{C}^{t-1}
\end{Vmatrix}_{F}$ respectively,  both converge to zero.
\end{enumerate}
\end{theorem}

Theorem \ref{theorem2} establishes that for an initial point $(\mathbf{C}^{0}, \mathbf{D}^{0})$, the \emph{bounded} iterate sequence in SOUP-DILLO is such that all its (compact set of) accumulation points achieve the same value $f^{*}$ of the objective. They are equivalent in that sense. 
In other words, the iterate sequence converges to an equivalence class of accumulation points.
The value of $f^{*}$ could vary with different initalizations.

Theorem \ref{theorem2} (Statement (iii)) also establishes that every accumulation point of the iterates is a critical point of $f(\mathbf{C}, \mathbf{D})$, i.e., for each initial $(\mathbf{C}^{0}, \mathbf{D}^{0})$, the iterate sequence converges to an equivalence class of critical points of $f$.
The results $ \begin{Vmatrix}
\mathbf{D}^{t} - \mathbf{D}^{t-1}
\end{Vmatrix}_{F} \to 0$ and $ \begin{Vmatrix}
\mathbf{C}^{t} - \mathbf{C}^{t-1}
\end{Vmatrix}_{F} \to 0$ also imply that the sparse approximation to the data $\mathbf{Z}^{t}=\mathbf{D}^{t}\left ( \mathbf{C}^{t} \right )^{H}$ satisfies $\begin{Vmatrix}
\mathbf{Z}^{t} - \mathbf{Z}^{t-1} 
\end{Vmatrix}_{F} \to 0$.  These are necessary but not sufficient conditions for the convergence of the entire sequences $\left \{ \mathbf{D}^{t} \right \}$, $\left \{ \mathbf{C}^{t} \right \}$, and $\left \{ \mathbf{Z}^{t} \right \}$.
The assumption on the entries of the matrix $\mathbf{B}$ in Theorem~\ref{theorem2} (i.e., $\left | b_{ji} \right | \neq \lambda$) is equivalent to assuming that for every $1\leq j \leq J$, there is a unique minimizer of $f$ with respect to $\mathbf{c}_{j}$ with all other variables fixed to their values in the accumulation point $(\mathbf{C}, \mathbf{D})$.

Although Theorem \ref{theorem2} uses a uniqueness condition with respect to each accumulation point (for Statement (iii)), the following conjecture postulates that provided the following Assumption 1 (that uses a probabilistic model for the data) holds, the uniqueness condition holds with probability $1$, i.e., the probability of a tie in assigning sparse codes is zero.

\textbf{Assumption 1.} The signals $\mathbf{y}_{i} \in \mathbb{C}^{n}$ for $1 \leq i \leq N$,  are  drawn independently from an absolutely continuous probability measure over the ball $S \triangleq \left \{ \mathbf{y} \in \mathbb{C}^{n} : \left \| \mathbf{y} \right \|_{2} \leq \beta_{0} \right \}$ for some $\beta_{0}>0$.

\newtheorem{conjecture}{Conjecture}

\begin{conjecture} \label{clstmetuneq} \vspace{0.02in}
Let Assumption 1 hold. Then, with probability 1, every accumulation point $\left ( \mathbf{C} , \mathbf{D} \right )$ of the iterate sequence in the SOUP-DILLO Algorithm is such that for each $1\leq j \leq J$, the minimizer of $f(\mathbf{c}_{1},...,\mathbf{c}_{j-1},\tilde{\mathbf{c}}_{j}, \mathbf{c}_{j+1},...,\mathbf{c}_{J}, \mathbf{d}_{1},...,\mathbf{d}_{J})$ with respect to $\tilde{\mathbf{c}}_{j}$ is unique.
\end{conjecture}
\vspace{0.04in}

If Conjecture \ref{clstmetuneq} holds, then every accumulation point of the iterates in SOUP-DILLO is immediately a critical point of $f(\mathbf{C}, \mathbf{D})$ with probability $1$.

We now briefly state the convergence result for the OS-DL method for (P2). The result is more of a special case of Theorem \ref{theorem2}. Here, the iterate sequence for an initial $(\mathbf{C}^{0}, \mathbf{D}^{0})$ converges directly (without additional conditions) to an equivalence class (i.e., corresponding to a common objective value $\tilde{f}^{*}=\tilde{f}^{*}(\mathbf{C}^{0}, \mathbf{D}^{0})$) of critical points of the objective $\tilde{f}(\mathbf{C}, \mathbf{D})$.

\begin{theorem}\label{theorem4} \vspace{0.02in}
Let $\left \{ \mathbf{C}^{t}, \mathbf{D}^{t} \right \}$ denote the bounded iterate sequence generated by the OS-DL Algorithm with data $\mathbf{Y}  \in \mathbb{C}^{n \times N}$ and initial $(\mathbf{C}^{0}, \mathbf{D}^{0})$. 
Then, the iterate sequence converges to an equivalence class of critical points of $\tilde{f}(\mathbf{C}, \mathbf{D})$, and $ \begin{Vmatrix}
\mathbf{D}^{t} - \mathbf{D}^{t-1}
\end{Vmatrix}_{F} \to 0$ and $ \begin{Vmatrix}
\mathbf{C}^{t} - \mathbf{C}^{t-1}
\end{Vmatrix}_{F} \to 0$ as $t \to \infty$.
\end{theorem}
\vspace{0.04in}

A brief proof of Theorem \ref{theorem2} is provided in the supplementary material. The proof for Theorem \ref{theorem4} is similar, as discussed in the supplement.

\subsection{Results for (P3) and (P4)} \label{sec4c}

First, we present the result for the SOUP-DILLO image reconstruction algorithm for (P3) in Theorem \ref{theorem5}.
We again assume that the initial $(\mathbf{C}^{0}, \mathbf{D}^{0})$ satisfies the constraints in the problem.  Recall that $\mathbf{Y}$ denotes the matrix with patches $\mathbf{P}_{i} \mathbf{y}$ for $1\leq i \leq N$, as its columns.

\begin{theorem}\label{theorem5} \vspace{0.02in}
Let $\left \{ \mathbf{C}^{t}, \mathbf{D}^{t}, \mathbf{y}^{t} \right \}$ denote the iterate sequence generated by the SOUP-DILLO image reconstruction Algorithm with measurements $\mathbf{z} \in \mathbb{C}^{m}$ and initial $(\mathbf{C}^{0}, \mathbf{D}^{0}, \mathbf{y}^{0})$. Then, the following results hold:
\begin{enumerate}[(i)]
\item The objective sequence  $\left \{ g^{t} \right \}$ with $g^{t} \triangleq g\left ( \mathbf{C}^{t}, \mathbf{D}^{t}, \mathbf{y}^{t} \right )$ is monotone decreasing, and converges to a finite value, say $g^{*}=g^{*}(\mathbf{C}^{0}, \mathbf{D}^{0}, \mathbf{y}^{0})$.
\item The iterate sequence is bounded, and all its accumulation points are equivalent in the sense that they achieve the exact same value $g^{*}$ of the objective.
\item Each accumulation point $(\mathbf{C}, \mathbf{D}, \mathbf{y})$ of the iterate sequence satisfies
\begin{equation}
 \mathbf{y} \in  \underset{\tilde{\mathbf{y}}}{\arg\min} \; \,  g(\mathbf{C}, \mathbf{D}, \tilde{\mathbf{y}})     \label{cnbcs4a}
\end{equation}
\item As $t \to \infty$, $ \left \| \mathbf{y}^{t} - \mathbf{y}^{t-1} \right \|_{2}$ converges to zero.
\item Suppose each accumulation point $(\mathbf{C}, \mathbf{D}, \mathbf{y})$ of the iterates is such that the matrix $\mathbf{B}$ with columns $\mathbf{b}_{j} = \mathbf{E}_{j}^{H}\mathbf{d}_{j}$ and $\mathbf{E}_{j} = \mathbf{Y} - \mathbf{D}\mathbf{C}^{H} + \mathbf{d}_{j}\mathbf{c}_{j}^{H}$, has no entry with magnitude $\lambda$. Then every accumulation point of the iterate sequence is a critical point of the objective $g$. Moreover, $ \begin{Vmatrix}
\mathbf{D}^{t} - \mathbf{D}^{t-1}
\end{Vmatrix}_{F} \to 0$ and $ \begin{Vmatrix}
\mathbf{C}^{t} - \mathbf{C}^{t-1}
\end{Vmatrix}_{F} \to 0$ as $t \to \infty$.
\end{enumerate}
\end{theorem}
\vspace{0.04in}

Statements (i) and (ii) of Theorem \ref{theorem5} establish that for each initial $(\mathbf{C}^{0}, \mathbf{D}^{0}, \mathbf{y}^{0})$, the bounded iterate sequence in the SOUP-DILLO image reconstruction algorithm converges to an equivalence class (common objective value) of accumulation points. 
Statements (iii) and (iv) establish that each accumulation point is a partial global minimizer (i.e., minimizer with respect to some variables while the rest are kept fixed) of $g\left ( \mathbf{C}, \mathbf{D}, \mathbf{y} \right )$ with respect to $\mathbf{y}$, and that $\begin{Vmatrix} \mathbf{y}^{t} - \mathbf{y}^{t-1} \end{Vmatrix}_{2} \to 0$.
Statement (v) shows that the iterates converge to the critical points of $g$. In fact, the accumulation points of the iterates can be shown to be partial global minimizers of $g\left ( \mathbf{C}, \mathbf{D}, \mathbf{y} \right )$ with respect to each column of $\mathbf{C}$ or $\mathbf{D}$. 
Statement (v) also establishes the properties $ \begin{Vmatrix}
\mathbf{D}^{t} - \mathbf{D}^{t-1}
\end{Vmatrix}_{F} \to 0$ and $ \begin{Vmatrix}
\mathbf{C}^{t} - \mathbf{C}^{t-1}
\end{Vmatrix}_{F} \to 0$. Similarly as in Theorem \ref{theorem2}, Statement (v) of Theorem \ref{theorem5} uses a uniqueness condition with respect to the accumulation points of the iterates.

Finally, we briefly state the convergence result for the SOUP-DILLI image reconstruction Algorithm for (P4). The result is a special version of Theorem \ref{theorem5}, where the iterate sequence for an initial $(\mathbf{C}^{0}, \mathbf{D}^{0}, \mathbf{y}^{0})$ converges directly (without additional conditions) to an equivalence class (i.e., corresponding to a common objective value $\tilde{g}^{*}=\tilde{g}^{*}(\mathbf{C}^{0}, \mathbf{D}^{0}, \mathbf{y}^{0})$) of critical points of the objective $\tilde{g}$.

\begin{theorem}\label{theorem6} \vspace{0.02in}
Let $\left \{ \mathbf{C}^{t}, \mathbf{D}^{t}, \mathbf{y}^{t} \right \}$ denote the iterate sequence generated by the SOUP-DILLI image reconstruction Algorithm for (P4) with measurements $\mathbf{z} \in \mathbb{C}^{m}$ and initial $(\mathbf{C}^{0}, \mathbf{D}^{0}, \mathbf{y}^{0})$.
Then, the iterate sequence converges to an equivalence class of critical points of $\tilde{g}(\mathbf{C}, \mathbf{D}, \mathbf{y})$. Moreover, $ \begin{Vmatrix}
\mathbf{D}^{t} - \mathbf{D}^{t-1}
\end{Vmatrix}_{F} \to 0$, $ \begin{Vmatrix}
\mathbf{C}^{t} - \mathbf{C}^{t-1}
\end{Vmatrix}_{F} \to 0$, and $\begin{Vmatrix} \mathbf{y}^{t} - \mathbf{y}^{t-1} \end{Vmatrix}_{2} \to 0$ as $t \to \infty$.
\end{theorem}
\vspace{0.04in}

A brief proof sketch for Theorems \ref{theorem5} and \ref{theorem6} is provided in the supplementary material.

The convergence results for the algorithms in Figs. \ref{im5p} and \ref{im6p} use the deterministic and cyclic ordering of the various updates (of variables). Whether one could generalize the results to other update orders (such as stochastic) is an interesting question that we leave for future work.

\begin{figure*}[!t]
\begin{center}
\begin{tabular}{ccccccc}
\includegraphics[height=0.864in]{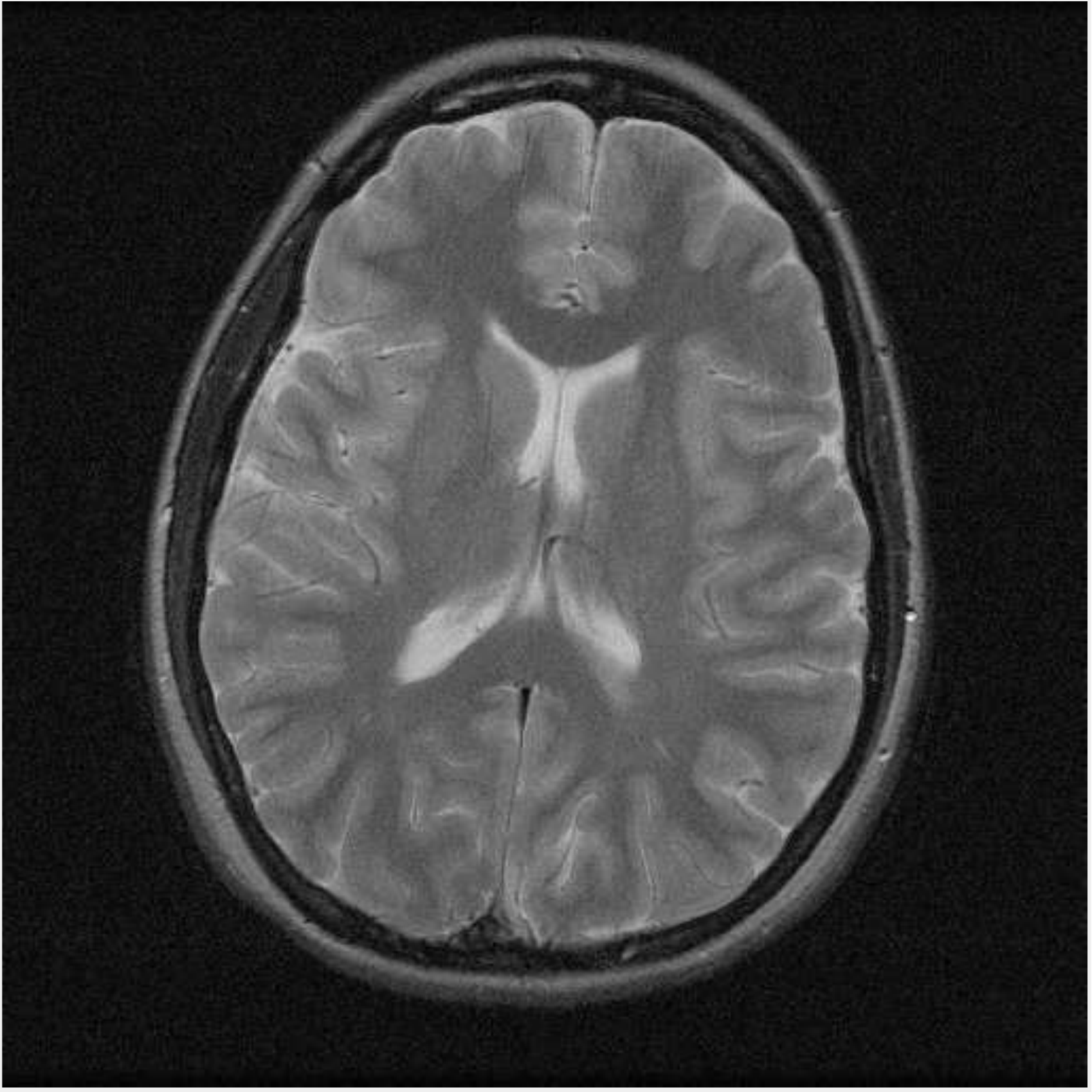}&
\includegraphics[height=0.864in]{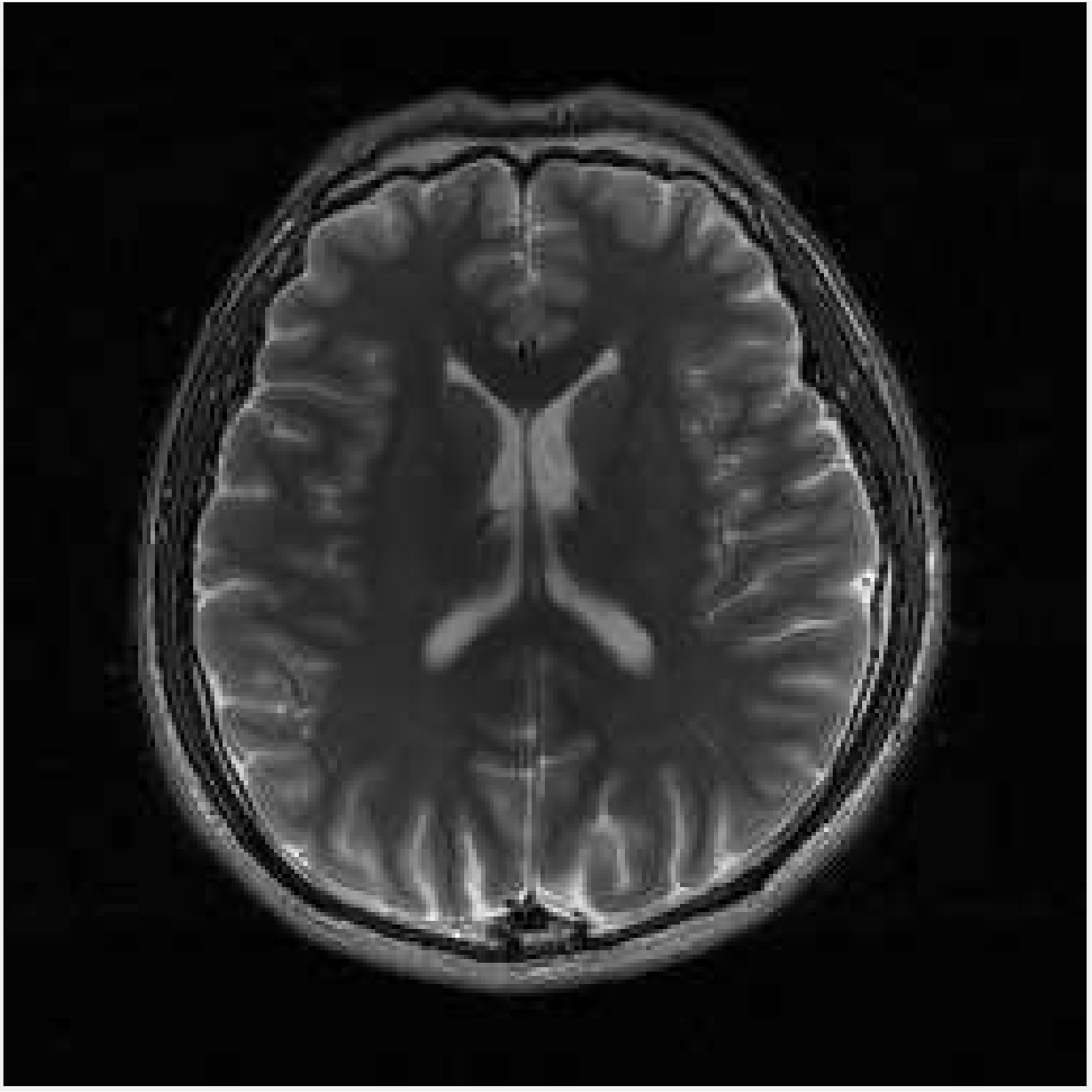} &
\includegraphics[height=0.864in]{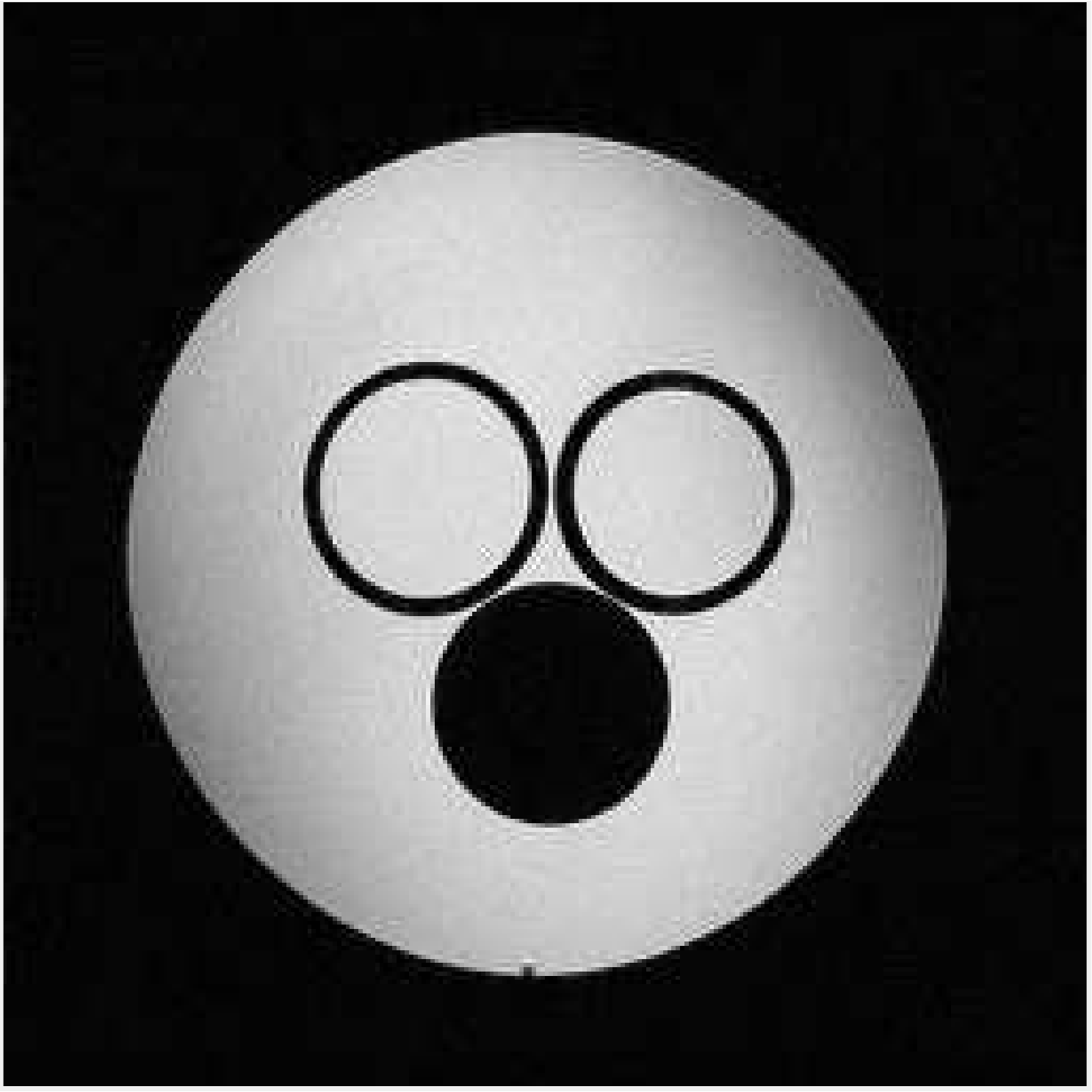}  &
\includegraphics[height=0.864in]{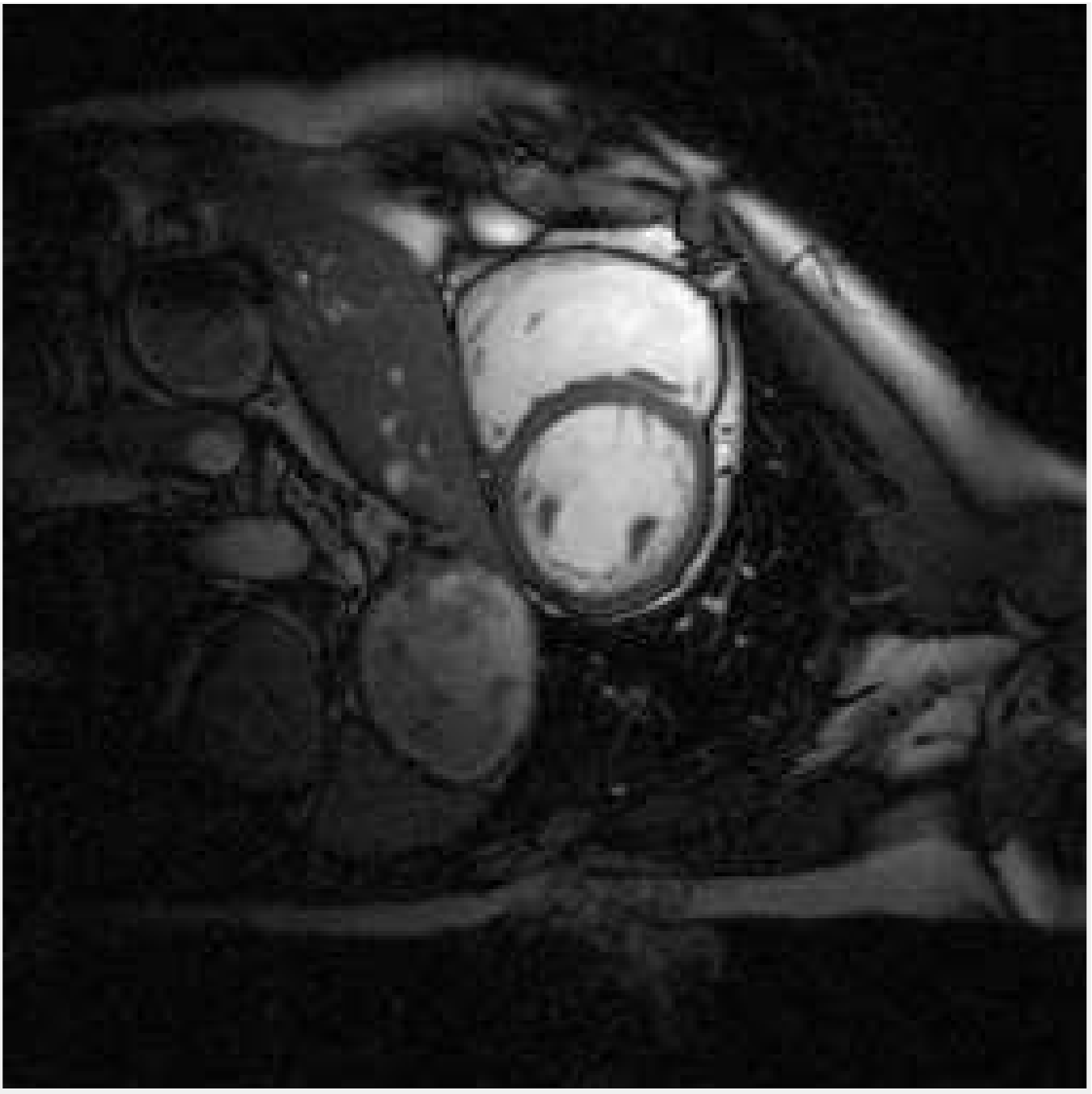}  &
\includegraphics[height=0.864in]{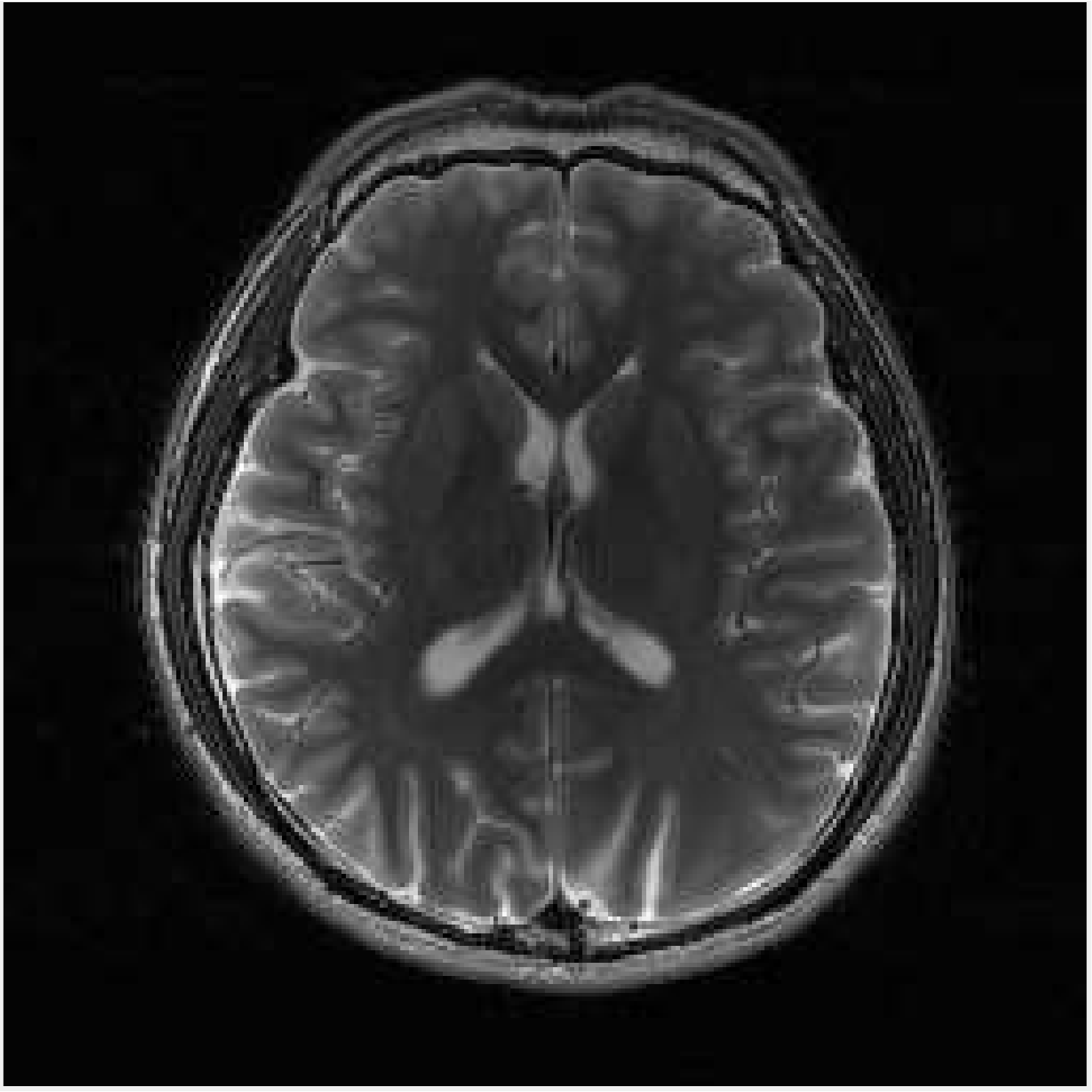} &
\includegraphics[height=0.864in]{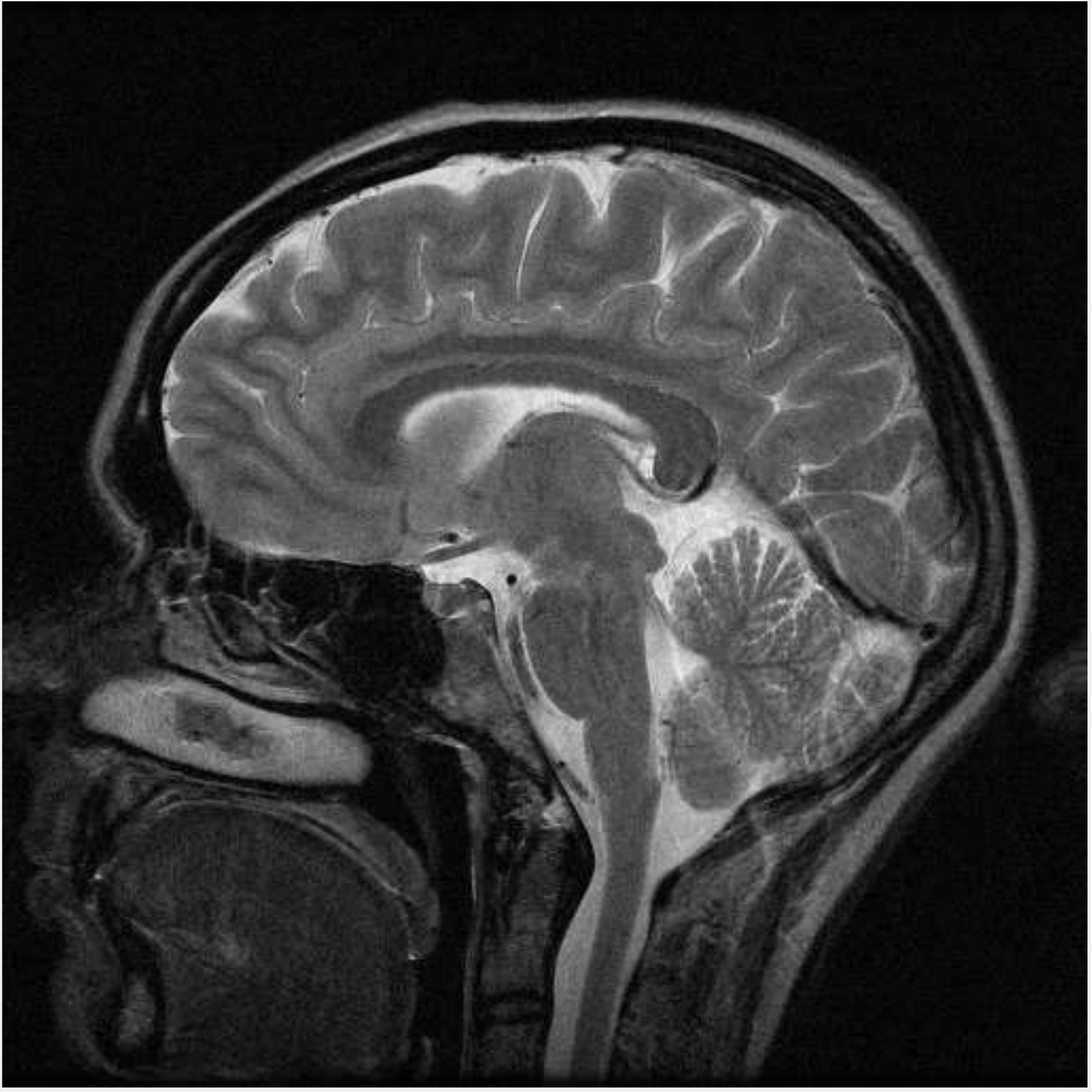} &
\includegraphics[height=0.864in]{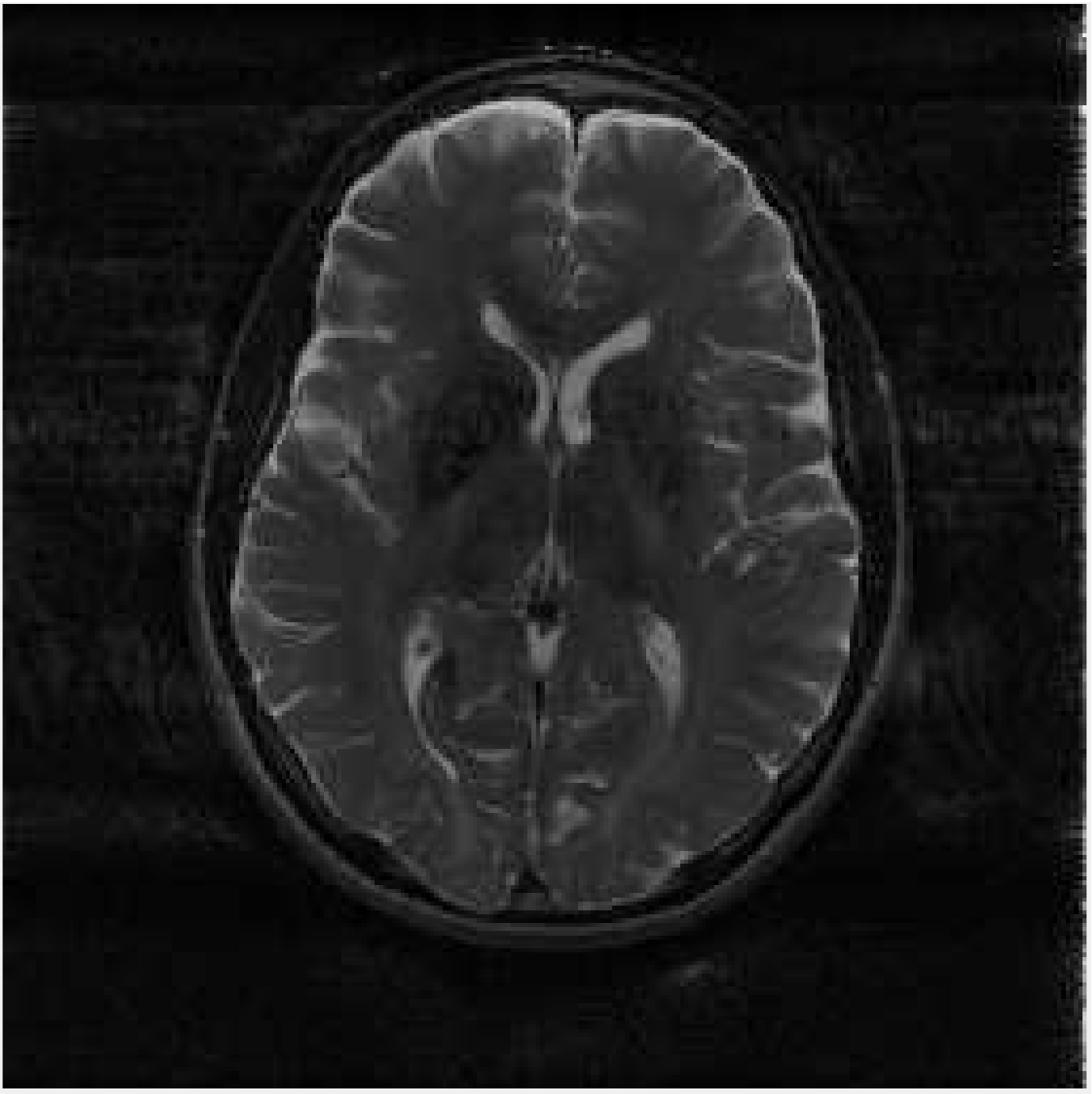} \\
(a) & (b) & (c) & (d) & (e)  & (f)  & (g) \\
\end{tabular}
\caption{Test data (magnitudes of the complex-valued MR data are displayed here). Image (a) is available at  \protect\url{http://web.stanford.edu/class/ee369c/data/brain.mat}. The images (b)-(e) are publicly available: (b) T2 weighted brain image \cite{PANOweb}, (c) water phantom \cite{Quweb1}, (d) cardiac image \cite{jcye1}, and (e) T2 weighted brain image \cite{Quweb2}. Image (f) is a reference sagittal brain slice provided by Prof. Michael Lustig, UC Berkeley. Image (g) is a complex-valued reference SENSE reconstruction of 32 channel fully-sampled Cartesian axial data from a standard spin-echo sequence. Images (a) and (f) are $512 \times 512$, while the rest are $256 \times 256$. The images (b) and (g) have been rotated clockwise by 90$^{\circ}$ here for display purposes. In the experiments, we use the actual orientations.}
\label{im1bcs}
\end{center}
\end{figure*}

\section{Numerical Experiments}
\label{sec5}

\subsection{Framework} \label{sec5a}

This section presents numerical results illustrating the convergence behavior as well as the usefulness of the proposed methods in applications such as sparse data representation and inverse problems. 
An empirical convergence study of the dictionary learning methods is included in the supplement.
We used a large $L=10^{8}$ in all experiments and the $\ell_{\infty}$ constraints were never active.

Section \ref{sec5c} illustrates the quality of sparse data representations obtained using the SOUP-DILLO method, where we consider data formed using vectorized 2D patches of natural images. 
We compare the sparse representation quality obtained with SOUP-DILLO to that obtained with OS-DL (for (P2)), and the recent proximal alternating dictionary learning (which we refer to as PADL) algorithm for (P1) \cite{bao1, bao2} (Algorithm 2 in \cite{bao2}). We used the publicly available implementation of the PADL method \cite{bao4}, and implemented OS-DL in a similar (memory efficient) manner as in Fig. \ref{im5p}.
We measure the quality of trained sparse representation of data $\mathbf{Y}$ using the \emph{normalized sparse representation error} (NSRE) $\left \| \mathbf{Y} - \mathbf{D}\mathbf{C}^{H} \right \|_{F}/$ $\left \| \mathbf{Y} \right \|_{F}$.

Results obtained using the SOUP-DILLO (learning) algorithm for image denoising are reported in \cite{sairajfes}. We have briefly discussed these results in Section \ref{app1} for completeness.
In the experiments of this work, we focus on general inverse problems involving non-trivial sensing matrices $\mathbf{A}$, where we use the iterative dictionary-blind image reconstruction algorithms discussed in Section \ref{sec7}.
In particular, we consider blind compressed sensing MRI \cite{bresai}, where $\mathbf{A} = \mathbf{F}_{\mathrm{u}} $, the undersampled Fourier encoding matrix.
Sections \ref{sec5d} and \ref{sec5e} examine the empirical convergence behavior and usefulness of the SOUP-DILLO and SOUP-DILLI image reconstruction algorithms for Problems (P3) and (P4), for blind compressed sensing MRI.
We refer to our algorithms for (P3) and (P4) for (dictionary-blind) MRI as SOUP-DILLO MRI and SOUP-DILLI MRI, respectively. Unlike recent synthesis dictionary learning-based works \cite{bresai, wangying} that involve computationally expensive algorithms with no convergence analysis, our algorithms for (P3) and (P4) are efficient and have proven convergence guarantees.

Figure \ref{im1bcs} shows the data (normalized to have unit peak pixel intensity) used in Sections \ref{sec5d} and \ref{sec5e}. 
In our experiments, we simulate undersampling of k-space with variable density 2D random sampling (feasible when data corresponding to multiple slices are jointly acquired, and the readout direction is perpendicular to image plane) \cite{bresai}, or using Cartesian sampling with variable density random phase encodes (1D random).
We compare the reconstructions from undersampled measurements provided by SOUP-DILLO MRI and SOUP-DILLI MRI to those provided by the benchmark DLMRI method \cite{bresai} that learns adaptive overcomplete dictionaries using K-SVD in a dictionary-blind image reconstruction framework.
We also compare to the non-adaptive Sparse MRI method \cite{lustig} that uses wavelets and total variation sparsity, the PANO method \cite{Qu2014843} that exploits the non-local similarities between image patches, and the very recent FDLCP method \cite{zhan33} that uses learned multi-class unitary dictionaries.
Similar to prior work \cite{bresai}, we employ the peak-signal-to-noise ratio (PSNR) to measure the quality of MR image reconstructions. The PSNR (expressed in decibels (dB)) is computed as the ratio of the peak intensity value of a reference image to the root mean square reconstruction error (computed between image magnitudes) relative to the reference.

All our algorithm implementations were coded in Matlab R2015a.
The computations in Section \ref{sec5c} were performed with an Intel Xeon CPU X3230 at 2.66 GHz and 8 GB memory, employing a 64-bit Windows 7 operating system.
The computations in Sections \ref{sec5d} and \ref{sec5e} were performed with an Intel Core i7 CPU at 2.6 GHz and 8 GB memory, employing a 64-bit Windows 7 operating system.
A link to software to reproduce results in this work will be provided at \url{http://web.eecs.umich.edu/~fessler/}.

\subsection{Adaptive Sparse Representation of Data} \label{sec5c}

Here, we extracted $3 \times 10^{4}$ patches of size $8 \times 8$ from randomly chosen locations in the $512 \times 512$ standard images Barbara, Boats, and Hill. For this data, we learned dictionaries of size $64 \times 256$ for various choices of the parameter $\lambda$ in (P1) (i.e., corresponding to a variety of solution sparsity levels). 
The initial estimate for $\mathbf{C}$ in SOUP-DILLO is an all-zero matrix, and the initial estimate for $\mathbf{D}$ is the overcomplete DCT \cite{elad2, el2}.
We measure the quality (performance) of adaptive data approximations $\mathbf{D}
\mathbf{C}^{H}$ using the NSRE metric. We also learned dictionaries using the recent methods for sparsity penalized dictionary learning in \cite{bao1, sadeg33}. All learning methods were initialized the same way.
We are interested in the NSRE versus sparsity trade-offs achieved by different learning methods for the $3 \times 10^{4}$ image patches (rather than for separate test data)\footnote{This study is useful because in the dictionary-blind image reconstruction framework of this work, the dictionaries are adapted without utilizing separate training data. Methods that provide sparser adaptive representations of the underlying data also typically provide better image reconstructions in that setting \cite{sravTCI1}.}.

First, we compare the NSRE values achieved by SOUP-DILLO to those obtained using the recent PADL (for (P1)) approach \cite{bao1, bao2}. Both the SOUP-DILLO and PADL methods were simulated for 30 iterations for an identical set of $\lambda$ values in (P1). We did not observe any marked improvements in performance with more iterations of learning. Since the PADL code \cite{bao4} outputs only the learned dictionaries, we performed 60 iterations of block coordinate descent (over the $\mathbf{c}_{j}$'s in (P1)) to obtain the sparse coefficients with the learned dictionaries.
Figs. \ref{im3}(a) and \ref{im3}(b) show the NSREs and sparsity factors obtained in SOUP-DILLO, and with learned PADL dictionaries for the image patch data. The proposed SOUP-DILLO achieves both lower NSRE (improvements up to 0.8 dB over the PADL dictionaries) and lower net sparsity factors. Moreover, it also has much lower learning times (Fig. \ref{im3}(c)) than PADL.

\begin{figure}[!t]
\begin{center}
\begin{tabular}{cc}
\includegraphics[height=1.19in]{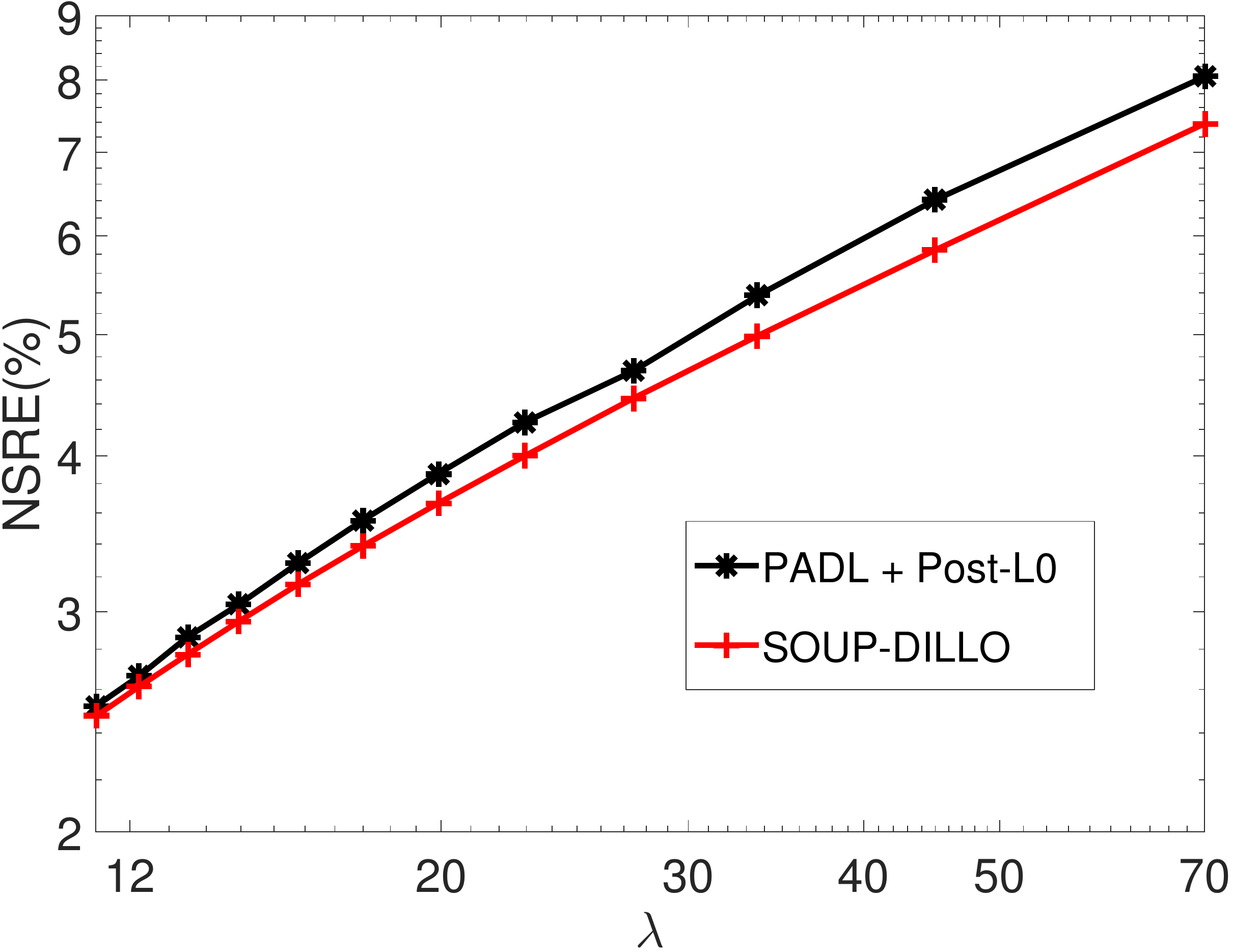}&
\includegraphics[height=1.19in]{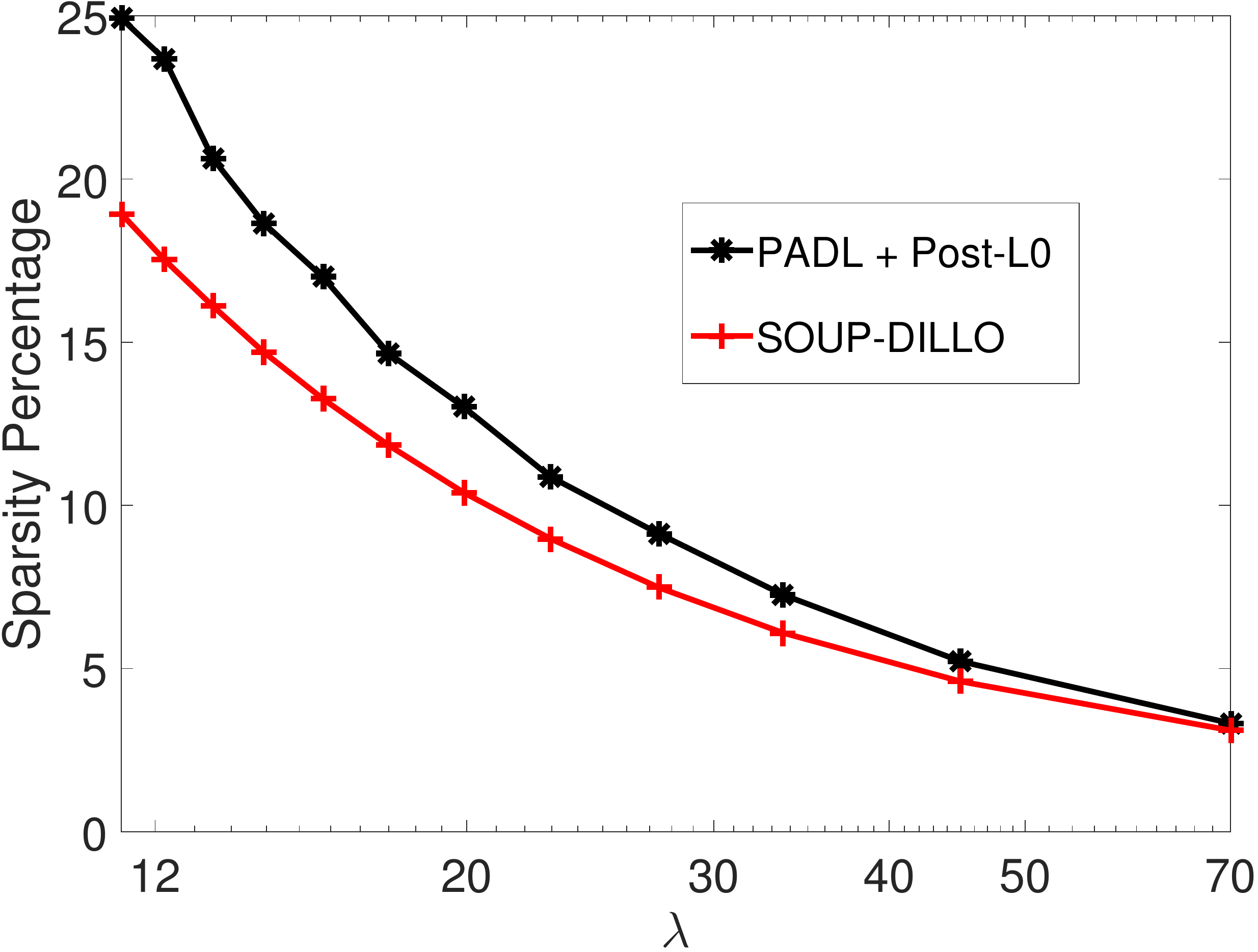}\\
(a) & (b) \\
\includegraphics[height=1.16in]{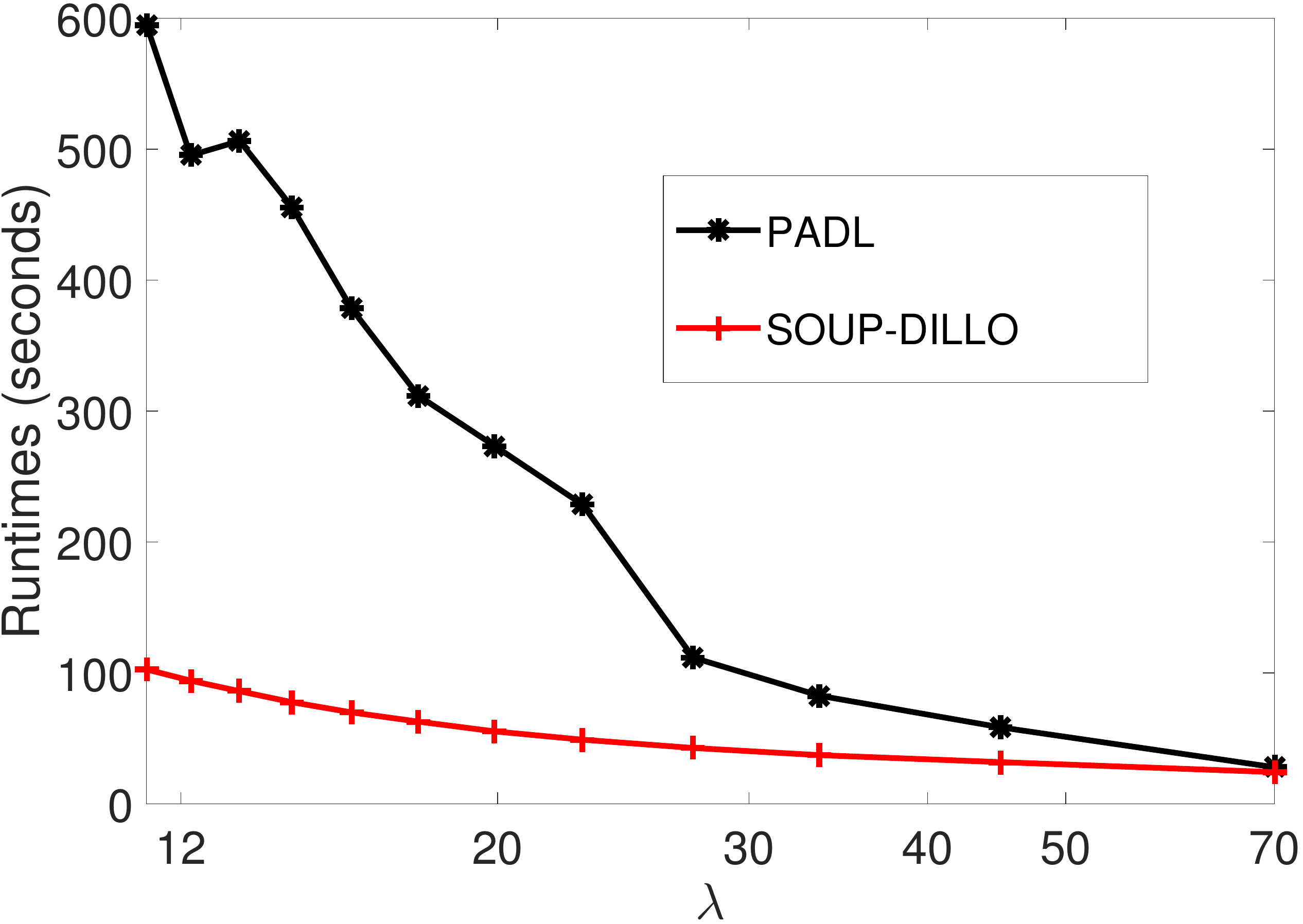}&
\includegraphics[height=1.17in]{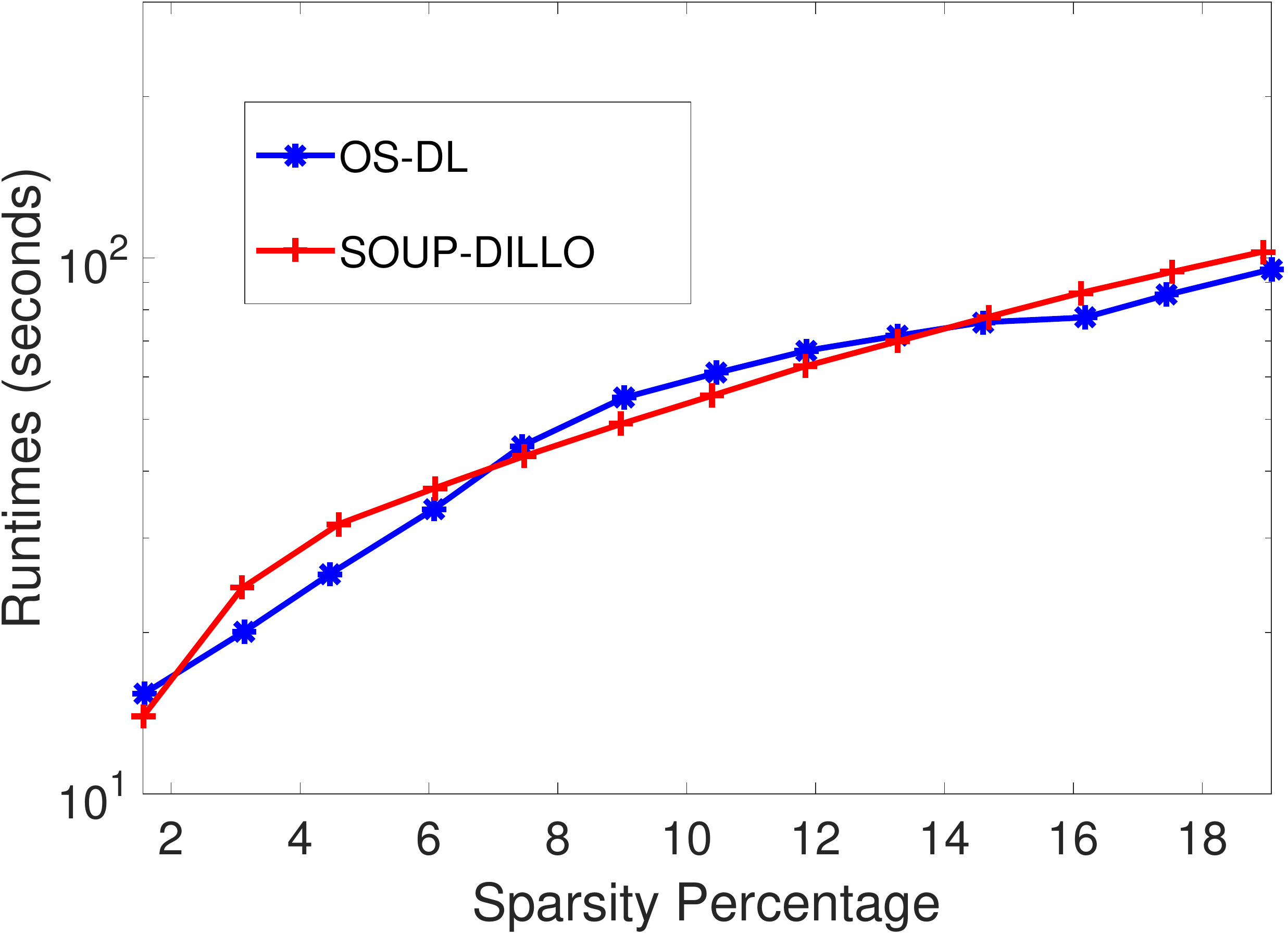}\\
(c) & (d) \\
\includegraphics[height=1.2in]{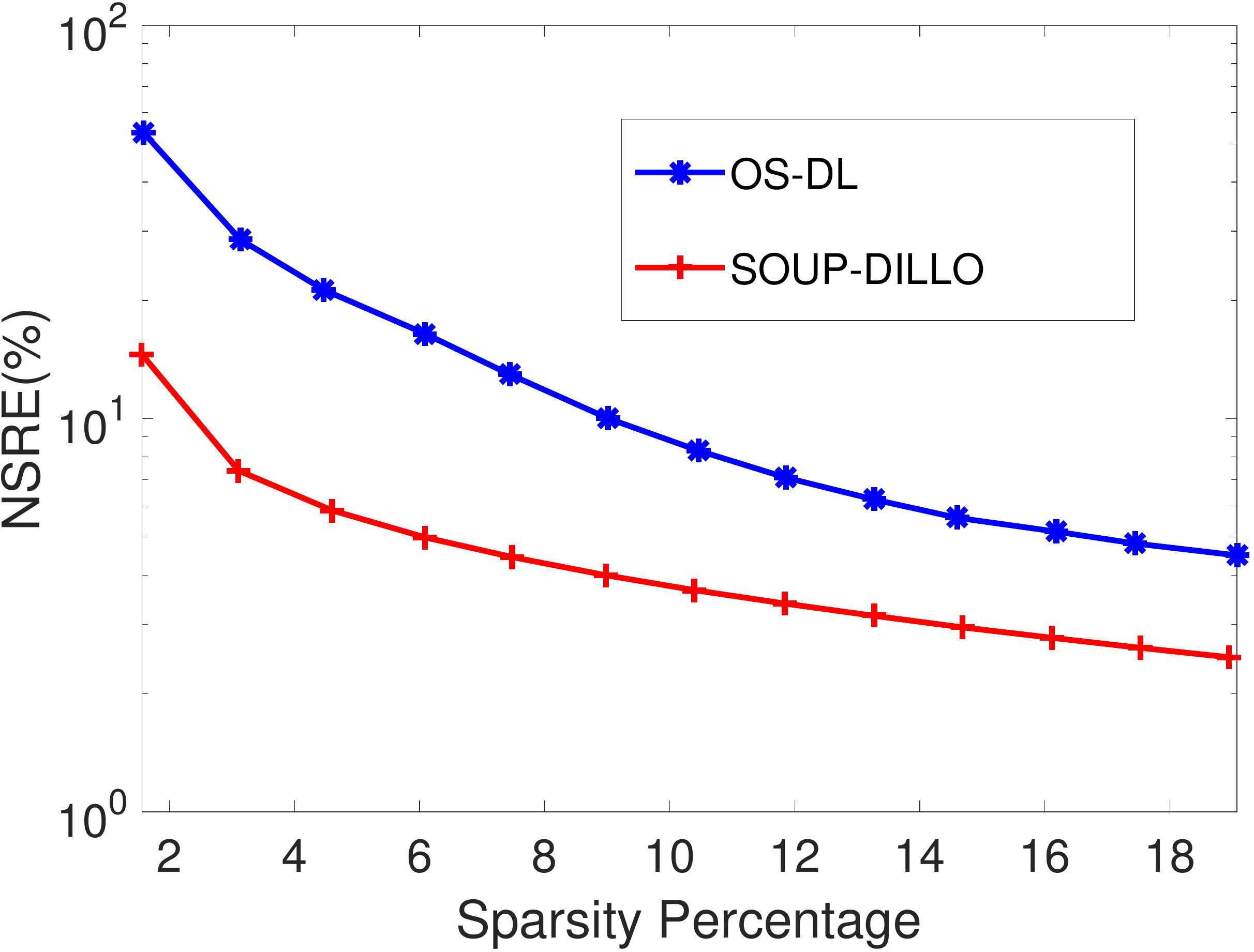} &
\includegraphics[height=1.2in]{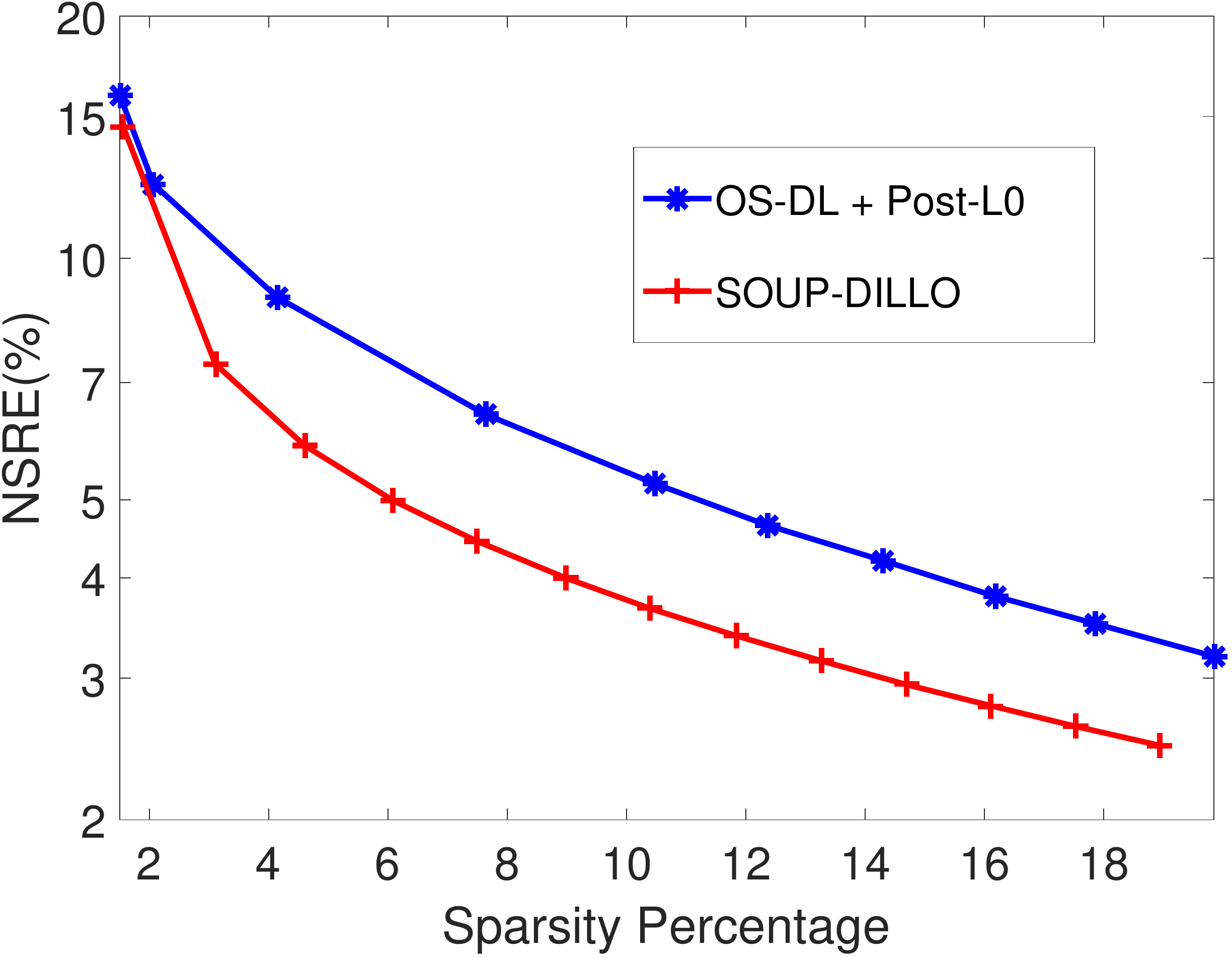}\\
(e)  & (f) \\
\end{tabular}
\caption{Comparison of dictionary learning approaches for adaptive sparse representation (NSRE and sparsity factors are expressed as percentages): (a) NSRE values for SOUP-DILLO at various $\lambda$ along with those obtained by performing $\ell_{0}$ block coordinate descent sparse coding (as in (P1)) using learned PADL \cite{bao1, bao2} dictionaries (denoted `Post-L0' in the plot legend);
(b) (net) sparsity factors for SOUP-DILLO at various $\lambda$ along with those obtained by performing $\ell_{0}$ block coordinate descent sparse coding using learned PADL \cite{bao1, bao2} dictionaries; (c) learning times for SOUP-DILLO and PADL; (d) learning times for SOUP-DILLO and OS-DL for various achieved (net) sparsity factors (in learning); (e) NSRE vs. (net) sparsity factors achieved within SOUP-DILLO and OS-DL; and (f)  NSRE vs. (net) sparsity factors achieved within SOUP-DILLO along with those obtained  by performing $\ell_{0}$ (block coordinate descent) sparse coding using learned OS-DL dictionaries.
}
\label{im3}
\end{center}
\end{figure}

Next, we compare the SOUP-DILLO and OS-DL methods for sparsely representing the same data. 
For completeness, we first show the NSRE versus sparsity trade-offs achieved during learning.
Here, we measured the sparsity factors (of $\mathbf{C}$) achieved within the schemes for various $\lambda$ and $\mu$ values in (P1) and (P2), and then compared the NSRE values achieved within SOUP-DILLO and OS-DL at similar (achieved) sparsity factor settings. OS-DL ran for 30 iterations, which was sufficient for good performance. 
Fig. \ref{im3}(e) shows the NSRE versus sparsity trade-offs achieved within the algorithms.
SOUP-DILLO clearly achieves significantly lower 
NSRE values at similar net sparsities than OS-DL. 
Since these methods are for the $\ell_{0}$ and $\ell_{1}$ learning problems respectively, we also took the learned sparse coefficients in OS-DL in Fig. \ref{im3}(e) and performed \emph{debiasing} \cite{matfiguer} by re-estimating the non-zero coefficient values (with supports fixed to the estimates in OS-DL) for each signal in a least squares sense to minimize the data fitting error. In this case, SOUP-DILLO in Fig. 4(e) provides an average NSRE improvement across various sparsities of 2.1 dB over OS-DL dictionaries.
Since both SOUP-DILLO and OS-DL involve similar types of operations, their runtimes (Fig. \ref{im3}(d)) for learning were quite similar.
Next, when the dictionaries learned by OS-DL for various sparsities in Fig. \ref{im3}(e) were used to estimate the sparse coefficients $\mathbf{C}$ in (P1) (using 60 iterations of $\ell_{0}$ block coordinate descent over the $\mathbf{c}_j$'s and choosing the corresponding $\lambda$ values in Fig. \ref{im3}(e)); the resulting representations $\mathbf{D} \mathbf{C}^{H}$ had on average
worse NSREs and usually more nonzero coefficients than SOUP-DILLO. Fig. \ref{im3}(f) plots the trade-offs.  For example, SOUP-DILLO provides 3.15 dB better NSRE than the learned OS-DL dictionary (used with $\ell_{0}$ sparse coding) at 7.5\% sparsity.
These results further
illustrate the benefits of the learned models in SOUP-DILLO.

\begin{figure*}[!t]
\begin{center}
\begin{tabular}{cccc}
\includegraphics[height=1.42in]{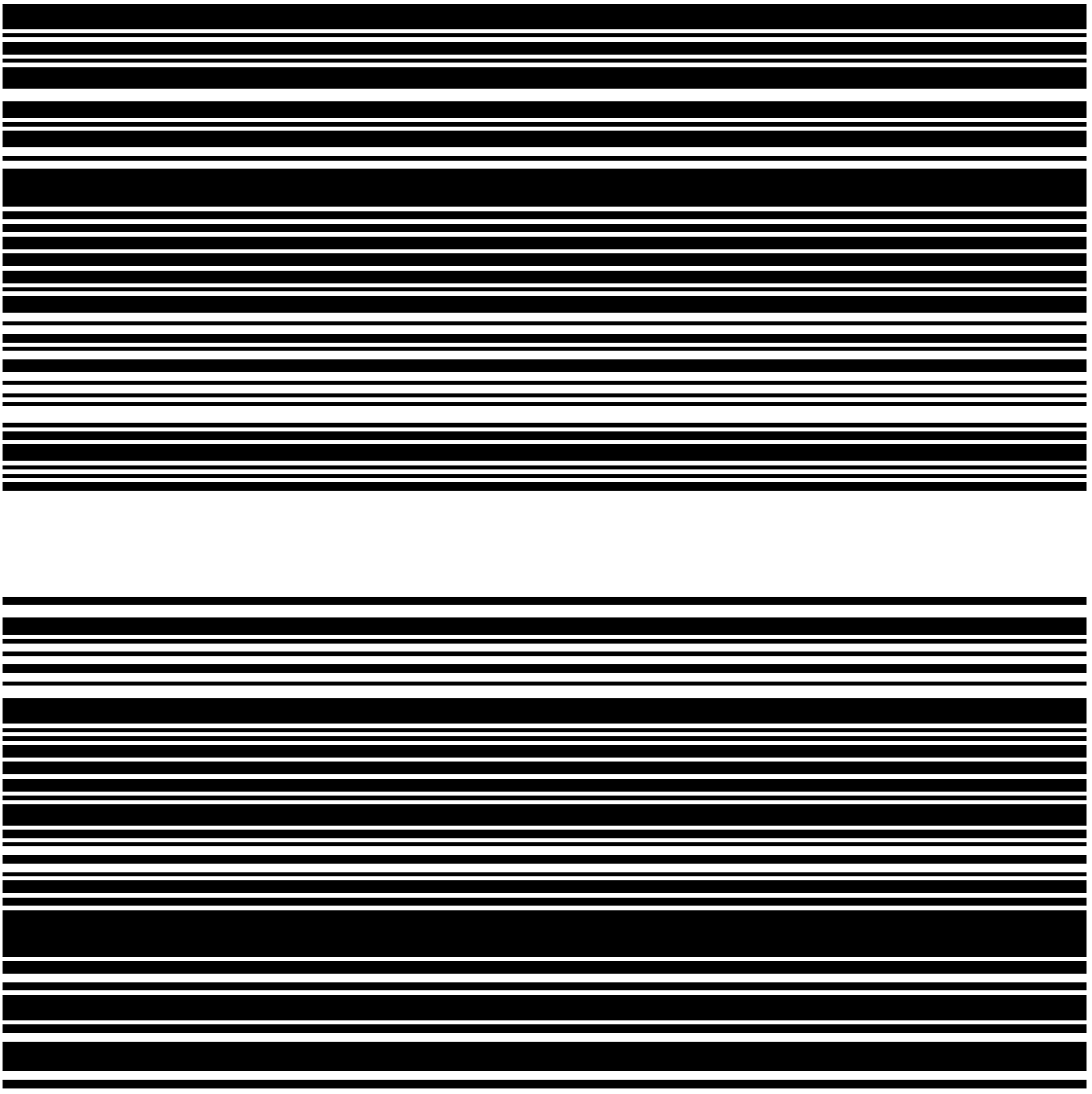}&
\includegraphics[height=1.42in]{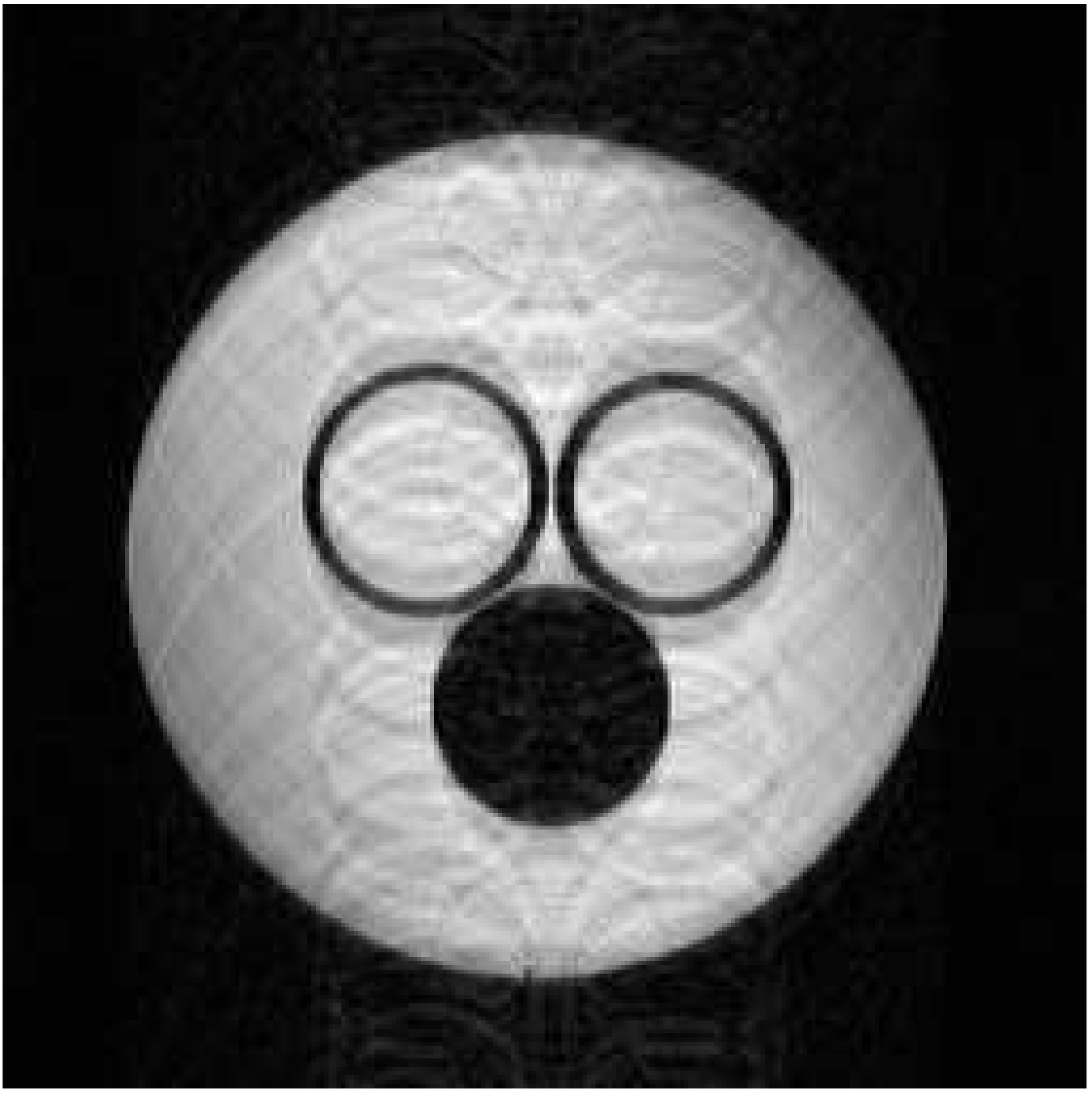}&
\includegraphics[height=1.42in]{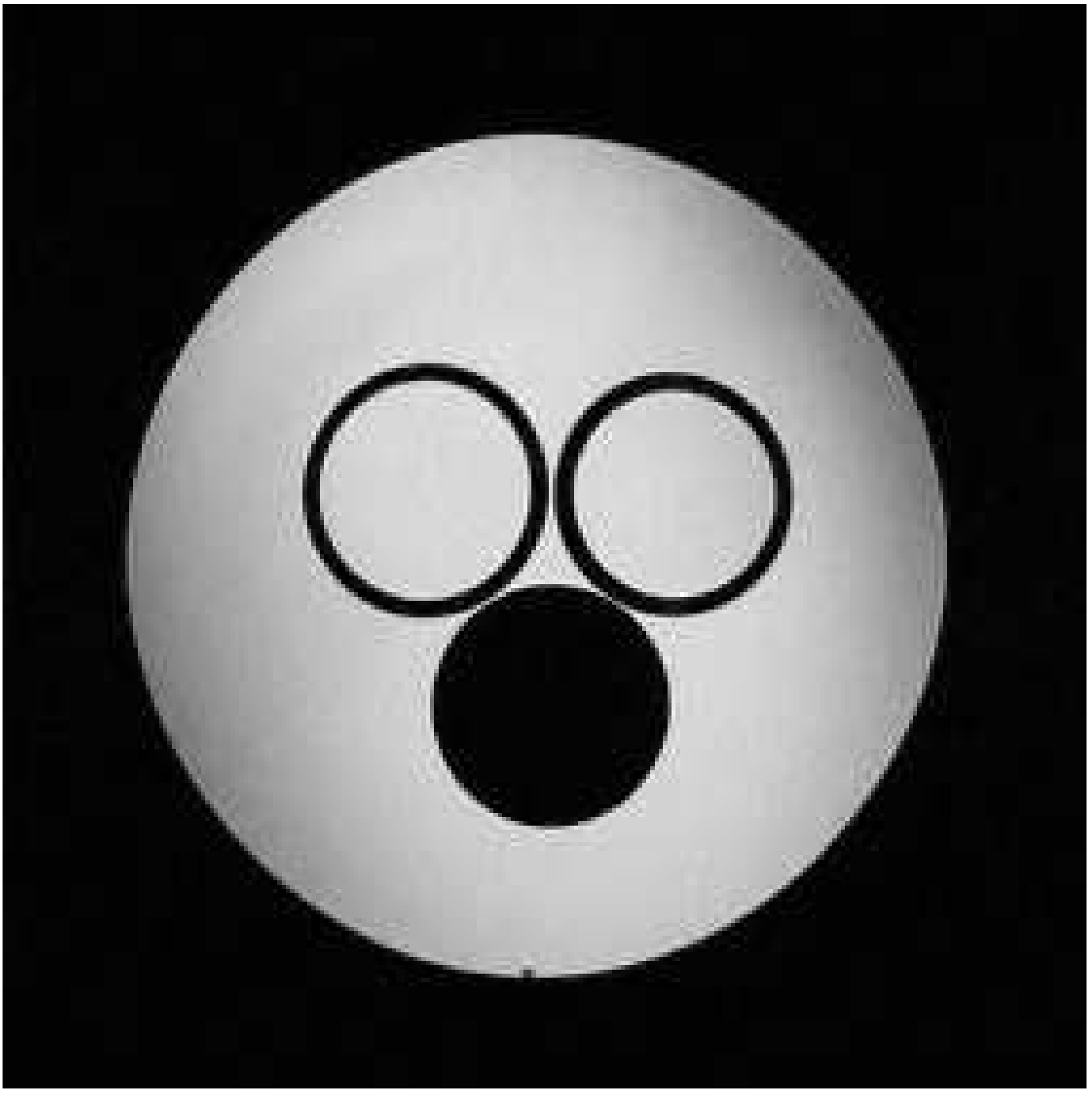}&
\includegraphics[height=1.42in]{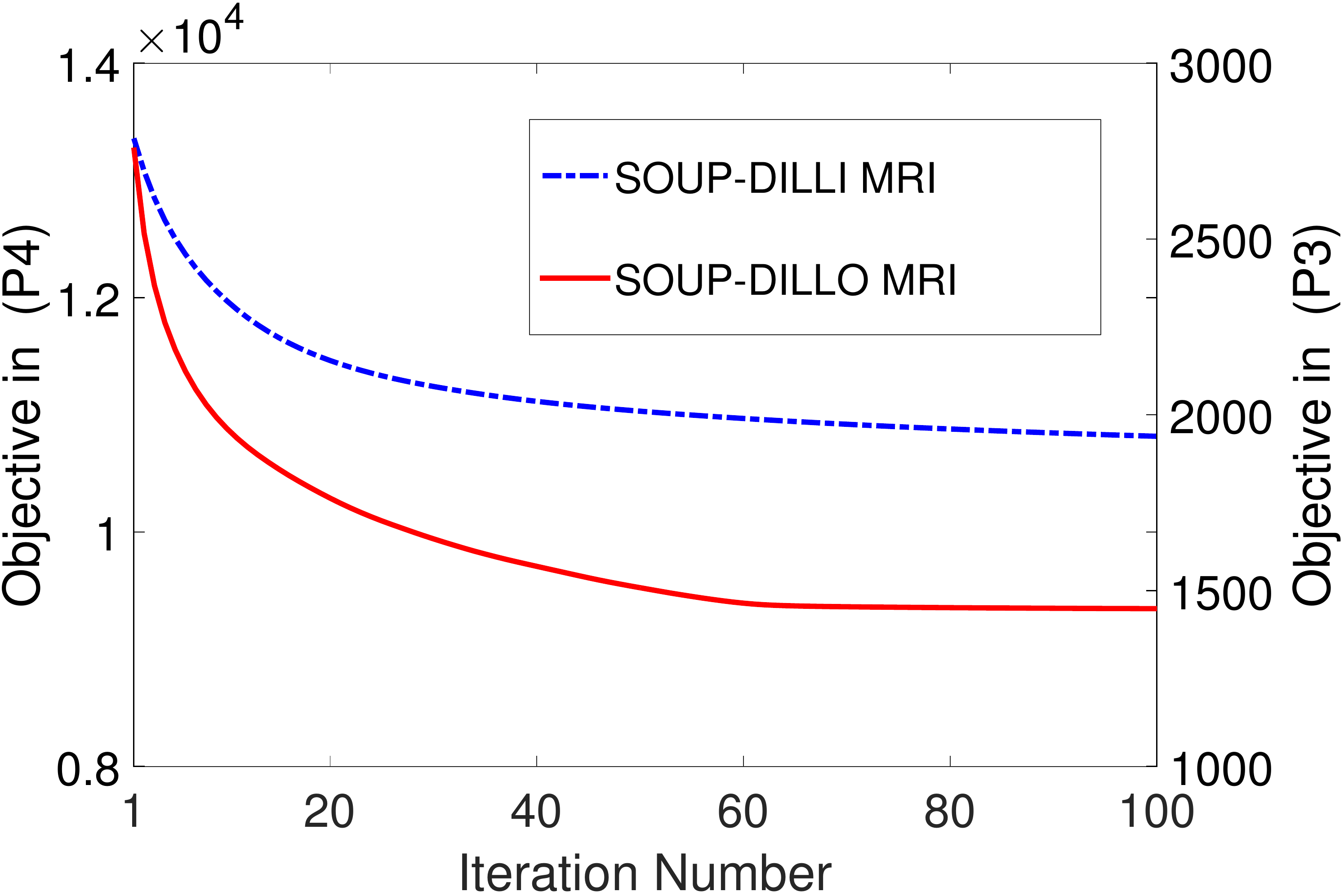}\\
(a) & (b) & (c) & (d) \\
\end{tabular}
\begin{tabular}{ccc}
\includegraphics[height=1.52in]{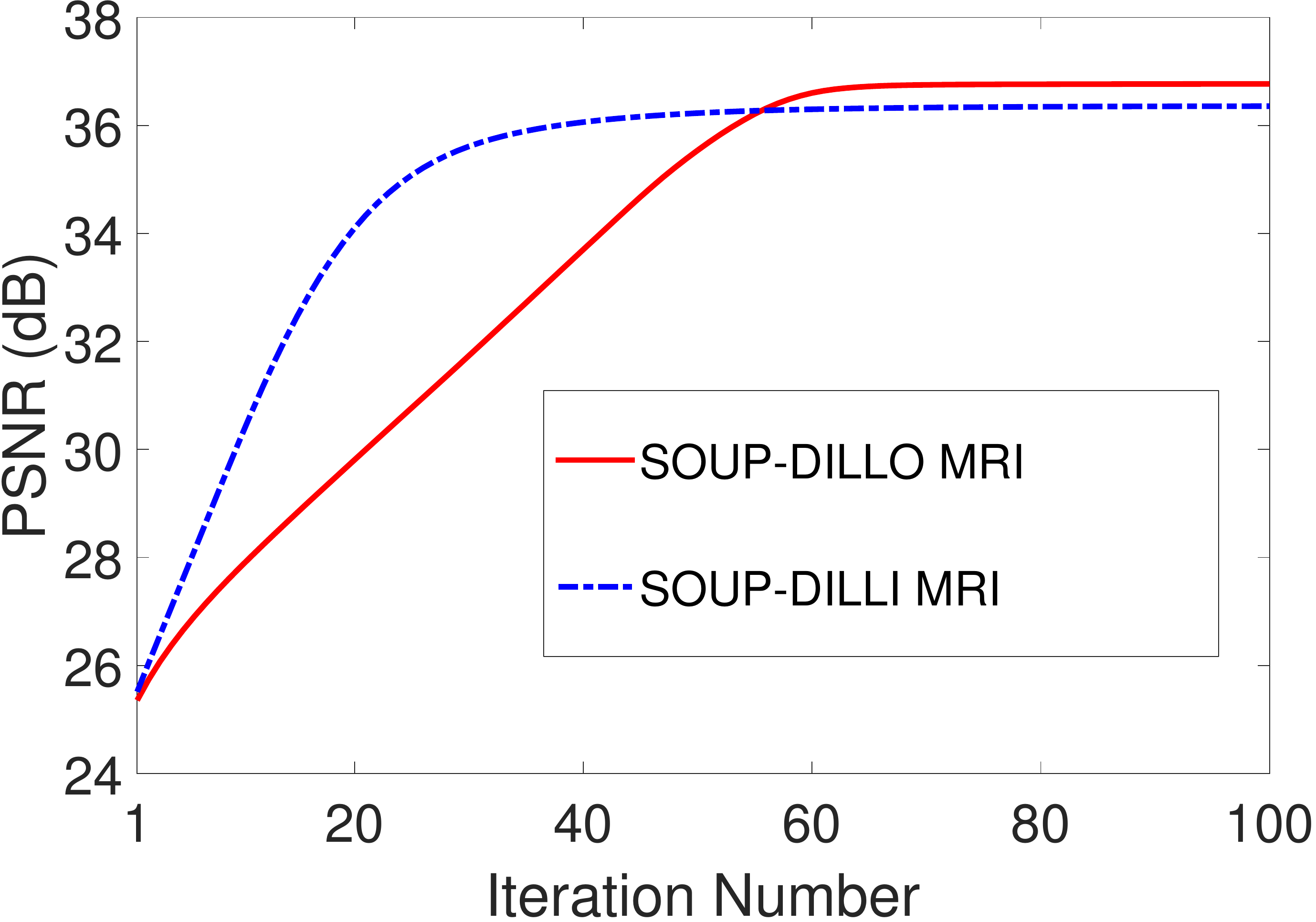}&
\hspace{0.0in}\includegraphics[height=1.482in]{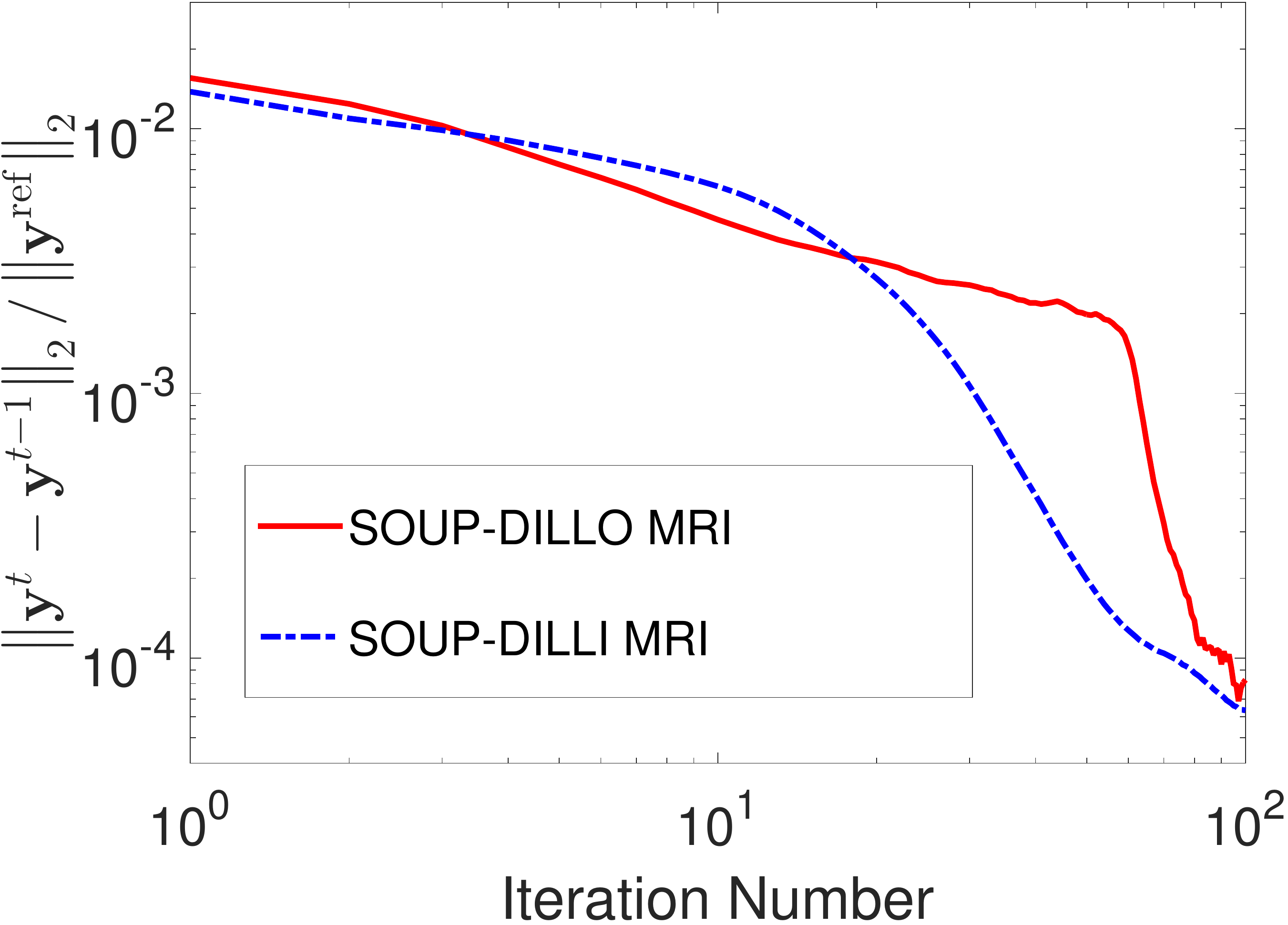}&
\hspace{0.0in}\includegraphics[height=1.568in]{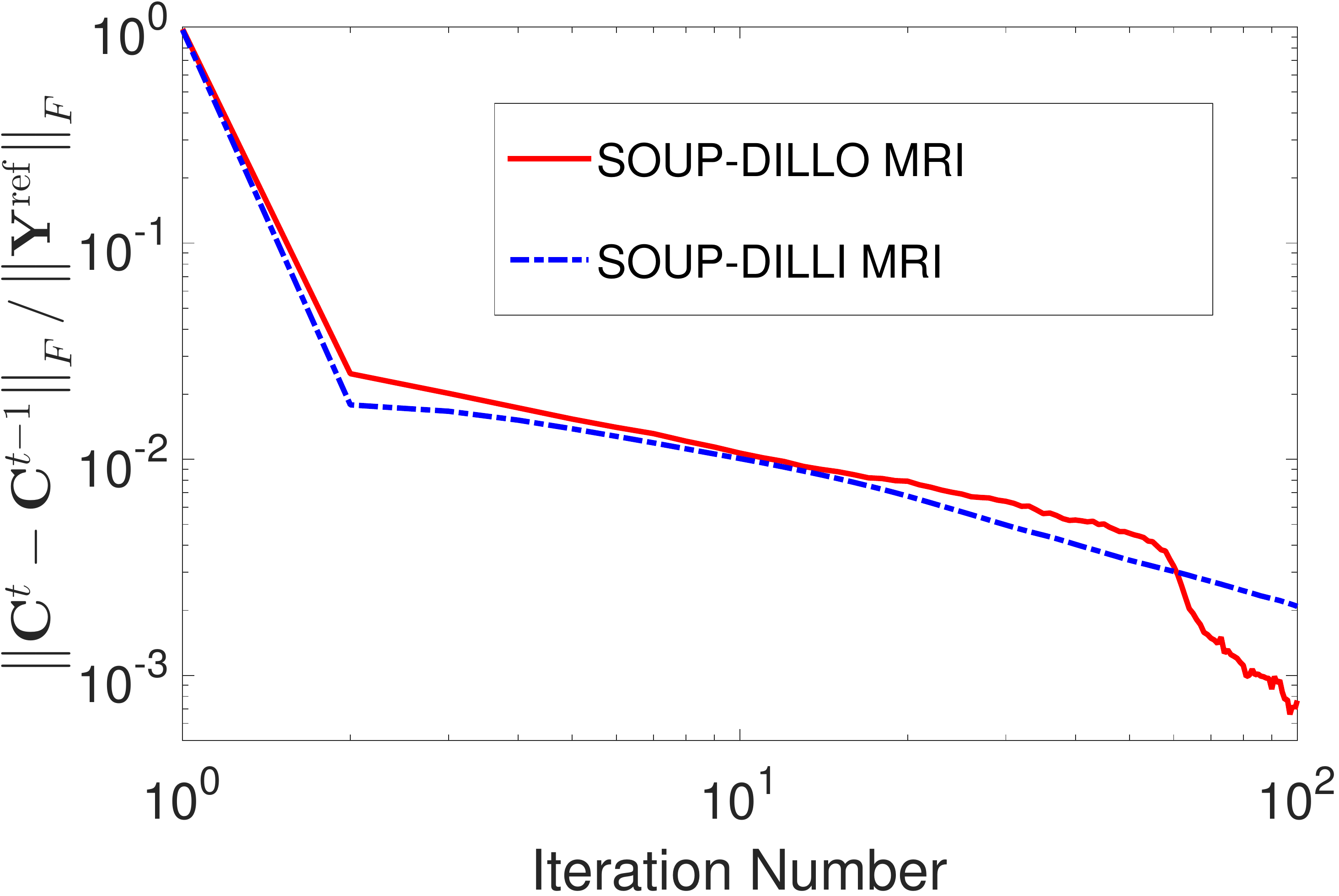}\\
(e) & \hspace{0.0in} (f) & \hspace{0.0in} (g)\\
\end{tabular}
\begin{tabular}{cccc}
\hspace{-0.0in}\includegraphics[height=1.462in]{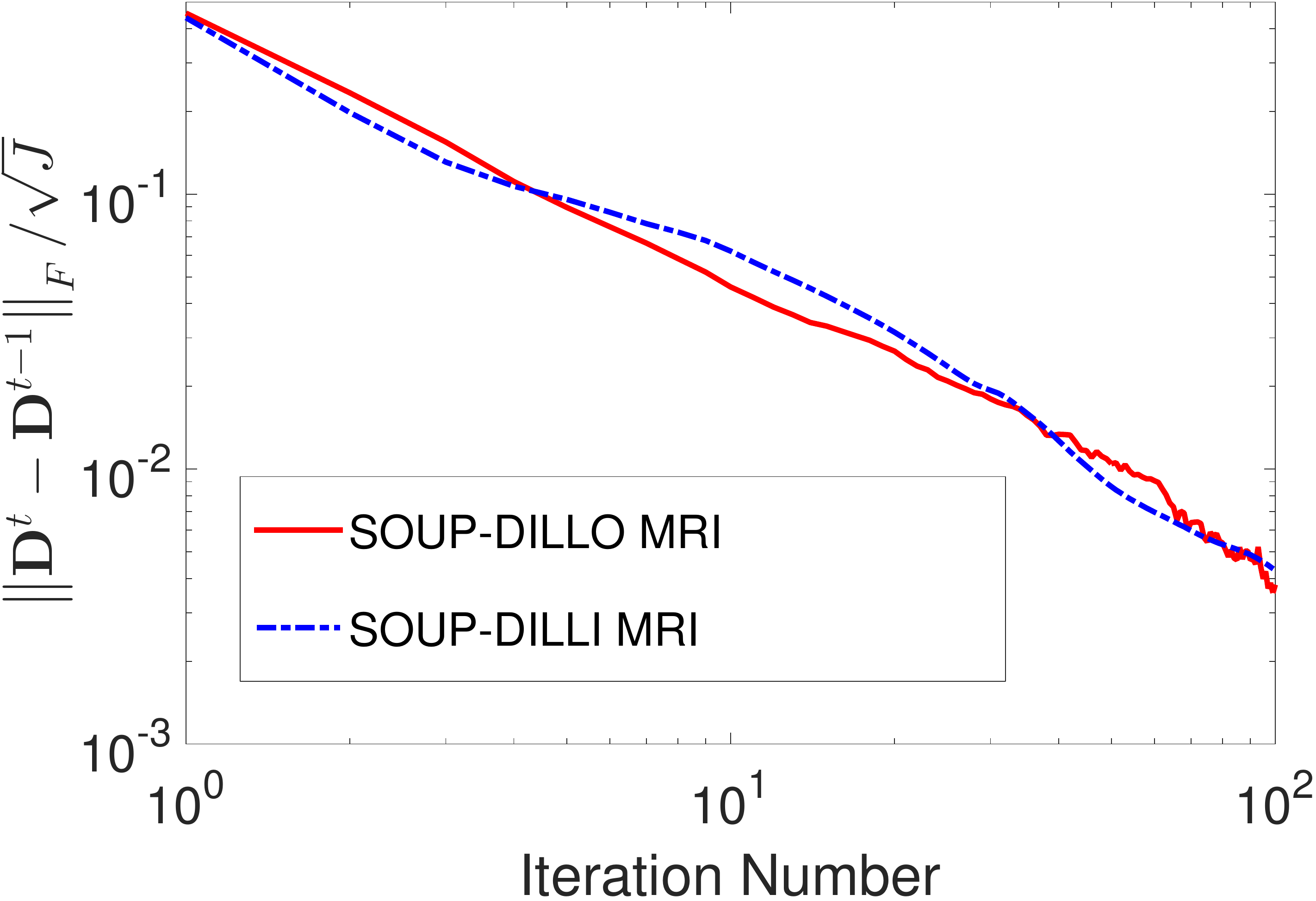}&
\includegraphics[height=1.45in]{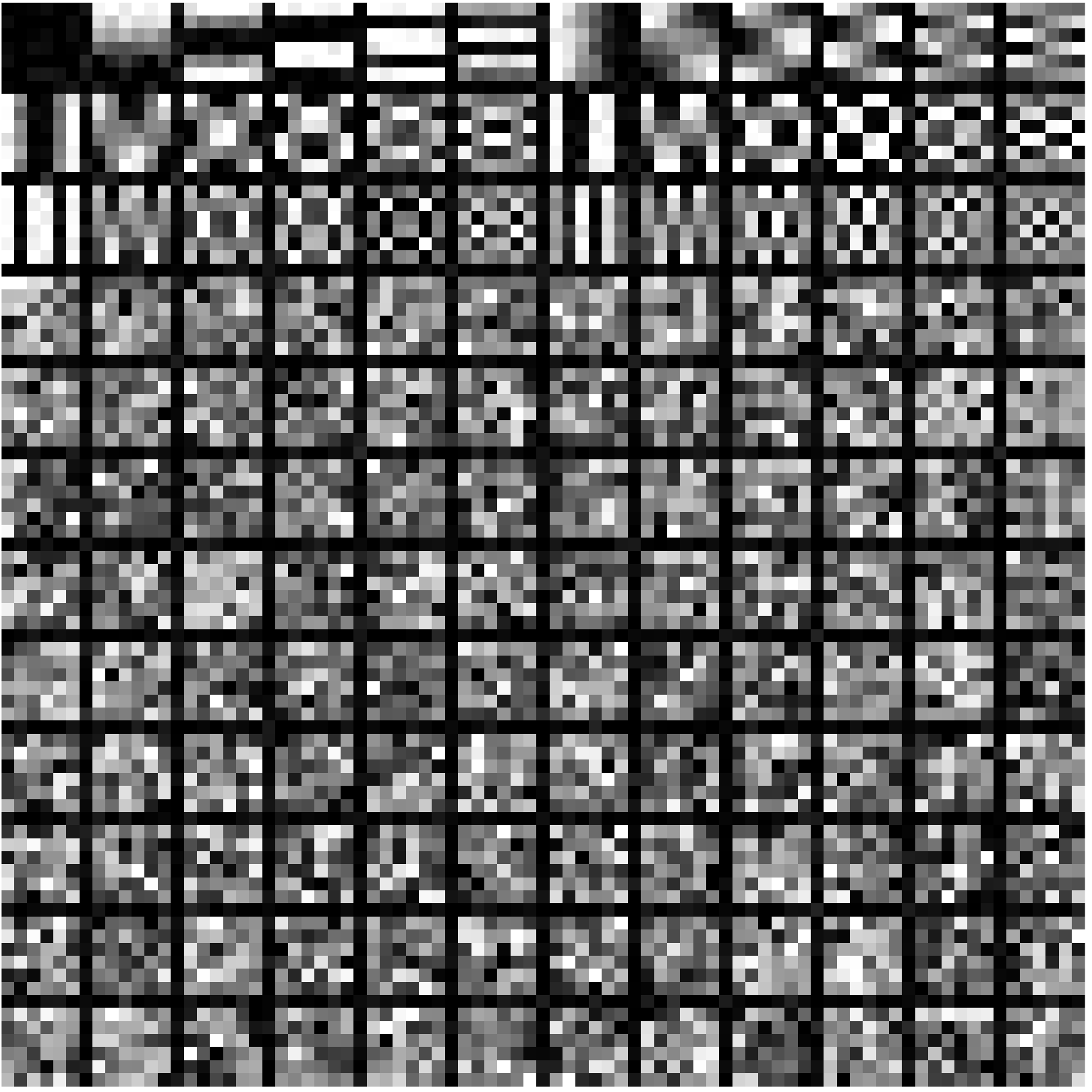}&
\includegraphics[height=1.45in]{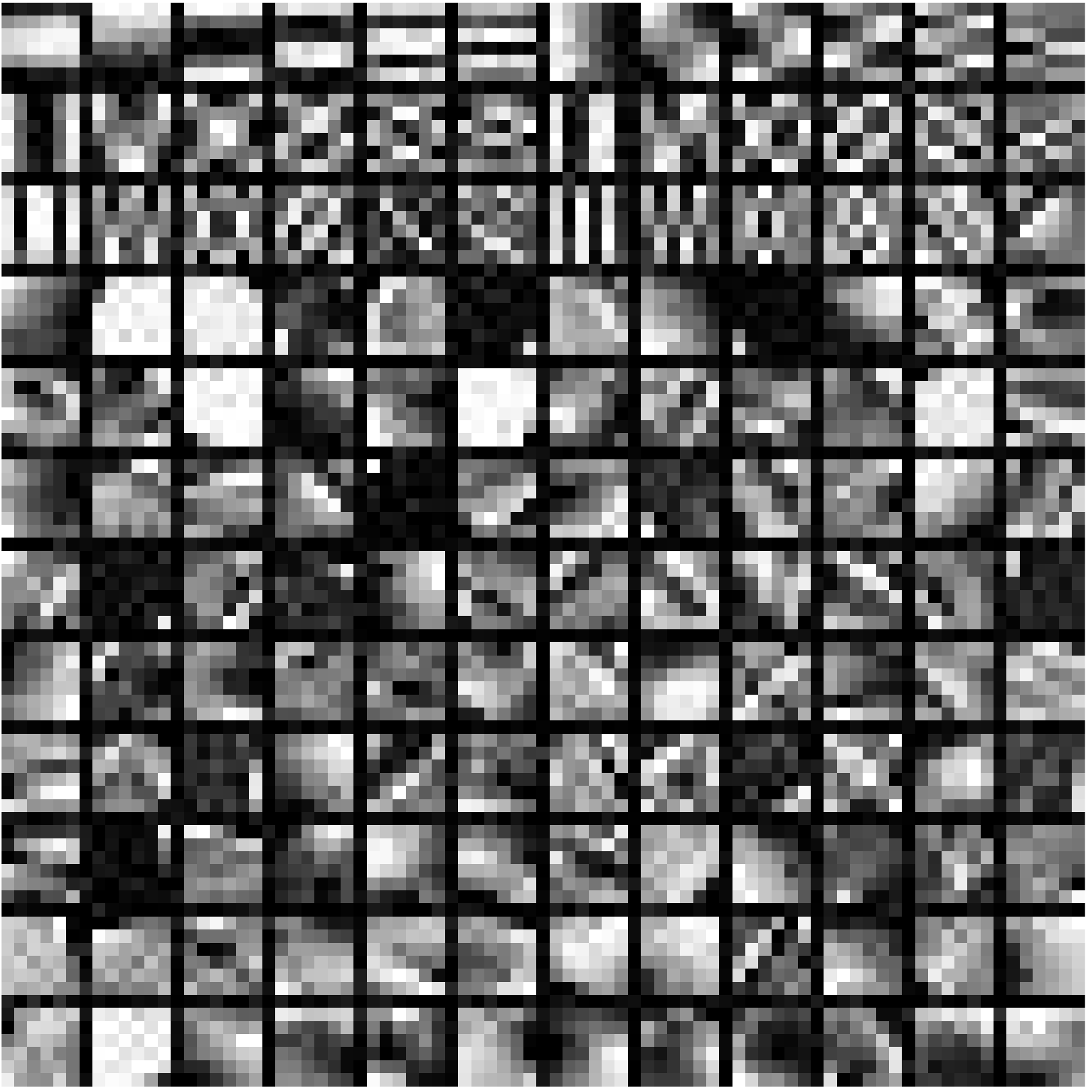}&
\includegraphics[height=1.45in]{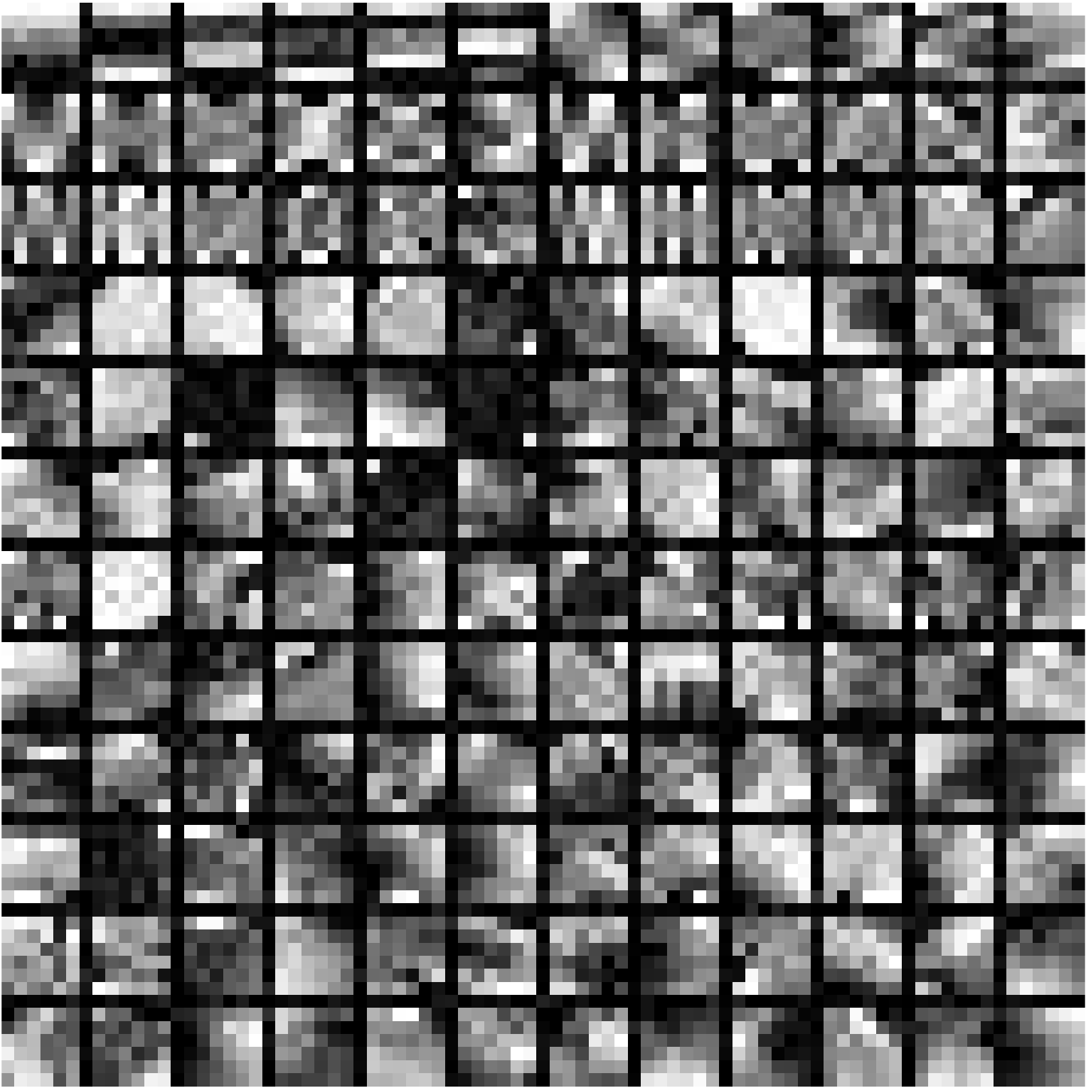}\\
(h) & (i) & (j) & (k)\\
\end{tabular}
\caption{Behavior of SOUP-DILLO MRI (for (P3)) and SOUP-DILLI MRI (for (P4)) for Image (c) with Cartesian sampling and 2.5x undersampling: (a) sampling mask in k-space; (b) magnitude of initial reconstruction $\mathbf{y}^{0}$ (PSNR = $24.9$ dB); (c) SOUP DILLO MRI (final) reconstruction magnitude (PSNR = $36.8$ dB); (d) objective function values for SOUP-DILLO MRI and SOUP-DILLI MRI; (e) reconstruction PSNR over iterations; (f) changes between successive image iterates ($\left \| \mathbf{y}^{t}-\mathbf{y}^{t-1} \right \|_{2}$) normalized by the norm of the reference image ($\left \| \mathbf{y}^{\mathrm{ref}} \right \|_{2} = 122.2$); (g) normalized changes between successive coefficient iterates ($\left \| \mathbf{C}^{t}-\mathbf{C}^{t-1} \right \|_{F}/\left \| \mathbf{Y}^{\mathrm{ref}} \right \|_{F}$) where $\mathbf{Y}^{\mathrm{ref}}$ is the patch matrix for the reference image; (h) normalized changes between successive dictionary iterates ($\left \| \mathbf{D}^{t}-\mathbf{D}^{t-1} \right \|_{F}/\sqrt{J}$); (i) initial real-valued dictionary in the algorithms; and (j) real and (k) imaginary parts of the learnt dictionary for SOUP-DILLO MRI. The dictionary columns or atoms are shown as $6 \times 6$ patches.}
\label{imcvbcs}
\end{center}
\end{figure*}

Finally, results included in the supplementary material show that when the learned dictionaries are used to sparse code (using orthogonal matching pursuit \cite{pati}) the data in a column-by-column (or signal-by-signal) manner, SOUP-DILLO dictionaries again outperform PADL dictionaries in terms of achieved NSRE. Moreover, at low sparsities, SOUP-DILLO dictionaries used with such column-wise sparse coding also outperformed (by 14-15 dB) dictionaries learned using K-SVD \cite{elad} (that is adapted for Problem (P0) with column-wise sparsity constraints).
Importantly, at similar net sparsity factors, the NSRE values achieved by SOUP-DILLO in Fig. \ref{im3}(e)  tend to be quite a bit lower (better) than those obtained using the K-SVD method for (P0). Thus, solving Problem (P1) may offer potential benefits for adaptively representing data sets (e.g., patches of an image) using very few total non-zero coefficients.
Further exploration of the proposed methods and comparisons for different dictionary sizes or larger datasets (of images or image patches) is left for future work.

\subsection{Convergence of SOUP-DIL Image Reconstruction Algorithms in Dictionary-Blind Compressed Sensing MRI} \label{sec5d}

Here, we consider the complex-valued reference image in Fig. \ref{im1bcs}(c) (Image (c)), and perform 2.5 fold undersampling of the k-space of the reference. Fig. \ref{imcvbcs}(a) shows the variable density sampling mask.
We study the behavior of the SOUP-DILLO MRI and SOUP-DILLI MRI algorithms for (P3) and (P4) respectively, when used to reconstruct the water phantom data from undersampled measurements. For SOUP-DILLO MRI, overlapping image patches of size $6 \times 6$ ($n=36$) were used with stride $r=1$ (with patch wrap around), $\nu = 10^{6}/p$ (with $p$ the number of image pixels), and we learned a fourfold overcomplete (or $36 \times 144$) dictionary with $K=1$ and $\lambda = 0.08$ in Fig. \ref{im6p}. The same settings were used for the SOUP-DILLI MRI method for (P4) with $\mu = 0.08$.
We initialized the algorithms with $\mathbf{y}^{0} = \mathbf{A}^{\dagger} \mathbf{z}$, $\mathbf{C}^{0}=\mathbf{0}$, and the initial $\mathbf{D}^{0}$ was formed by concatenating a square DCT dictionary with normalized random gaussian vectors.

Fig. \ref{imcvbcs} shows the behavior of the proposed dictionary-blind image reconstruction methods. The objective function values (Fig. \ref{imcvbcs}(d)) in (P3) and (P4) decreased monotonically and quickly in the SOUP-DILLO MRI and SOUP-DILLI MRI algorithms, respectively.
The initial reconstruction (Fig. \ref{imcvbcs}(b)) shows large aliasing artifacts and has a low PSNR of $24.9$ dB. The reconstruction PSNR (Fig. \ref{imcvbcs}(e)), however, improves significantly over the iterations in the proposed methods and converges, with the final SOUP-DILLO MRI reconstruction (Fig. \ref{imcvbcs}(c)) having a PSNR of $36.8$ dB. 
For the $\ell_{1}$ method, the PSNR converges to $36.4$ dB, which is lower than for the $\ell_{0}$ case.
The sparsity factor for the learned coefficient matrix $C$ was $5\%$ for (P3) and $16\%$ for (P4). Although larger values of $\mu$ decrease the sparsity factor for the learned $C$ in (P4), we found that the PSNR also degrades for such settings in this example.

The changes between successive iterates $\left \| \mathbf{y}^{t}-\mathbf{y}^{t-1} \right \|_{2}$ (Fig. \ref{imcvbcs}(f)) or $\left \| \mathbf{C}^{t}-\mathbf{C}^{t-1} \right \|_{F}$ (Fig. \ref{imcvbcs}(g)) or $\left \| \mathbf{D}^{t}-\mathbf{D}^{t-1} \right \|_{F}$ (Fig. \ref{imcvbcs}(h)) decreased to small values for the proposed algorithms. 
Such behavior was predicted for the algorithms by Theorems \ref{theorem5} and \ref{theorem6}, and is indicative (necessary but not suffficient condition) of convergence of the respective sequences.

\begin{table*}[!t]
\centering
\fontsize{9}{10pt}\selectfont
\begin{tabular}{|c|c|c|c|c|c|c|c|c|}
\hline
Image & Sampling & UF       & Zero-filling  & Sparse MRI & PANO         & DLMRI  & SOUP-DILLI MRI          & SOUP-DILLO MRI \\ 
\hline
a          &     Cartesian   & 7x           &  27.9                 &  28.6         &    \textbf{31.1}               & \textbf{31.1}            &   30.8                &\textbf{31.1} \\
\hline 
b          &    Cartesian    & 2.5x          &  27.7                 &  31.6         &  41.3                 & 40.2            &  38.5                 &\textbf{42.3} \\
\hline 
c          &    Cartesian   & 2.5x         &  24.9                 &  29.9         &    34.8               & 36.7            &   36.6                &\textbf{37.3} \\
\hline 
c          &    Cartesian   & 4x            &  25.9                 & 28.8         &     \textbf{32.3}              &  32.1           &   32.2                &\textbf{32.3} \\
\hline 
d          &  Cartesian   &  2.5x        &   29.5                &  32.1         &      36.9             &  38.1           &    36.7              &\textbf{38.4} \\
\hline 
e          &   Cartesian     & 2.5x         &  28.1                 &  31.7         &  40.0                 & 38.0            &   37.9                &\textbf{41.5} \\
\hline 
f          & 2D random  &  5x           &   26.3                &  27.4         &      30.4             &  30.5           &    30.3              &\textbf{30.6} \\
\hline
g          &   Cartesian   & 2.5x         &   32.8                &  39.1         &     41.6              &  41.7           &   42.2               &\textbf{43.2} \\
\hline 
\end{tabular}
\caption{PSNRs corresponding to the zero-filling (initial $\mathbf{y}^{0}=\mathbf{A}^{\dagger}\mathbf{z}$), Sparse MRI \cite{lustig}, PANO \cite{Qu2014843}, DLMRI \cite{bresai}, SOUP-DILLI MRI (for (P4)), and SOUP-DILLO MRI (for (P3)) reconstructions for various images. The simulated undersampling factors (UF), and k-space undersampling schemes are listed for each example. The best PSNRs are marked in bold. The image labels are as per Fig. \ref{im1bcs}.}
\label{tab2bcs}
\end{table*}

Finally, Fig. \ref{imcvbcs} also shows the dictionary learned (jointly with the reconstruction) for image patches by SOUP-DILLO MRI along with the initial (Fig. \ref{imcvbcs}(i)) dictionary.
The learned synthesis dictionary is complex-valued whose real (Fig. \ref{imcvbcs}(j)) and imaginary (Fig. \ref{imcvbcs}(k)) parts are displayed, with the atoms shown as patches. The learned atoms appear quite different from the initial ones and
display frequency or edge like structures that were learned efficiently from a few k-space measurements.

\subsection{Dictionary-Blind Compressed Sensing MRI Results} \label{sec5e}

We now consider images (a)-(g) in Fig. \ref{im1bcs} and evaluate the efficacy of the proposed algorithms for (P3) and (P4) for reconstructing the images from undersampled k-space measurements.
We compare the reconstructions obtained by the proposed methods to those obtained by the DLMRI \cite{bresai}, Sparse MRI \cite{lustig}, PANO \cite{Qu2014843}, and FDLCP \cite{zhan33} methods.
We used the built-in parameter settings in the publicly available implementations of Sparse MRI \cite{lus33} and PANO \cite{PANOweb}, which performed well in our experiments. We used the zero-filling reconstruction as the initial guide image in PANO \cite{Qu2014843, PANOweb}.

We used the publicly available implementation of the multi-class dictionaries learning-based FDLCP method \cite{FDLCPweb}.  The $\ell_{0}$ ``norm''-based FDLCP was used in our experiments, as it was shown in \cite{zhan33} to outperform the $\ell_{1}$ version. 
The built-in settings \cite{FDLCPweb} for the FDLCP parameters such as patch size,  $\lambda$, etc., performed well in our experiments, and we tuned the parameter $\beta$ in each experiment to achieve the best image reconstruction quality.

For the DLMRI implementation \cite{dlmri1}, we used image patches of size\footnote{The reconstruction quality improves slightly with a larger patch size, but with a substantial increase in runtime.} $6 \times 6$ \cite{bresai}, and learned a $36 \times 144$ dictionary and performed image reconstruction using 45 iterations of the algorithm. The patch stride $r=1$, and $14400$ randomly selected patches\footnote{Using a larger training size during the dictionary learning step of DLMRI provides negligible improvement in image reconstruction quality, while leading to increased runtimes. A different random subset is used in each iteration of DLMRI.} were used during the dictionary learning step (executed with 20 iterations of K-SVD) of DLMRI.
Mean-subtraction was not performed for the patches prior to the dictionary learning step. (We adopted this strategy for DLMRI here as it led to better performance.)
A maximum sparsity level (of $s= 7$ per patch) is employed together with an error threshold (for sparse coding) during the dictionary learning step.
The $\ell_{2}$ error threshold per patch varies linearly from $0.34$ to $0.04$ over the DLMRI iterations, except for Figs. \ref{im1bcs}(a), \ref{im1bcs}(c), and \ref{im1bcs}(f) (noisier data), where it varies from $0.34$ to $0.15$ over the iterations. 
Once the dictionary is learnt in the dictionary learning step of each DLMRI (outer) iteration, all image patches are sparse coded with the same error threshold as used in learning and a relaxed maximum sparsity level of $14$.
This relaxed sparsity level is indicated in the DLMRI-Lab toolbox \cite{dlmri1}, as it leads to better performance in practice. As an example, DLMRI with these parameter settings provides 0.4 dB better reconstruction PSNR for the data in Fig. \ref{imcvbcs} compared to DLMRI with a common maximum sparsity level (other parameters as above) of $s=7$ in the dictionary learning and follow-up sparse coding (of all patches) steps.
We observed the above parameter settings (everything else as per the indications in the DLMRI-Lab toolbox \cite{dlmri1}) to work well for DLMRI in the experiments.

For SOUP-DILLO MRI and SOUP-DILLI MRI, patches of size $6 \times 6$ were again used ($n=36$ like for DLMRI) with stride $r=1$ (with patch wrap around), $\nu = 10^{6}/p$, $M=45$ (same number of outer iterations as for DLMRI), and a $36 \times 144$ dictionary was learned. We found that using larger values of $\lambda$ or $\mu$ during the initial outer iterations of the methods led to faster convergence and better aliasing removal. Hence, we vary $\lambda$ from $0.35$ to $0.01$ over the  (outer $t$) iterations in Fig. \ref{im6p}, except for Figs. \ref{im1bcs}(a), \ref{im1bcs}(c), and \ref{im1bcs}(f) (noisier data), where it varies from $0.35$ to $0.04$. These settings and $\mu = \lambda/1.4$ worked well in our experiments. We used $5$ inner iterations of SOUP-DILLO and $1$ inner iteration (observed optimal) of OS-DL. The iterative reconstruction algorithms were initialized as mentioned in Section \ref{sec5d}.

Table \ref{tab2bcs} lists the reconstruction PSNRs\footnote{While we compute PSNRs using magnitudes (typically the useful component of the reconstruction) of images, we have observed similar trends as in Table \ref{tab2bcs} when the PSNR is computed based on the difference (error) between the complex-valued images.} corresponding to the zero-filling (the initial $\mathbf{y}^{0}$ in our methods), Sparse MRI, DLMRI, PANO, SOUP-DILLO MRI, and SOUP-DILLI MRI reconstructions for several cases.
The proposed SOUP-DILLO MRI Algorithm for (P3) provides the best reconstruction PSNRs in Table  \ref{tab2bcs}. In particular, it provides 1 dB better PSNR on the average compared to the K-SVD \cite{elad} based DLMRI method and the non-local patch similarity-based PANO method.
While the K-SVD-based algorithm for image denoising \cite{elad2} explicitly uses information of the noise variance (of Gaussian noise) in the observed noisy patches, in the compressed sensing MRI application here, the artifact properties (variance or distribution of the aliasing/noise artifacts) in each iteration of DLMRI are typically unknown, i.e., the DLMRI algorithm does not benefit from a superior modeling of artifact statistics and one must empirically set parameters such as the patch-wise error thresholds. The improvements provided by SOUP-DILLO MRI over DLMRI thus might stem from a better optimization framework for the former (e.g., the overall sparsity penalized formulation or the exact and guaranteed block coordinate descent algorithm). A more detailed theoretical analysis including the investigation of plausible recovery guarantees for the proposed schemes is left for future work.

\begin{table}[!t]
\centering
\fontsize{7}{10pt}\selectfont
\begin{tabular}{|c|c|c|c|c|}
\hline
Image &  UF       & FDLCP   & SOUP-DILLO MRI & SOUP-DILLO MRI \\ 
           &             & ($\ell_{0}$ ``norm'')  &   (Zero-filling init.)    &   (FDLCP init.)  \\
\hline
a          & 7x           &      \textbf{31.5}          &     31.1             &\textbf{31.5} \\
\hline 
b          & 2.5x         &      44.2          &   42.3               &\textbf{44.8} \\
\hline 
c          & 2.5x         &       33.5         &   \textbf{37.3}               &\textbf{37.3} \\
\hline 
c          & 4x            &        32.8        &   32.3               &\textbf{33.5} \\
\hline 
d         &  2.5x        &        38.5        &   38.4               &\textbf{38.7} \\
\hline 
e         & 2.5x         &        43.4        &   41.5               &\textbf{43.9} \\
\hline 
f           &  5x           &      30.4          &  \textbf{30.6}                &\textbf{30.6} \\
\hline
g         & 2.5x         &       43.2         &   43.2               &\textbf{43.5} \\
\hline 
\end{tabular}
\caption{PSNRs corresponding to the $\ell_{0}$ ``norm''-based FDLCP reconstructions \cite{zhan33}, and the SOUP-DILLO MRI (for (P3)) reconstructions obtained with a zero-filling ($\mathbf{y}^{0}=\mathbf{A}^{\dagger}\mathbf{z}$) initialization or by initializing with the FDLCP result (last column). The various images, sampling schemes, and undersampling factors (UF) are the same as in Table \ref{tab2bcs}. The best PSNRs are marked in bold.}
\label{tab2bcsb}
\end{table}

SOUP-DILLO MRI (average runtime of 2180 seconds) was also faster in Table \ref{tab2bcs} than the previous DLMRI (average runtime of 3156 
seconds).
Both the proposed SOUP methods significantly improved the reconstruction quality compared to the classical non-adaptive Sparse MRI method.
Moreover, the $\ell_{0}$ ``norm''-based SOUP-DILLO MRI outperformed the corresponding $\ell_{1}$ method (SOUP-DILLI MRI) by 1.4 dB on average in Table \ref{tab2bcs}, indicating potential benefits for $\ell_{0}$ penalized dictionary adaptation in practice.
The promise of non-convex sparsity regularizers (including the $\ell_{0}$ or $\ell_{p}$ norm for $p<1$) compared to $\ell_{1}$ norm-based techniques for compressed sensing MRI has been demonstrated in prior works \cite{josh, Char, zhan33}.

Table \ref{tab2bcsb} compares the reconstruction PSNRs obtained by SOUP-DILLO MRI to those obtained by the recent $\ell_{0}$ ``norm''-based FDLCP \cite{zhan33} for the same cases as in Table \ref{tab2bcs}.
SOUP-DILLO MRI initialized with zero-filling reconstructions performs quite similarly on the average  (0.1 dB worse) as $\ell_{0}$ FDLCP in Table \ref{tab2bcsb}.
However, with better initializations, SOUP-DILLO MRI can provide even better reconstructions than with the zero-filling initialization.
We investigated SOUP-DILLO MRI, but initialized with the $\ell_{0}$ FDLCP reconstructions (for $\mathbf{y}$).
The parameter $\lambda$ was set to the eventual value in Table~\ref{tab2bcs}, i.e., $0.01$ or $0.04$ (for noisier data), with decreasing $\lambda$'s used for Image (c) with 2.5x Cartesian undersampling, where the FDLCP reconstruction was still highly aliased.
In this case, SOUP-DILLO MRI consistently improved over the $\ell_{0}$ FDLCP reconstructions (initializations), and provided 0.8 dB better PSNR on the average in Table \ref{tab2bcsb}.
These results illustrate the benefits and potential for the proposed dictionary-blind compressed sensing approaches.
The PSNRs for our schemes could be further improved with better parameter selection strategies.

\begin{figure}[!t]
\begin{center}
\begin{tabular}{cc}
\includegraphics[height=1.35in]{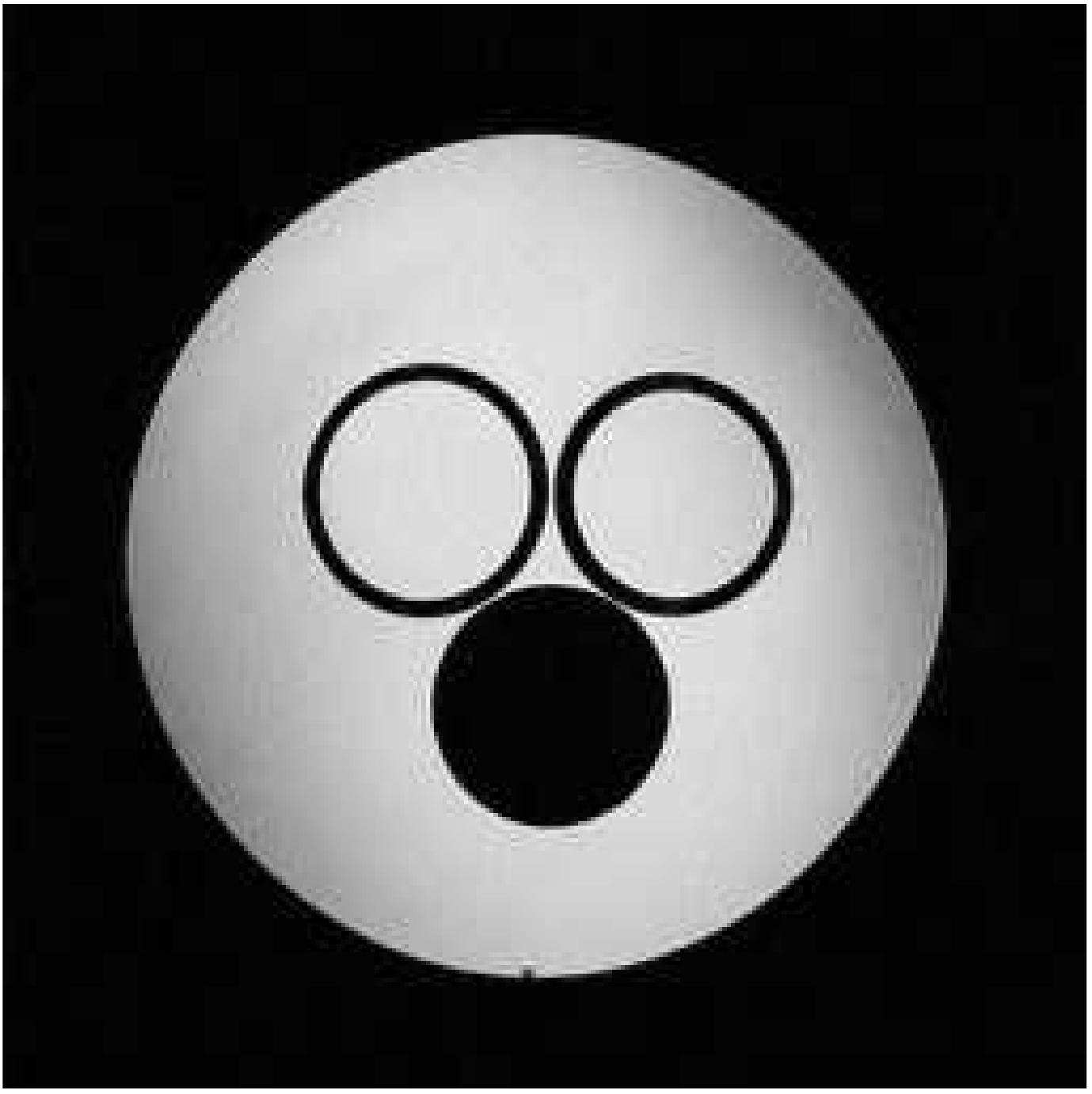}&
\includegraphics[height=1.35in]{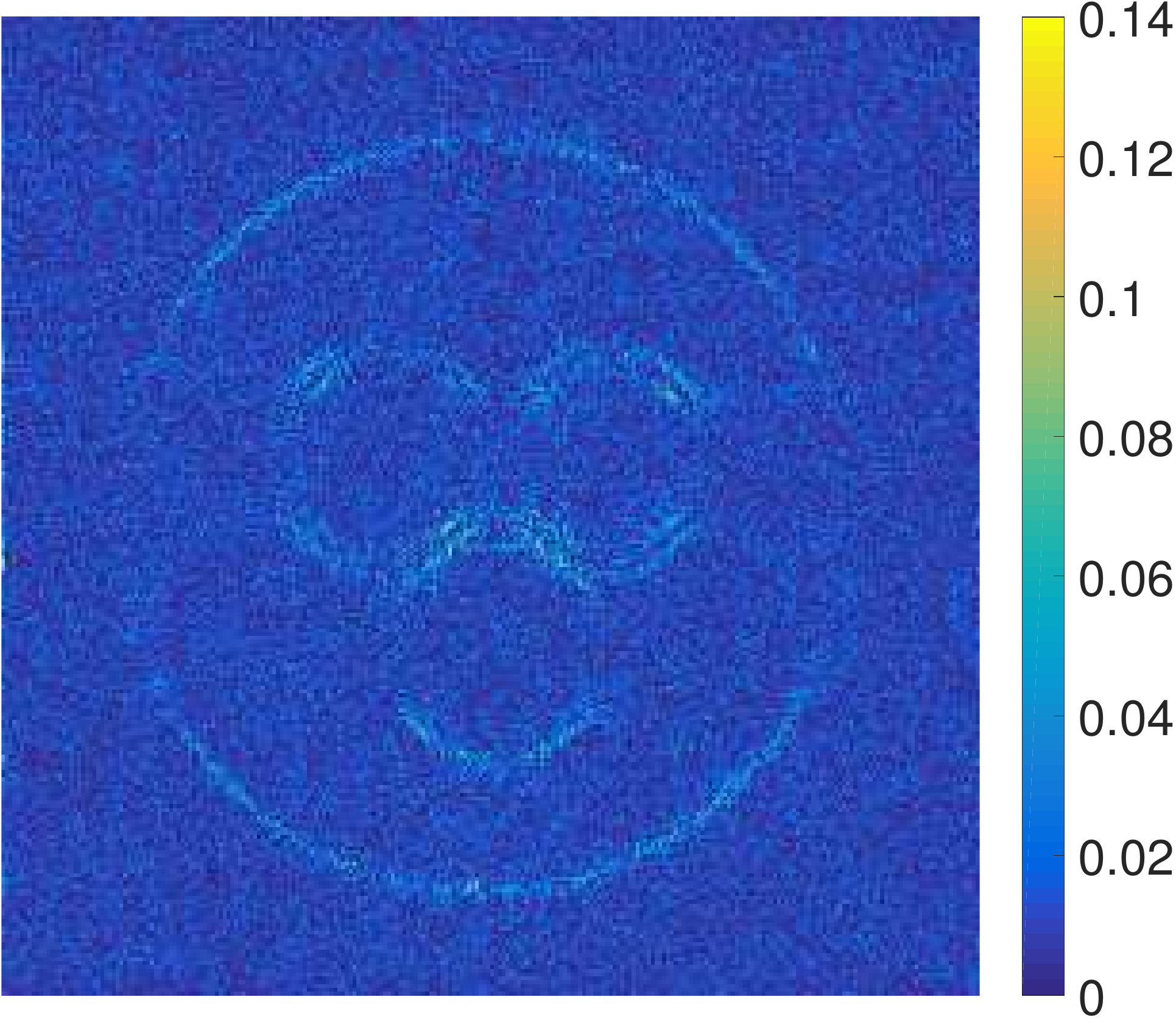}\\
(a) & (e) \\
\includegraphics[height=1.35in]{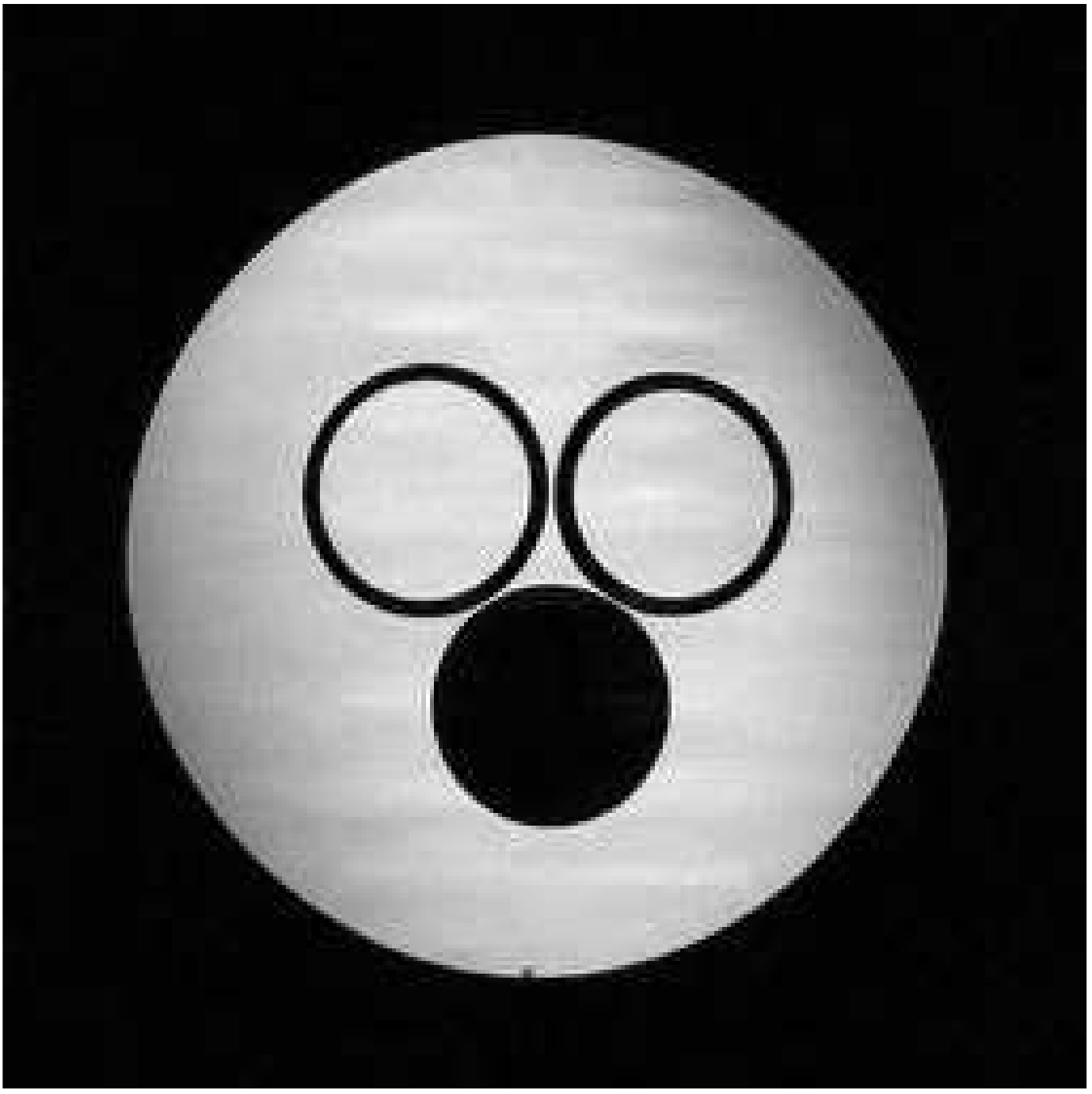}&
\includegraphics[height=1.35in]{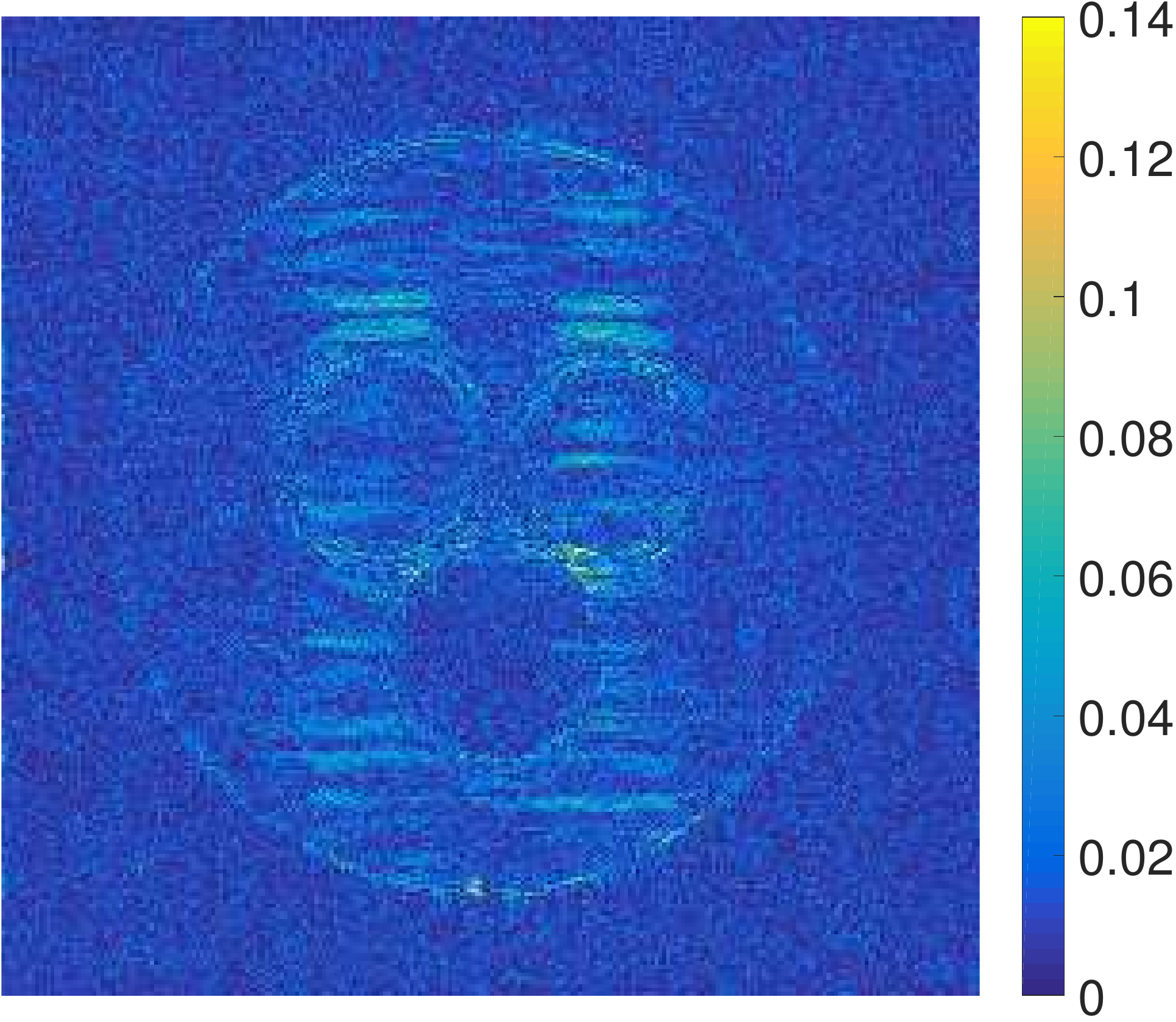}\\
 (b) & (f)\\
\includegraphics[height=1.35in]{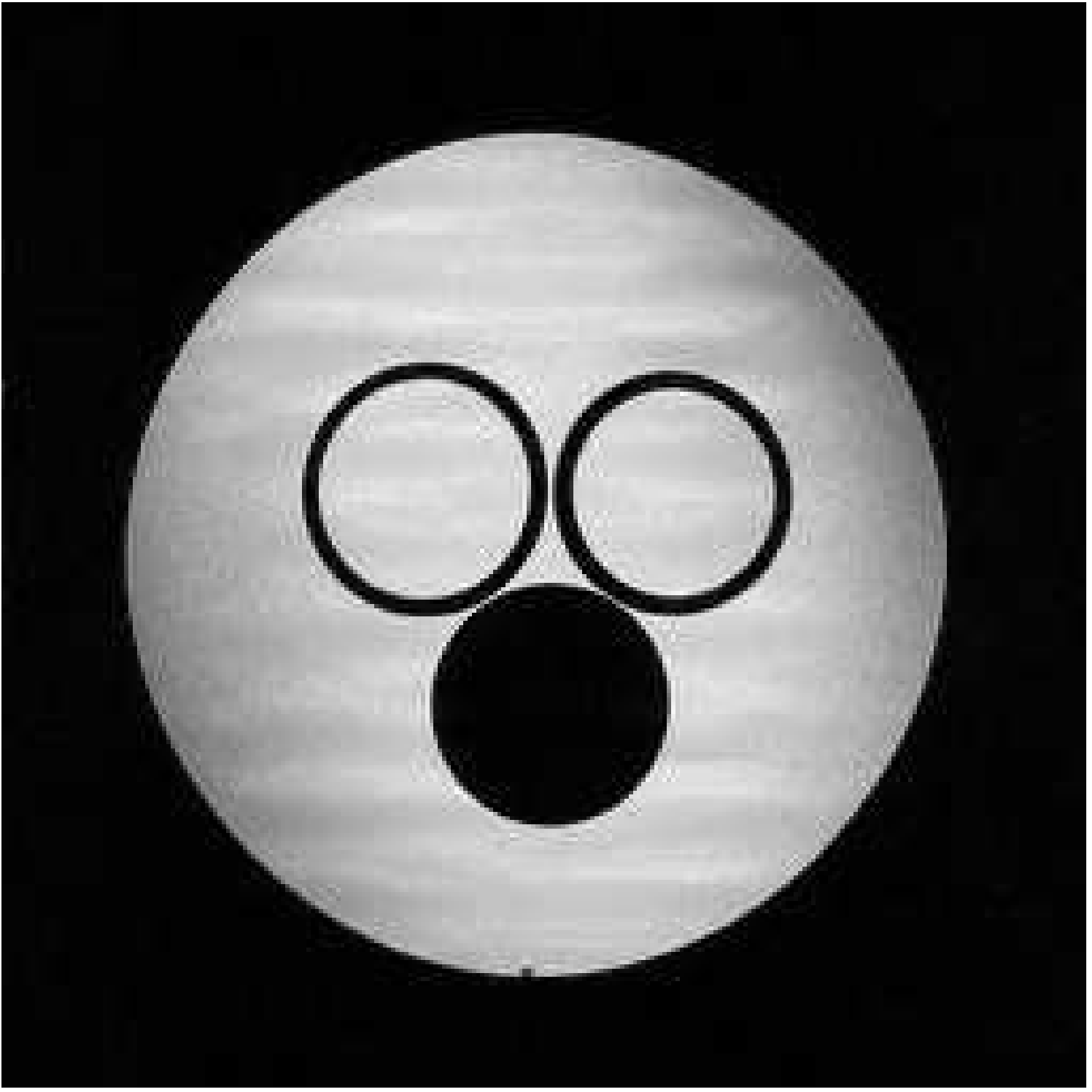}&
\includegraphics[height=1.35in]{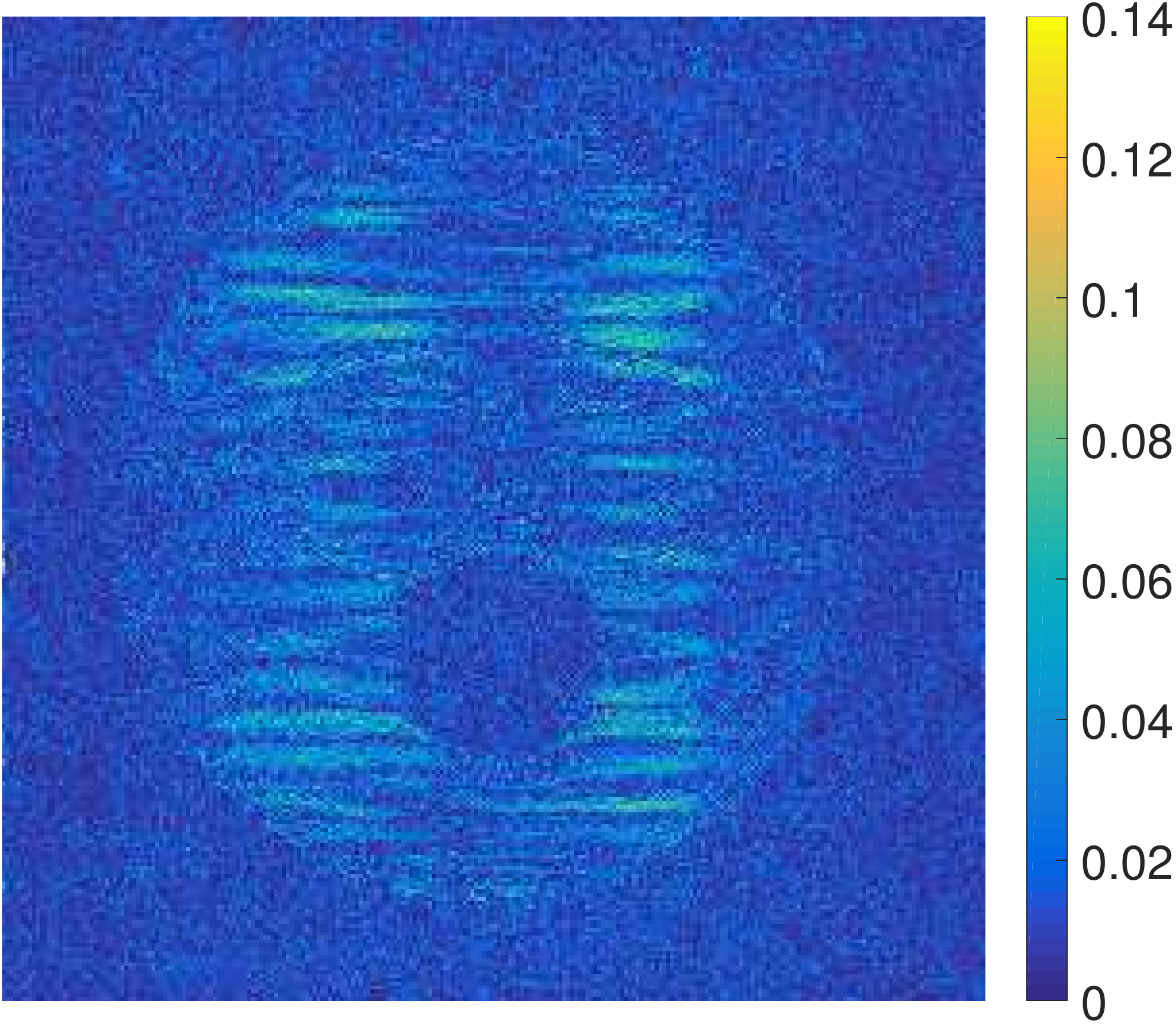}\\
 (c) & (g) \\
\includegraphics[height=1.35in]{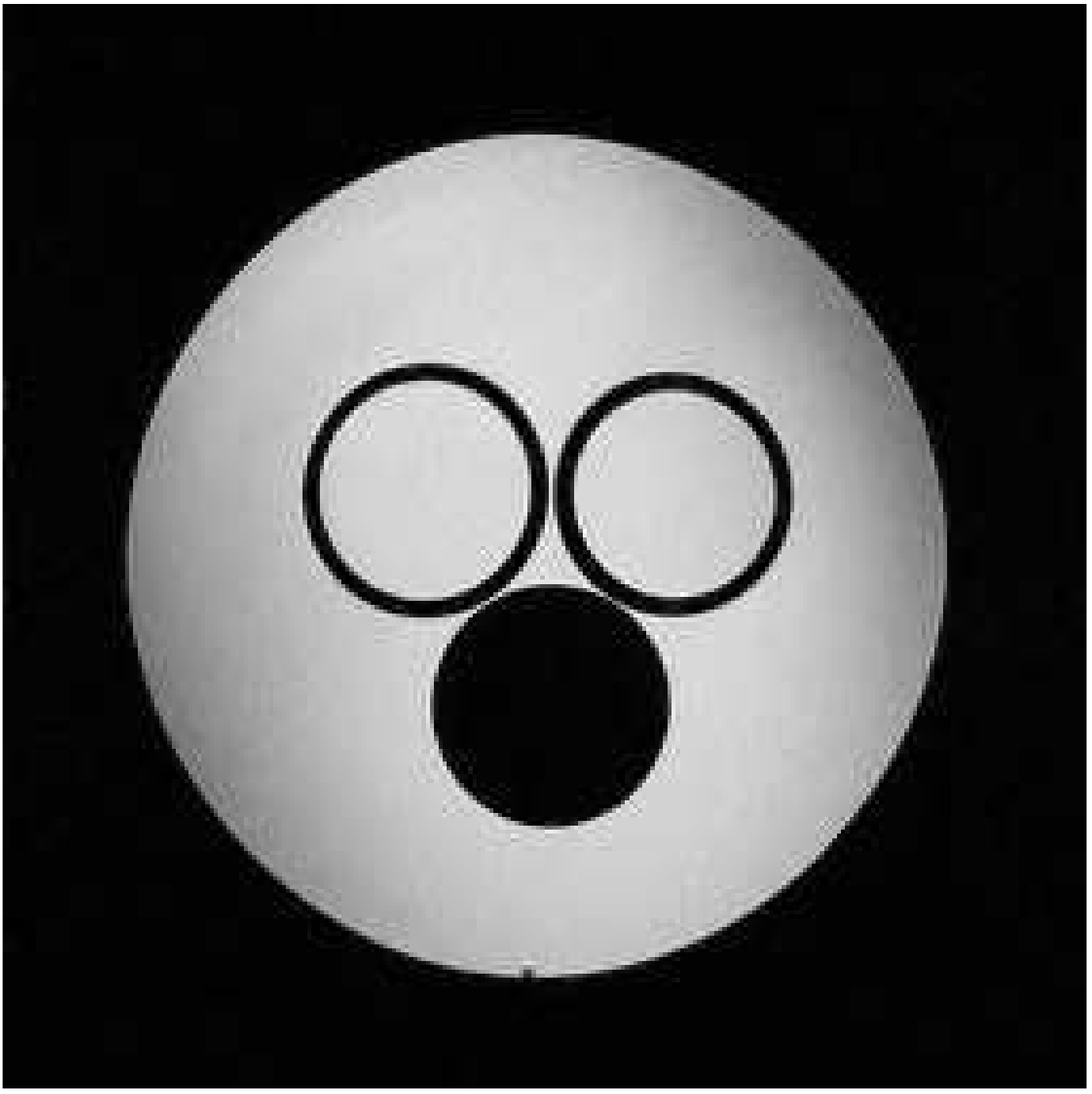}&
\includegraphics[height=1.35in]{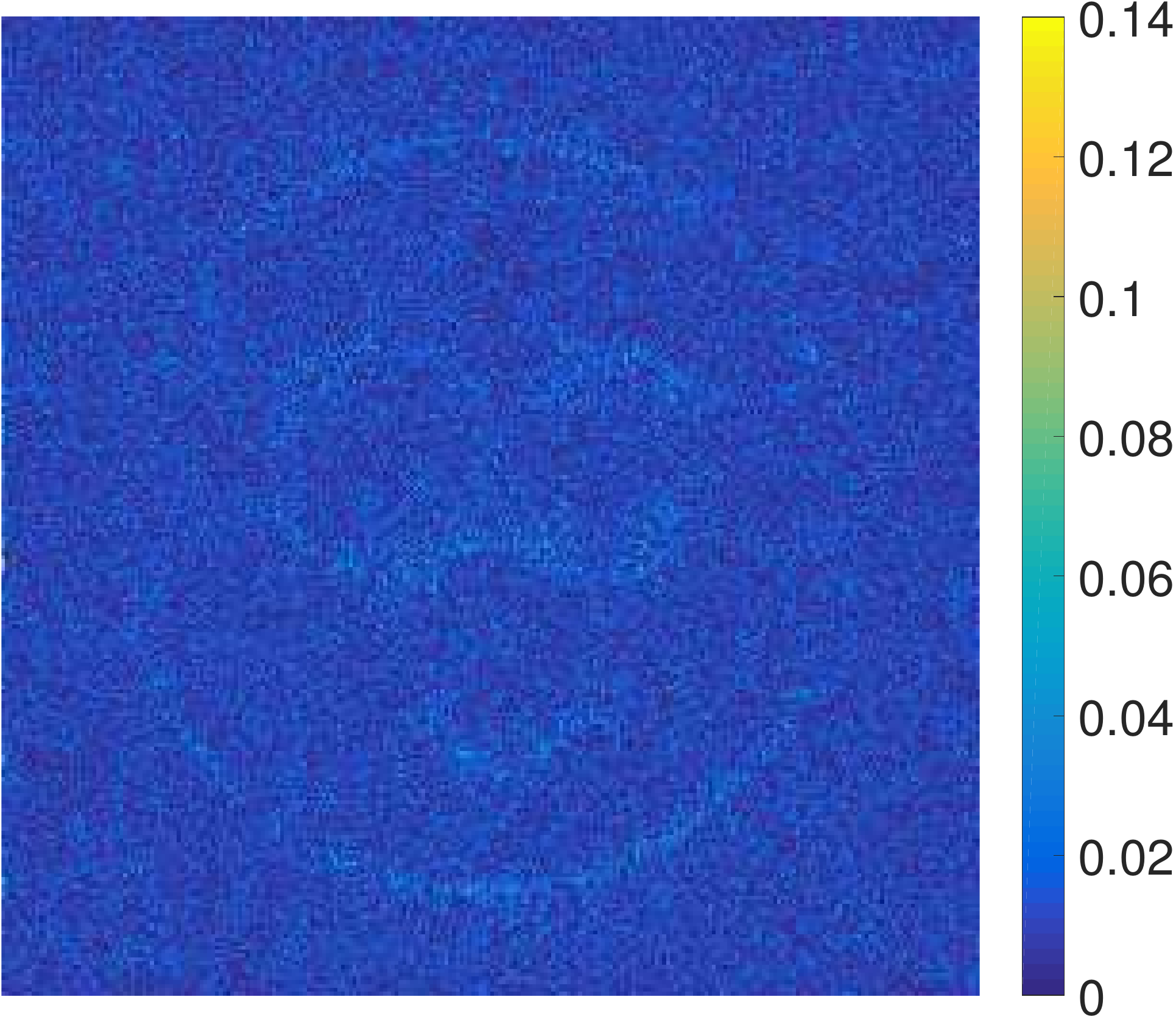}\\
(d) & (h) \\
\end{tabular}
\caption{Results for Image (c) with Cartesian sampling and 2.5x undersampling. The sampling mask is shown in Fig. \ref{imcvbcs}(a). Reconstructions (magnitudes): (a) DLMRI \cite{bresai}; (b) PANO \cite{Qu2014843}; (c) $\ell_{0}$ ``norm''-based FDLCP \cite{zhan33}; and (d) SOUP-DILLO MRI (with zero-filling initialization). (e)-(h) are the reconstruction error maps for (a)-(d), respectively.}
\label{im4bcsbb}
\end{center}
\end{figure}

Fig. \ref{im4bcsbb} shows the reconstructions and reconstruction error maps (i.e., the magnitude of the difference between the magnitudes of the reconstructed and reference images) for various methods for an example in Table \ref{tab2bcs}.
The reconstructed images and error maps for SOUP-DILLO MRI show much fewer artifacts and smaller distortions than for the other methods. Another comparison is included in the supplement.

\section{Conclusions}
\label{sec6}  
This paper investigated in detail fast methods for synthesis dictionary learning. 
The SOUP algorithms for dictionary learning were further extended to the scenario of dictionary-blind image reconstruction.
A convergence analysis was presented for the various efficient algorithms in highly non-convex problem settings.
The proposed SOUP-DILLO algorithm for aggregate sparsity penalized dictionary learning had superior performance over recent dictionary learning methods for sparse data representation.
The proposed SOUP-DILLO (dictionary-blind) image reconstruction method outperformed standard benchmarks involving the K-SVD algorithm, as well as some other recent methods in the compressed sensing MRI application.
Recent works have investigated the data-driven adaptation of alternative signal models such as the analysis dictionary \cite{akd} or transform model \cite{sabres, doubsp2l, saiwen, sravTCI1}. While we focused on synthesis dictionary learning methodologies in this work, we plan to compare various kinds of data-driven models in future work.
We have considered extensions of the SOUP-DIL methodology to other novel settings and applications elsewhere \cite{saibrrajfes1}.
Extensions of the SOUP-DIL methods for online learning \cite{Mai} or for learning multi-class models are also of interest, and are left for future work.


\section{Discussion of Image Denoising Results for SOUP-DILLO in  \cite{sairajfes}} \label{app1}

Results obtained using the SOUP-DILLO  (learning) algorithm for image denoising are reported in \cite{sairajfes}, where the results were compared to those obtained using the K-SVD image denoising method \cite{elad2}.  We briefly discuss these results here for completeness.

Recall that the goal in image denoising is to recover an estimate of an image $\mathbf{y} \in \mathbb{C}^{p}$ (2D image represented as a vector) from its corrupted measurements $\mathbf{z} = \mathbf{y} + \mathbf{\epsilon}$, where $\mathbf{\epsilon}$ is the noise (e.g., i.i.d. Gaussian).
First, while both K-SVD and the SOUP-DILLO (for (P1)) methods could be applied to noisy image patches to obtain adaptive denoising (as $\mathbf{D} \mathbf{x}_{i}$ in $\mathbf{P}_{i}\mathbf{z} \approx \mathbf{D} \mathbf{x}_{i}$) of the patches (the denoised image is obtained easily from denoised patches by averaging together the overlapping patches at their respective 2D locations, or solving (22) in \cite{sairajfes}), the K-SVD-based denoising method \cite{elad2} uses a dictionary learning procedure where the $\ell_{0}$ ``norms" of the sparse codes are minimized so that a fitting constraint or error constraint of $ \left \| \mathbf{P}_{i}\mathbf{z}- \mathbf{D} \mathbf{x}_{i} \right \|_{2}^{2} \leq\epsilon$ 
is met for representing each noisy patch. In particular, when the noise is i.i.d. Gaussian, $\epsilon=nC^{2}\sigma^{2}$ is used, with $C>1$  (typically chosen very close to 1) a constant and $\sigma^{2}$ being the noise variance for pixels. Such a constraint serves as a strong prior (law of large numbers), and is an important reason for the denoising capability of K-SVD \cite{elad2}.
 
In the SOUP-DILLO denoising method in \cite{sairajfes}, we set $\lambda\propto\sigma$ during learning (in (P1)), and once the dictionary is learned from noisy image patches, we re-estimated the patch sparse codes using a single pass (over the noisy patches) of orthogonal matching pursuit (OMP)  \cite{pati}, by employing an error constraint criterion like in K-SVD denoising. This strategy only uses information on the Gaussian noise statistics in a sub-optimal way, especially during learning. However, SOUP-DILLO still provided comparable denoising performance vis-a-vis K-SVD with this approach (cf. \cite{sairajfes}). Importantly, SOUP-DILLO provided up to 0.1-0.2 dB better denoising PSNR than K-SVD in (very) high noise cases in \cite{sairajfes}.

%% file: supplement.tex
\setcounter{prop}{0}
\setcounter{theorem}{0}

This document provides proofs and additional experimental results to accompany our manuscript \citeSupp{sairakfes55:supp}.

\section{Proofs of Propositions 1-3} \label{app1new}

Here, we provide the proofs for Propositions 1-3 in our manuscript \citeSupp{sairakfes55:supp}. We state the propositions below for completeness.
Recall that propositions \ref{prop1}-\ref{prop2} provide the solutions for the following problems:
\begin{align} 
&\min_{\mathbf{c}_{j} \in \mathbb{C}^{N}} \; \begin{Vmatrix}
\mathbf{E}_{j} - \mathbf{d}_{j}\mathbf{c}_{j}^{H}
\end{Vmatrix}_{F}^{2} + \lambda^{2} \left \| \mathbf{c}_{j} \right \|_{0}  \;\; \mathrm{s.t.}\; \: \left \| \mathbf{c}_{j} \right \|_{\infty} \leq L \label{eqop5new} \\
&  \min_{\mathbf{c}_{j}  \in \mathbb{C}^{N}} \; \begin{Vmatrix}
\mathbf{E}_{j} - \mathbf{d}_{j}\mathbf{c}_{j}^{H}
\end{Vmatrix}_{F}^{2} + \mu \left \| \mathbf{c}_{j} \right \|_{1}  \label{eqop5bbnew} \\
&  \min_{\mathbf{d}_{j} \in \mathbb{C}^{n}} \; \begin{Vmatrix}
\mathbf{E}_{j} - \mathbf{d}_{j}\mathbf{c}_{j}^{H}
\end{Vmatrix}_{F}^{2}  \;\:\; \mathrm{s.t.}\; \: \left \| \mathbf{d}_{j} \right \|_2 =1.  \label{eqop6new}
\end{align}

\begin{prop} \vspace{0.02in}
Given $\mathbf{E}_{j} \in \mathbb{C}^{n \times N}$ and $\mathbf{d}_{j} \in \mathbb{C}^{n}$, and assuming $L > \lambda$,  a global minimizer of the sparse coding problem \eqref{eqop5new} is obtained by the following truncated hard-thresholding operation:
\begin{equation} \label{tru1ch4new}
\hat{\mathbf{c}}_{j} =  \min\left ( \begin{vmatrix}
H_{\lambda} \left ( \mathbf{E}_{j}^{H}\mathbf{d}_{j} \right )
\end{vmatrix}, L \mathbf{1}_{N} \right ) \, \odot \, e^{j \angle \,  \mathbf{E}_{j}^{H}\mathbf{d}_{j}  }.
\end{equation}
The minimizer of \eqref{eqop5new} is unique if and only if the vector $\mathbf{E}_{j}^{H}\mathbf{d}_{j}$ has no entry with a magnitude of $\lambda$. 
\end{prop}

\begin{proof}
First, for a vector $\mathbf{d}_{j}$ that has unit $\ell_2$ norm, we have the following equality:
\begin{align} 
\nonumber & \begin{Vmatrix}
\mathbf{E}_{j} - \mathbf{d}_{j}\mathbf{c}_{j}^{H}
\end{Vmatrix}_{F}^{2}= \left \| \mathbf{E}_{j} \right \|_{F}^{2} +  \left \| \mathbf{c}_{j} \right \|_{2}^{2} - 2\, \text{Re} \left \{ \mathbf{c}_{j}^{H} \mathbf{E}_{j}^{H} \mathbf{d}_{j}  \right \}\\
& \;\;\;\;= \left \| \mathbf{c}_{j} - \mathbf{E}_{j}^{H} \mathbf{d}_{j} \right \|_{2}^{2} + \left \| \mathbf{E}_{j} \right \|_{F}^{2} - \left \| \mathbf{E}_{j}^{H}\mathbf{d}_{j} \right \|_{2}^{2}.
\label{eqop8new}
\end{align}
By substituting \eqref{eqop8new} into \eqref{eqop5new},
it is clear that  Problem \eqref{eqop5new} is equivalent to 
\begin{equation} \label{eqop5bnew}
\min_{\mathbf{c}_{j}} \; \left \| \mathbf{c}_{j} - \mathbf{E}_{j}^{H} \mathbf{d}_{j} \right \|_{2}^{2}  + \lambda^{2} \left \| \mathbf{c}_{j} \right \|_{0} \;\; \mathrm{s.t.}\; \: \left \| \mathbf{c}_{j} \right \|_{\infty} \leq L.
\end{equation}
Define $\mathbf{b} \triangleq  \mathbf{E}_{j}^{H} \mathbf{d}_{j}$.
Then, the objective in \eqref{eqop5bnew} simplifies to $\sum_{i=1}^{N} \left \{ \left | c_{ji} - b_{i} \right |^{2}  + \lambda^{2} \, \theta(c_{ji})
 \right \}$  with
\begin{equation} \label{bbt5apeq3new}
 \theta \left ( a \right )=\left\{\begin{matrix}
 0,& \; \mathrm{if} \;\, a = 0 \\
1,  & \; \mathrm{if} \;\, a \neq 0.
\end{matrix}\right.
\end{equation}

Therefore, we solve for each entry $c_{ji}$ of $\mathbf{c}_{j}$ as
\begin{equation} \label{eqop10new}
\hat{c}_{ji} = \underset{c_{ji} \in \mathbb{C}}{\arg \min}  \left | c_{ji} - b_{i} \right |^{2}  + \lambda^{2} \, \theta(c_{ji})   \;\; \mathrm{s.t.}\; \: \left | c_{ji} \right | \leq L.
\end{equation}

For the term $\left | \hat{c}_{ji} - b_{i} \right |^{2}$ to be minimal in \eqref{eqop10new}, clearly, the phases of $\hat{c}_{ji}$ and $b_{i}$ must match, and thus, the first term in the cost can be equivalently replaced (by factoring out the optimal phase) with $\begin{vmatrix}
\left | c_{ji} \right | - \left | b_{i} \right |
\end{vmatrix}^{2}$.
It is straightforward to show that when $\left | b_{i} \right | \leq L$,
\begin{equation} \label{bbt5apeq5new}
\left | \hat{c}_{ji} \right | =\left\{\begin{matrix}
 0,& \; \mathrm{if} \;\, \left | b_{i} \right |^{2} < \lambda^{2} \\
 \left | b_{i} \right |,  & \; \mathrm{if} \;\, \left | b_{i} \right |^{2} > \lambda^{2}
\end{matrix}\right.
\end{equation}
When $\left | b_{i} \right | = \lambda$ ($\lambda \leq L$), the optimal $\left | \hat{c}_{ji} \right |$ can be either $\left | b_{i} \right |$ or $0$ (non-unique), and both these settings achieve the minimum objective value $\lambda^{2}$ in \eqref{eqop10new}.
We choose $\left | \hat{c}_{ji} \right | =  \left | b_{i} \right |$ to break the tie.
Next, when $\left | b_{i} \right | > L$, we have
\begin{equation} \label{bbt5apeq5yunew}
 \left | \hat{c}_{ji} \right | =\left\{\begin{matrix}
 0, & \mathrm{if} \;\, \left | b_{i} \right |^{2} < \left ( L - \left | b_{i} \right | \right )^{2} + \lambda^{2}\\
  L,   & \; \mathrm{if} \;\, \left | b_{i} \right |^{2} > \left ( L - \left | b_{i} \right | \right )^{2} + \lambda^{2}
\end{matrix}\right.
\end{equation}
Since $L > \lambda$, clearly $\left | b_{i} \right |^{2} > \left ( L - \left | b_{i} \right | \right )^{2} + \lambda^{2}$ in \eqref{bbt5apeq5yunew}.

Thus, an optimal $\hat{c}_{ji}$ in \eqref{eqop10new} is compactly written as $\hat{c}_{ji} =  \min\left ( \begin{vmatrix}
H_{\lambda} \left ( b_{i} \right )
\end{vmatrix}, L \right ) \, \cdot \, e^{j \angle b_{i}}$, thereby establishing \eqref{tru1ch4new}. The condition for uniqueness of the sparse coding solution follows from the arguments for the case $\left | b_{i} \right | = \lambda$ above.
\end{proof}

\begin{prop} \vspace{0.02in}
Given $\mathbf{E}_{j} \in \mathbb{C}^{n \times N}$ and $\mathbf{d}_{j} \in \mathbb{C}^{n}$,  the unique global minimizer of the sparse coding problem \eqref{eqop5bbnew} is
\begin{equation} \label{tru1ch4bnmnnew}
\hat{\mathbf{c}}_{j} =  \max \left (\begin{vmatrix}
  \mathbf{E}_{j}^{H}\mathbf{d}_{j}
\end{vmatrix} - \frac{\mu}{2} \mathbf{1}_{N}, \, 0 \right ) \, \odot \, e^{j \angle \, \mathbf{E}_{j}^{H}\mathbf{d}_{j} }.
\end{equation}
\end{prop}

\begin{proof}
 Following the same arguments as in the proof of Proposition \ref{prop1}, \eqref{eqop5bbnew} corresponds to solving the following problem for each $c_{ji}$:
\begin{equation} \label{eqop10bnew}
\hat{c}_{ji} = \underset{c_{ji} \in \mathbb{C}}{\arg \min}  \left | c_{ji} - b_{i} \right |^{2}  +  \mu \left | c_{ji} \right |,
\end{equation}
after replacing the term $\lambda^{2} \theta(c_{ji})$ in \eqref{eqop10new} with $ \mu \left | c_{ji} \right |$ above.
Clearly, the phases of $\hat{c}_{ji}$ and $b_{i}$ must match for the first term in the cost above to be minimal. We then have $\left | \hat{c}_{ji} \right | = \underset{\beta \geq 0}{\arg \min}  \begin{pmatrix}
\beta - \left | b_{i} \right |
\end{pmatrix}^{2}  +  \mu \, \beta$. Thus, $\left | \hat{c}_{ji} \right | = $ $\max(\left | b_{i} \right |- \mu/2, \, 0)$.
\end{proof}

\begin{prop} \vspace{0.02in}
Given $\mathbf{E}_{j} \in \mathbb{C}^{n \times N}$ and $\mathbf{c}_{j} \in \mathbb{C}^{N}$, a global minimizer of the dictionary atom update problem \eqref{eqop6new} is
\begin{equation} \label{tru1ch4gnew}
\hat{\mathbf{d}}_{j} =  \left\{\begin{matrix}
\frac{\mathbf{E}_{j}\mathbf{c}_{j}}{\left \| \mathbf{E}_{j}\mathbf{c}_{j} \right \|_{2}}, & \mathrm{if}\,\, \mathbf{c}_{j}\neq 0 \\ 
\mathbf{v}, & \mathrm{if}\,\, \mathbf{c}_{j}= 0 
\end{matrix}\right.
\end{equation}
where $\mathbf{v}$ can be any vector on the unit sphere. In particular, here, we set $\mathbf{v}$ to be the first column of the $n \times n $ identity matrix. The solution is unique if and only if $\mathbf{c}_{j}\neq 0$. 
\end{prop}

\begin{proof}
First, for a vector $\mathbf{d}_{j}$ that has unit $\ell_2$ norm, the following holds:
\begin{align} 
& \begin{Vmatrix}
\mathbf{E}_{j} - \mathbf{d}_{j}\mathbf{c}_{j}^{H}
\end{Vmatrix}_{F}^{2}= \left \| \mathbf{E}_{j} \right \|_{F}^{2} +  \left \| \mathbf{c}_{j} \right \|_{2}^{2} - 2\, \text{Re}\left \{ \mathbf{d}_{j}^{H} \mathbf{E}_{j} \mathbf{c}_{j}  \right \}.
\label{eqop12new}
\end{align}
Substituting \eqref{eqop12new} into \eqref{eqop6new}, Problem \eqref{eqop6new} simplifies to
\begin{equation} \label{eqop6bnew}
 \max_{\mathbf{d}_{j} \in \mathbb{C}^{n}} \; \text{Re}\left \{ \mathbf{d}_{j}^{H} \mathbf{E}_{j} \mathbf{c}_{j}  \right \}  \;\:\; \mathrm{s.t.}\; \: \left \| \mathbf{d}_{j} \right \|_2 =1.
\end{equation}
For unit norm $\mathbf{d}_{j}$, by the Cauchy Schwarz inequality, $\text{Re}\left \{ \mathbf{d}_{j}^{H} \mathbf{E}_{j} \mathbf{c}_{j}  \right \} \leq \left |  \mathbf{d}_{j}^{H} \mathbf{E}_{j} \mathbf{c}_{j}  \right |$ $ \leq \left \| \mathbf{E}_{j} \mathbf{c}_{j} \right \|_{2}$. Thus, a solution to \eqref{eqop6bnew} that achieves the value $ \left \| \mathbf{E}_{j} \mathbf{c}_{j} \right \|_{2}$ for the objective is
\begin{equation} \label{tru1ch4g22new}
\hat{\mathbf{d}}_{j} =  \left\{\begin{matrix}
\frac{\mathbf{E}_{j}\mathbf{c}_{j}}{\left \| \mathbf{E}_{j}\mathbf{c}_{j} \right \|_{2}}, & \mathrm{if}\,\, \mathbf{E}_{j}\mathbf{c}_{j} \neq 0 \\ 
\mathbf{v}, & \mathrm{if}\,\, \mathbf{E}_{j}\mathbf{c}_{j} = 0 
\end{matrix}\right.
\end{equation}
Obviously, any $\mathbf{d} \in \mathbb{C}^{n}$ with unit $\ell_{2}$ norm would be a minimizer (non-unique) in \eqref{eqop6bnew} when $\mathbf{E}_{j}\mathbf{c}_{j} = \mathbf{0}$. In particular, the first column of the identity matrix works.

Next, we show that $\mathbf{E}_{j}\mathbf{c}_{j} = \mathbf{0} $ in our algorithm if and only if $\mathbf{c}_{j}=\mathbf{0}$. This result together with \eqref{tru1ch4g22new} immediately establishes the proposition. Since, in the case of (P1), the $\mathbf{c}_{j}$ used in the dictionary atom update step \eqref{eqop6new} was obtained as a minimizer in the preceding sparse coding step \eqref{eqop5new}, we have the following inequality for all $\mathbf{c} \in \mathbb{C}^{N}$ with $\left \| \mathbf{c} \right \|_{\infty} \leq L$, and $\tilde{\mathbf{d}}_{j}$ denotes the $j$th atom in the preceding sparse coding step:
\begin{align} 
 & \hspace{-0.1in} \begin{Vmatrix}
\mathbf{E}_{j} - \tilde{\mathbf{d}}_{j}\mathbf{c}_{j}^{H}
\end{Vmatrix}_{F}^{2} + \lambda^{2} \left \| \mathbf{c}_{j} \right \|_{0}  \leq \begin{Vmatrix}
\mathbf{E}_{j} - \tilde{\mathbf{d}}_{j}\mathbf{c}^{H}
\end{Vmatrix}_{F}^{2} + \lambda^{2} \left \| \mathbf{c} \right \|_{0}.
\label{eqop16bnew}
\end{align}
If $\mathbf{E}_{j}\mathbf{c}_{j} = \mathbf{0}$, the left hand side above simplifies to $\left \| \mathbf{E}_{j} \right \|_{F}^{2}$ $ + \left \| \mathbf{c}_{j} \right \|_{2}^{2}$ $+ \lambda^{2} \left \| \mathbf{c}_{j} \right \|_{0} $, which is clearly minimal when $\mathbf{c}_{j}= \mathbf{0}$. 
For (P2), by replacing the $\ell_{0}$ ``norm'' above with the $\ell_{1}$ norm (and ignoring the condition $\left \| \mathbf{c} \right \|_{\infty} \leq L$), an identical result holds.
Thus, when $\mathbf{E}_{j}\mathbf{c}_{j}=\mathbf{0}$, we must also have $\mathbf{c}_{j} = \mathbf{0}$.
\end{proof}

\section{Convergence Theorems and Proofs} \label{app1bnew}

This section provides a brief proof sketch for Theorems 1-4 corresponding to the algorithms for Problems (P1)-(P4) in our manuscript \citeSupp{sairakfes55:supp}. Appendix \ref{app56new} provides a brief review of the notions of sub-differential and critical points \citeSupp{vari1:supp}.

Recall from Section V of \citeSupp{sairakfes55:supp} that Problem (P1) for $\ell_{0}$ sparsity penalized dictionary learning can be rewritten in an unconstrained form using barrier functions as follows:
\begin{align} 
\nonumber & f(\mathbf{C}, \mathbf{D}) = f\left ( \mathbf{c}_{1}, \mathbf{c}_{2},..., \mathbf{c}_{J}, \mathbf{d}_{1}, \mathbf{d}_{2},..., \mathbf{d}_{J} \right ) =
\lambda^{2} \sum_{j=1}^{J} \left \| \mathbf{c}_{j} \right \|_{0}\\
& \;\; + \begin{Vmatrix}
\mathbf{Y}- \sum_{j=1}^{J} \mathbf{d}_{j}\mathbf{c}_{j}^{H}
\end{Vmatrix}_{F}^{2}  + \sum_{j=1}^{J} \chi (\mathbf{d}_{j})  + \sum_{j=1}^{J} \psi(\mathbf{c}_{j}).  \label{eqop32new}
\end{align}
Problems (P2), (P3), and (P4) can also be similarly rewritten in an unconstrained form with corresponding objectives $\tilde{f}(\mathbf{C}, \mathbf{D})$, $g(\mathbf{C}, \mathbf{D}, \mathbf{y})$, and $\tilde{g}(\mathbf{C}, \mathbf{D}, \mathbf{y})$, respectively \citeSupp{sairakfes55:supp}.

Theorems 1-4 are restated here for completeness along with the proofs. The block coordinate descent algorithms for (P1)-(P4) referred to as SOUP-DILLO, OS-DL \citeSupp{sadeg33:supp}, and SOUP-DILLO and SOUP-DILLI image reconstruction algorithms, respectively, are described in Sections III and IV (cf. Fig. 1 and Fig. 2) of \citeSupp{sairakfes55:supp}.

\subsection{Main Results for SOUP-DILLO and OS-DL} 

The iterates computed in the $t$th outer iteration of SOUP-DILLO (or alternatively in OS-DL) are denoted by the 2J-tuple $\left ( \mathbf{c}_{1}^{t}, \mathbf{d}_{1}^{t}, \mathbf{c}_{2}^{t}, \mathbf{d}_{2}^{t},..., \mathbf{c}_{J}^{t}, \mathbf{d}_{J}^{t} \right )$, or alternatively by the pair of matrices $\left (  \mathbf{C}^{t}, \mathbf{D}^{t} \right )$.

\begin{theorem} \vspace{0.02in}
Let $\left \{ \mathbf{C}^{t}, \mathbf{D}^{t} \right \}$ denote the bounded iterate sequence generated by the SOUP-DILLO Algorithm with training data $\mathbf{Y}  \in \mathbb{C}^{n \times N}$ and initial $(\mathbf{C}^{0}, \mathbf{D}^{0})$. 
Then, the following results hold:
\begin{enumerate}[(i)]
\item The objective sequence  $\left \{ f^{t} \right \}$ with $f^{t} \triangleq f\left ( \mathbf{C}^{t}, \mathbf{D}^{t} \right )$ is monotone decreasing, and converges to a finite value, say $f^{*}=f^{*}(\mathbf{C}^{0}, \mathbf{D}^{0})$.
\item All the accumulation points of the iterate sequence are equivalent in the sense that they achieve the exact same value $f^{*}$ of the objective.
\item Suppose each accumulation point $\left ( \mathbf{C}, \mathbf{D}\right )$ of the iterate sequence is such that the matrix $\mathbf{B}$ with columns $\mathbf{b}_{j} = \mathbf{E}_{j}^{H}\mathbf{d}_{j}$ and $\mathbf{E}_{j} = \mathbf{Y} - \mathbf{D}\mathbf{C}^{H} + \mathbf{d}_{j}\mathbf{c}_{j}^{H}$, has no entry with magnitude $\lambda$. Then every accumulation point of the iterate sequence is a critical point of the objective $f(\mathbf{C}, \mathbf{D})$. Moreover, the two sequences with terms $ \begin{Vmatrix}
\mathbf{D}^{t} - \mathbf{D}^{t-1}
\end{Vmatrix}_{F}$ and $ \begin{Vmatrix}
\mathbf{C}^{t} - \mathbf{C}^{t-1}
\end{Vmatrix}_{F}$ respectively,  both converge to zero.
\end{enumerate}
\end{theorem}

\begin{theorem} \vspace{0.02in}
Let $\left \{ \mathbf{C}^{t}, \mathbf{D}^{t} \right \}$ denote the bounded iterate sequence generated by the OS-DL Algorithm with training data $\mathbf{Y}  \in \mathbb{C}^{n \times N}$ and initial $(\mathbf{C}^{0}, \mathbf{D}^{0})$. 
Then, the iterate sequence converges to an equivalence class of critical points of $\tilde{f}(\mathbf{C}, \mathbf{D})$, and $ \begin{Vmatrix}
\mathbf{D}^{t} - \mathbf{D}^{t-1}
\end{Vmatrix}_{F} \to 0$ and $ \begin{Vmatrix}
\mathbf{C}^{t} - \mathbf{C}^{t-1}
\end{Vmatrix}_{F} \to 0$ as $t \to \infty$.
\end{theorem}
\vspace{0.02in}

\subsection{Proof of Theorem \ref{theorem2}} \label{app2new}

Here, we discuss the proof of Theorem \ref{theorem2} (for the SOUP-DILLO algorithm). The proof for Theorem \ref{theorem4} is very similar, and the distinctions are briefly mentioned in Section \ref{app2cnew}.
We compute all sub-differentials of functions here (and in the later proofs) with respect to the (real-valued) real and imaginary parts of the input variables.

\subsubsection{Equivalence of Accumulation Points} \label{app2a1new}
First, we prove Statements (i) and (ii) of Theorem \ref{theorem2}.
At every iteration $t$ and inner iteration $j$ in the block coordinate descent method (Fig. 1 of \citeSupp{sairakfes55:supp}), we solve the sparse coding (with respect to $\mathbf{c}_{j}$) and dictionary atom update (with respect to $\mathbf{d}_{j}$) subproblems exactly. 
Thus, the objective function decreases in these steps.
Therefore, at the end of the $J$ inner iterations of the $t$th iteration, $f(\mathbf{C}^{t}, \mathbf{D}^{t}) \leq$ $  f(\mathbf{C}^{t-1}, \mathbf{D}^{t-1})$ holds. Since $\left \{ f(\mathbf{C}^{t}, \mathbf{D}^{t}) \right \}$ is monotone decreasing and lower bounded (by $0$), it converges to a finite value $f^{*} = f^{*}(\mathbf{C}^{0}, \mathbf{D}^{0})$ (that may depend on the initial conditions).

The boundedness of the $\left \{ \mathbf{D}^{t} \right \}$ and $\left \{ \mathbf{C}^{t} \right \}$ sequences is obvious from the constraints in (P1). Thus, the accumulation points of the iterates form a non-empty and compact set.
To show that each accumulation point achieves the same value $f^{*}$ of $f$, we consider a convergent subsequence $\left \{ \mathbf{C}^{q_{t}}, \mathbf{D}^{q_{t}} \right \}$ of the iterate sequence with limit $\left ( \mathbf{C}^{*}, \mathbf{D}^{*} \right )$. Because $\left \| \mathbf{d}_{j}^{q_t} \right \|_2 =1$, $\left \| \mathbf{c}_{j}^{q_t} \right \|_{\infty} \leq L$ for all $j$ and every $t$, therefore, due to the continuity of the norms, we have $\left \| \mathbf{d}_{j}^{*} \right \|_2 =1$ and $\left \| \mathbf{c}_{j}^{*} \right \|_{\infty} \leq L$, $\forall$ $j$, i.e.,
\begin{equation} \label{ceyui2new}
\chi(\mathbf{d}_{j}^{*})=0, \, \psi(\mathbf{c}_{j}^{*})=0 \,\; \forall \, j.
\end{equation}

By Proposition 1 of \citeSupp{sairakfes55:supp}, $\mathbf{c}_{j}^{q_t}$ does not have non-zero entries of magnitude less than $\lambda$. Since $\mathbf{c}_{j}^{q_t}  \to \mathbf{c}_{j}^{*}$ entry-wise, we have the following results for each entry $c_{ji}^{*}$ of $\mathbf{c}_{j}^{*}$.
If $c_{ji}^{*} = 0$, then $\exists$ $t_{0} \in \mathbb{N}$ such that  ($i$th entry of $\mathbf{c}_{j}^{q_t}$) $c_{ji}^{q_t}= 0$ for all $t \geq t_{0}$.  Clearly, if $c_{ji}^{*} \neq 0$, then $\exists$ $t_{1} \in \mathbb{N}$ such that $c_{ji}^{q_t} \neq 0$ $\forall$ $t \geq t_{1}$. Thus, we readily have
\begin{equation} \label{vfevnew}
\lim_{t \to \infty} \begin{Vmatrix}
\mathbf{c}_{j}^{q_t} 
\end{Vmatrix}_{0} = \begin{Vmatrix}
\mathbf{c}_{j}^{*} 
\end{Vmatrix}_{0} \, \forall \, j
\end{equation} 
and the convergence in \eqref{vfevnew} happens in a finite number of iterations. We then have the following result:
\begin{align}
 & \lim_{t \to \infty} f(\mathbf{C}^{q_t}, \mathbf{D}^{q_t}) =  \lim_{t \to \infty} \begin{Vmatrix}
\mathbf{Y}-  \mathbf{D}^{q_t} \left ( \mathbf{C}^{q_t} \right )^{H}
\end{Vmatrix}_{F}^{2}  \label{uo1new}  \\ 
\nonumber &   +   \lambda^{2} \sum_{j=1}^{J} \lim_{t \to \infty} \left \| \mathbf{c}_{j}^{q_t} \right \|_{0} = \lambda^{2} \sum_{j=1}^{J} \left \| \mathbf{c}_{j}^{*} \right \|_{0}  + \begin{Vmatrix}
\mathbf{Y}-  \mathbf{D}^{*} \left ( \mathbf{C}^{*} \right )^{H}
\end{Vmatrix}_{F}^{2}  
\end{align}
The right hand side above coincides with $ f(\mathbf{C}^{*}, \mathbf{D}^{*})$. Since the objective sequence converges to $f^{*}$, therefore, $f(\mathbf{C}^{*}, \mathbf{D}^{*})   = \lim_{t \to \infty} f(\mathbf{C}^{q_t}, \mathbf{D}^{q_t}) = f^{*}$. $\;\;\; \blacksquare$

\subsubsection{Critical Point Property} \label{app2anew}

Consider a convergent subsequence $\left \{ \mathbf{C}^{q_t}, \mathbf{D}^{q_t} \right \}$ of the iterate sequence in the SOUP-DILLO Algorithm with limit $(\mathbf{C}^{*}, \mathbf{D}^{*})$. Let $\left \{ \mathbf{C}^{q_{n_t}+1}, \mathbf{D}^{q_{n_t} +1} \right \}$ be a convergent subsequence of the bounded $\left \{ \mathbf{C}^{q_{t}+1}, \mathbf{D}^{q_{t} +1} \right \}$, with limit $(\mathbf{C}^{**}, \mathbf{D}^{**})$.
For each iteration $t$ and inner iteration $j$ in the algorithm, define the matrix $\mathbf{E}_{j}^{t} \triangleq \mathbf{Y} - \sum_{k<j} \mathbf{d}_{k}^{t} \left ( \mathbf{c}_{k}^{t} \right )^{H}$ $ - \sum_{k>j} \mathbf{d}_{k}^{t-1} \left ( \mathbf{c}_{k}^{t-1} \right )^{H}$. For the accumulation point $(\mathbf{C}^{*}, \mathbf{D}^{*})$, let $\mathbf{E}_{j}^{*} \triangleq \mathbf{Y} - \mathbf{D}^{*} \left ( \mathbf{C}^{*} \right )^{H} $ $ + \mathbf{d}_{j}^{*} \left ( \mathbf{c}_{j}^{*} \right )^{H}$. In this proof, for simplicity, we denote the objective $f$ \eqref{eqop32new} in the $j$th sparse coding step of iteration $t$ (Fig. 1 of \citeSupp{sairakfes55:supp}) as 
\begin{align} 
\nonumber & f\left ( \mathbf{E}_{j}^{t},\mathbf{c}_{j}, \mathbf{d}_{j}^{t-1} \right ) \triangleq \begin{Vmatrix}
\mathbf{E}_{j}^{t}- \mathbf{d}_{j}^{t-1}\mathbf{c}_{j}^{H}
\end{Vmatrix}_{F}^{2} +
\lambda^{2} \sum_{k < j} \left \| \mathbf{c}_{k}^{t} \right \|_{0} \\
&\;\;\; + \lambda^{2} \sum_{k > j} \left \| \mathbf{c}_{k}^{t-1} \right \|_{0} + \lambda^{2} \left \| \mathbf{c}_{j} \right \|_{0} + \psi(\mathbf{c}_{j}). \label{eqopp20new}
\end{align}
All but the $j$th atom and sparse vector $\mathbf{c}_{j}$ are represented via $\mathbf{E}_{j}^{t}$ on the left hand side in this notation. The objective that is minimized in the dictionary atom update step is similarly denoted as $f\left ( \mathbf{E}_{j}^{t},\mathbf{c}_{j}^{t}, \mathbf{d}_{j} \right )$ with 
\begin{align} 
\nonumber & f\left ( \mathbf{E}_{j}^{t},\mathbf{c}_{j}^{t}, \mathbf{d}_{j} \right )\triangleq \begin{Vmatrix}
\mathbf{E}_{j}^{t}- \mathbf{d}_{j}\left ( \mathbf{c}_{j}^{t} \right )^{H}
\end{Vmatrix}_{F}^{2} +
\lambda^{2} \sum_{k \leq j} \left \| \mathbf{c}_{k}^{t} \right \|_{0} \\
&\;\;\; + \lambda^{2} \sum_{k > j} \left \| \mathbf{c}_{k}^{t-1} \right \|_{0} + \chi(\mathbf{d}_{j}). \label{eqopp20bbbnew}
\end{align}
Finally, the functions $f\left ( \mathbf{E}_{j}^{*},\mathbf{c}_{j}, \mathbf{d}_{j}^{*} \right )$ and $f\left ( \mathbf{E}_{j}^{*},\mathbf{c}_{j}^{*}, \mathbf{d}_{j} \right )$ are defined in a similar way with respect to the accumulation point $(\mathbf{C}^{*}, \mathbf{D}^{*})$.

To establish the critical point property of $(\mathbf{C}^{*}, \mathbf{D}^{*})$, we first show the partial global optimality of each column of the matrices $\mathbf{C}^{*}$ and $\mathbf{D}^{*}$ for $f$. By partial global optimality, we mean that each column of $\mathbf{C}^{*}$ (or $\mathbf{D}^{*}$) is a global minimizer of $f$, when all other variables are kept fixed to the values in $(\mathbf{C}^{*}, \mathbf{D}^{*})$.
First, for $j=1$ and iteration $q_{n_t} +1$, we have the following result for the sparse coding step for all $\mathbf{c}_{1} \in \mathbb{C}^{N}$:
\begin{equation} \label{eqopp7new}
f\left ( \mathbf{E}_{1}^{q_{n_t}+1},\mathbf{c}_{1}^{q_{n_t}+1}, \mathbf{d}_{1}^{q_{n_t}} \right ) \leq f\left ( \mathbf{E}_{1}^{q_{n_t}+1},\mathbf{c}_{1}, \mathbf{d}_{1}^{q_{n_t}} \right ).
\end{equation}
Taking the limit $t \to \infty$ above and using \eqref{vfevnew} to obtain limits of $\ell_{0}$ terms in the cost \eqref{eqopp20new}, and using \eqref{ceyui2new}, and the fact that $\mathbf{E}_{1}^{q_{n_t}+1} \to \mathbf{E}_{1}^{*}$, we have
\begin{equation} \label{eqopp31new}
f\left ( \mathbf{E}_{1}^{*},\mathbf{c}_{1}^{**}, \mathbf{d}_{1}^{*} \right ) \leq f\left ( \mathbf{E}_{1}^{*},\mathbf{c}_{1}, \mathbf{d}_{1}^{*} \right ) \, \forall \, \mathbf{c}_{1} \in \mathbb{C}^{N}.
\end{equation} 
This means that $\mathbf{c}_{1}^{**}$ is a minimizer of $f$ with all other variables fixed to their values in $(\mathbf{C}^{*}, \mathbf{D}^{*})$. Because of the (uniqueness) assumption in the theorem, we have
\begin{equation} \label{gobop1new}
\mathbf{c}_{1}^{**} = \underset{\mathbf{c}_{1}}{\arg \min} \, f\left ( \mathbf{E}_{1}^{*},\mathbf{c}_{1}, \mathbf{d}_{1}^{*} \right ).
\end{equation}
Furthermore, because of the equivalence of accumulation points, $f\left ( \mathbf{E}_{1}^{*},\mathbf{c}_{1}^{**}, \mathbf{d}_{1}^{*} \right )=$ $f\left ( \mathbf{E}_{1}^{*},\mathbf{c}_{1}^{*}, \mathbf{d}_{1}^{*} \right )=f^{*}$ holds. This result together with \eqref{gobop1new} implies that $\mathbf{c}_{1}^{**}=\mathbf{c}_{1}^{*}$ and
\begin{equation} \label{gobop1anew}
\mathbf{c}_{1}^{*} = \underset{\mathbf{c}_{1}}{\arg \min} \, f\left ( \mathbf{E}_{1}^{*},\mathbf{c}_{1}, \mathbf{d}_{1}^{*} \right ).
\end{equation}
Therefore, $\mathbf{c}_{1}^{*}$ is a partial global minimizer of $f$, or $0 \in \partial  f_{\mathbf{c}_{1}}\left (\mathbf{E}_{1}^{*}, \mathbf{c}_{1}^{*}, \mathbf{d}_{1}^{*}\right )$.

Next, for the first dictionary atom update step ($j=1$) in iteration $q_{n_t} +1$, we have the following for all $\mathbf{d}_{1} \in \mathbb{C}^{n}$:
\begin{equation} \label{eqopp7bnew}
f\left ( \mathbf{E}_{1}^{q_{n_t}+1},\mathbf{c}_{1}^{q_{n_t}+1}, \mathbf{d}_{1}^{q_{n_t}+1} \right ) \leq f\left ( \mathbf{E}_{1}^{q_{n_t}+1},\mathbf{c}_{1}^{q_{n_t}+1}, \mathbf{d}_{1} \right ).
\end{equation}
Just like in \eqref{eqopp7new}, upon taking the limit $t \to \infty$ above and using $\mathbf{c}_{1}^{**}=\mathbf{c}_{1}^{*}$, we get 
\begin{equation} \label{eqopp31bnew}
f\left ( \mathbf{E}_{1}^{*},\mathbf{c}_{1}^{*}, \mathbf{d}_{1}^{**} \right ) \leq f\left ( \mathbf{E}_{1}^{*},\mathbf{c}_{1}^{*}, \mathbf{d}_{1} \right ) \, \forall \, \mathbf{d}_{1} \in \mathbb{C}^{n}.
\end{equation} 
Thus, $\mathbf{d}_{1}^{**}$ is a minimizer of $ f\left ( \mathbf{E}_{1}^{*},\mathbf{c}_{1}^{*}, \mathbf{d}_{1} \right )$ with respect to $\mathbf{d}_{1}$. Because of the equivalence of accumulation points, we have $f\left ( \mathbf{E}_{1}^{*},\mathbf{c}_{1}^{*}, \mathbf{d}_{1}^{**} \right )=$ $f\left ( \mathbf{E}_{1}^{*},\mathbf{c}_{1}^{*}, \mathbf{d}_{1}^{*} \right )=f^{*}$. This implies that $\mathbf{d}_{1}^{*}$ is also a partial global minimizer of $f$ in \eqref{eqopp31bnew} satisfying
\begin{equation} \label{gobop1bnew}
\mathbf{d}_{1}^{*} \in \underset{\mathbf{d}_{1}}{\arg \min} \, f\left ( \mathbf{E}_{1}^{*},\mathbf{c}_{1}^{*}, \mathbf{d}_{1} \right ) 
\end{equation}
or $0 \in \partial  f_{\mathbf{d}_{1}}\left (\mathbf{E}_{1}^{*}, \mathbf{c}_{1}^{*}, \mathbf{d}_{1}^{*}\right )$.
By Proposition 3 of \citeSupp{sairakfes55:supp}, the minimizer of the dictionary atom update cost is unique as long as the corresponding sparse code (in \eqref{gobop1anew}) is non-zero. Thus, $\mathbf{d}_{1}^{**}= \mathbf{d}_{1}^{*}$ is the \emph{unique} minimizer in \eqref{gobop1bnew}, except when $\mathbf{c}_{1}^{*}=0$.

When  $\mathbf{c}_{1}^{*}=0$, we use \eqref{vfevnew} to conclude that $\mathbf{c}_{1}^{q_t}=0$ for all sufficiently large $t$ values. Since $\mathbf{c}_{1}^{**}=\mathbf{c}_{1}^{*}$, we must also have that $\mathbf{c}_{1}^{q_{n_t}+1} = 0$ for all large enough $t$.
Therefore, for all sufficiently large $t$, $\mathbf{d}_{1}^{q_t}$ and $ \mathbf{d}_{1}^{q_{n_t}+1} $ are the minimizers of dictionary atom update steps, wherein the corresponding sparse coefficients $\mathbf{c}_{1}^{q_t}$ and $\mathbf{c}_{1}^{q_{n_t}+1}$ are zero, implying that $\mathbf{d}_{1}^{q_t} = \mathbf{d}_{1}^{q_{n_t}+1} = \mathbf{v} $ (with $\mathbf{v}$ the first column of the $n \times n$ identity, as in Proposition 3 of \citeSupp{sairakfes55:supp}) for all sufficiently large $t$. Thus, the limits satisfy $\mathbf{d}_{1}^{*}= \mathbf{d}_{1}^{**} = \mathbf{v}$.
Therefore, $\mathbf{d}_{1}^{**}= \mathbf{d}_{1}^{*}$ holds, even when $\mathbf{c}_{1}^{*} = 0$.
Therefore, for $j=1$,
\begin{equation} \label{zebra1new}
\mathbf{d}_{1}^{**} =  \mathbf{d}_{1}^{*}, \,\,\, \mathbf{c}_{1}^{**} =  \mathbf{c}_{1}^{*}.
\end{equation}

Next, we repeat the above procedure by considering first the sparse coding step and then the dictionary atom update step for $j=2$ and iteration $q_{n_t} +1$.
For $j=2$, we consider the matrix
\begin{equation}
\mathbf{E}_{2}^{q_{n_t}+1} = \mathbf{Y} - \sum_{k>2}\mathbf{d}_{k}^{q_{n_t}}\left ( \mathbf{c}_{k}^{q_{n_t}} \right )^{H}  - \mathbf{d}_{1}^{q_{n_t}+1}\left ( \mathbf{c}_{1}^{q_{n_t}+1} \right )^{H}. \label{rollp1new}
\end{equation} 
It follows from \eqref{zebra1new} that $\mathbf{E}_{2}^{q_{n_t}+1}  \to \mathbf{E}_{2}^{*}$ as $t \to \infty$. Then, by repeating the steps \eqref{eqopp7new} - \eqref{zebra1new} for $j=2$, we can easily show that $\mathbf{c}_{2}^{*}$ and $\mathbf{d}_{2}^{*}$ are each partial global minimizers of $f$ when all other variables are fixed to their values in $(\mathbf{C}^{*}, \mathbf{D}^{*})$. Moreover, $\mathbf{c}_{2}^{**} =  \mathbf{c}_{2}^{*}$ and $\mathbf{d}_{2}^{**} =  \mathbf{d}_{2}^{*}$. Similar such arguments can be repeated sequentially for each next $j$ until $j=J$.

Finally, the partial global optimality of each column of $\mathbf{C}^{*}$ and $\mathbf{D}^{*}$ for the cost $f$ implies (use Proposition 3 in \citeSupp{Attouchaa:supp}) that $0 \in \partial  f\left ( \mathbf{C}^{*}, \mathbf{D}^{*} \right )$, i.e., $(\mathbf{C}^{*}, \mathbf{D}^{*})$ is a critical point of $f$.
$\;\;\; \blacksquare$

\subsubsection{Convergence of the Difference between Successive Iterates} \label{app2bnew}

Consider the sequence $ \left \{ a^{t} \right \}$ whose elements are  $a^{t} \triangleq  \begin{Vmatrix}
\mathbf{D}^{t} - \mathbf{D}^{t-1} 
\end{Vmatrix}_{F} $. Clearly, this sequence is bounded because of the unit norm constraints on the dictionary atoms. We will show that every convergent subsequence of this bounded scalar sequence converges to zero, thereby implying that zero is the limit inferior and the limit superior of the sequence, i.e., $ \left \{ a^{t} \right \}$ converges to 0. 
A similar argument establishes that $\begin{Vmatrix}
\mathbf{C}^{t} - \mathbf{C}^{t-1} 
\end{Vmatrix}_{F} \to 0$ as $t \to \infty$.

Consider a convergent subsequence $ \left \{ a^{q_{t}} \right \}$ of the sequence $ \left \{ a^{t} \right \}$.
The bounded sequence $\left \{ \left ( \mathbf{C}^{q_{t}-1}, \mathbf{D}^{q_{t}-1}, \mathbf{C}^{q_{t}}, \mathbf{D}^{q_{t}} \right ) \right \}$ (whose elements are formed by pairing successive elements of the iterate sequence)  must have a convergent subsequence $\left \{ \left ( \mathbf{C}^{q_{n_t} - 1}, \mathbf{D}^{q_{n_t} - 1}, \mathbf{C}^{q_{n_t}}, \mathbf{D}^{q_{n_t}} \right ) \right \}$ that converges to a point $(\mathbf{C}^{*}, \mathbf{D}^{*}, \mathbf{C}^{**}, \mathbf{D}^{**})$.
Based on the results in Section \ref{app2anew}, we have $\mathbf{d}_{j}^{**} = \mathbf{d}_{j}^{*} $ (and $\mathbf{c}_{j}^{**} = \mathbf{c}_{j}^{*} $) for each $1\leq j \leq J$, or
\begin{equation} \label{zebra2new}
\mathbf{D}^{**} = \mathbf{D}^{*}.
\end{equation}
Thus, clearly $a^{q_{n_t}} \to 0$ as $t \to \infty$. 
Since, $ \left \{ a^{q_{n_t}} \right \}$ is a subsequence of the convergent $ \left \{ a^{q_{t}} \right \}$, we must have $a^{q_{t}} \to 0$ too. Thus, we have shown that zero is the limit of any arbitrary convergent subsequence of the bounded sequence $ \left \{ a^{t} \right \}$.
$\;\;\;\blacksquare$

\subsection{Proof Sketch for Theorem \ref{theorem4}} \label{app2cnew}

The proof of Theorem \ref{theorem4} follows using the same line of arguments as in Section \ref{app2new}, except that the $\ell_{0}$ ``norm'' is replaced with the continuous $\ell_{1}$ norm, and we use the fact that the minimizer of the sparse coding problem in this case is unique (see Proposition 2 of \citeSupp{sairakfes55:supp}).
Thus, the proof of Theorem~\ref{theorem4} is simpler, and results such as \eqref{vfevnew} and \eqref{gobop1new} follow easily with the $\ell_{1}$ norm. $\;\;\; \blacksquare$

\subsection{Main Results for the SOUP-DILLO and SOUP-DILLI Image Reconstruction Algorithms} 

\begin{theorem} \vspace{0.02in}
Let $\left \{ \mathbf{C}^{t}, \mathbf{D}^{t}, \mathbf{y}^{t} \right \}$ denote the iterate sequence generated by the SOUP-DILLO image reconstruction Algorithm with measurements $\mathbf{z} \in \mathbb{C}^{m}$ and initial $(\mathbf{C}^{0}, \mathbf{D}^{0}, \mathbf{y}^{0})$. Then, the following results hold:
\begin{enumerate}[(i)]
\item The objective sequence  $\left \{ g^{t} \right \}$ with $g^{t} \triangleq g\left ( \mathbf{C}^{t}, \mathbf{D}^{t}, \mathbf{y}^{t} \right )$ is monotone decreasing, and converges to a finite value, say $g^{*}=g^{*}(\mathbf{C}^{0}, \mathbf{D}^{0}, \mathbf{y}^{0})$.
\item The iterate sequence is bounded, and all its accumulation points are equivalent in the sense that they achieve the exact same value $g^{*}$ of the objective.
\item Each accumulation point $(\mathbf{C}, \mathbf{D}, \mathbf{y})$ of the iterate sequence satisfies
\begin{equation}
 \mathbf{y} \in  \underset{\tilde{\mathbf{y}}}{\arg\min} \; \,  g(\mathbf{C}, \mathbf{D}, \tilde{\mathbf{y}})     \label{cnbcs4anew}
\end{equation}
\item The sequence $\left \{ a^{t} \right \}$ with $a^{t} \triangleq \left \| \mathbf{y}^{t} - \mathbf{y}^{t-1} \right \|_{2}$, converges to zero.
\item Suppose each accumulation point $(\mathbf{C}, \mathbf{D}, \mathbf{y})$ of the iterates is such that the matrix $\mathbf{B}$ with columns $\mathbf{b}_{j} = \mathbf{E}_{j}^{H}\mathbf{d}_{j}$ and $\mathbf{E}_{j} = \mathbf{Y} - \mathbf{D}\mathbf{C}^{H} + \mathbf{d}_{j}\mathbf{c}_{j}^{H}$, has no entry with magnitude $\lambda$. Then every accumulation point of the iterate sequence is a critical point of the objective $g$. Moreover, $ \begin{Vmatrix}
\mathbf{D}^{t} - \mathbf{D}^{t-1}
\end{Vmatrix}_{F} \to 0$ and $ \begin{Vmatrix}
\mathbf{C}^{t} - \mathbf{C}^{t-1}
\end{Vmatrix}_{F} \to 0$ as $t \to \infty$.
\end{enumerate}
\end{theorem}

\begin{theorem} \vspace{0.02in}
Let $\left \{ \mathbf{C}^{t}, \mathbf{D}^{t}, \mathbf{y}^{t} \right \}$ denote the iterate sequence generated by the SOUP-DILLI image reconstruction Algorithm for (P4) with measurements $\mathbf{z} \in \mathbb{C}^{m}$ and initial $(\mathbf{C}^{0}, \mathbf{D}^{0}, \mathbf{y}^{0})$.
Then, the iterate sequence converges to an equivalence class of critical points of $\tilde{g}(\mathbf{C}, \mathbf{D}, \mathbf{y})$. Moreover, $ \begin{Vmatrix}
\mathbf{D}^{t} - \mathbf{D}^{t-1}
\end{Vmatrix}_{F} \to 0$, $ \begin{Vmatrix}
\mathbf{C}^{t} - \mathbf{C}^{t-1}
\end{Vmatrix}_{F} \to 0$, and $\begin{Vmatrix} \mathbf{y}^{t} - \mathbf{y}^{t-1} \end{Vmatrix}_{2} \to 0$ as $t \to \infty$.
\end{theorem}

\subsection{Proofs of Theorems \ref{theorem5} and \ref{theorem6}} \label{app3new}

First, recall from Section IV.B of our manuscript \citeSupp{sairakfes55:supp} that the unique solution in the image update step of the algorithms for (P3) and (P4) satisfies the following equation:
\begin{align} 
 & \left (  \sum_{i=1}^{N} \mathbf{P}_{i}^{T} \mathbf{P}_{i} \; +\;\nu\: \mathbf{A}^{H}\mathbf{A} \right )\mathbf{y}=  \sum_{i=1}^{N} \mathbf{P}_{i}^{T}\mathbf{D} \mathbf{x}_{i} + \nu \: \mathbf{A}^{H}\mathbf{z}. \label{bcs10new}
\end{align}

We provide a brief proof sketch for Theorem \ref{theorem5} for the SOUP-DILLO image reconstruction algorithm (see Fig. 2 of \citeSupp{sairakfes55:supp}). The proof for Theorem \ref{theorem6} differs in a similar manner as the proof for Theorem \ref{theorem4} differs from that of Theorem \ref{theorem2} in Section \ref{app2cnew}.
Assume an initialization $ ( \mathbf{C}^{0}, \mathbf{D}^{0}, \mathbf{y}^{0} )$ in the algorithms. We discuss the proofs of Statements (i)-(v) in Theorem \ref{theorem5} one by one.

\textbf{Statement (i):} Since we perform exact block coordinate descent over the variables $\left \{ \mathbf{c}_{j} \right \}_{j=1}^{J}$, $\left \{ \mathbf{d}_{j} \right \}_{j=1}^{J}$, and $\mathbf{y}$ of the cost $g\left ( \mathbf{C}, \mathbf{D}, \mathbf{y} \right )$ (in Problem (P3)), the objective sequence $\left \{ g^{t} \right \}$ is monotone decreasing. Since  $g$ is non-negative, $\left \{ g^{t} \right \}$ thus converges to some finite value $g^{*}$.

\textbf{Statement (ii):} While the boundedness of $\left \{ \mathbf{D}^{t} \right \}$ and $\left \{ \mathbf{C}^{t} \right \}$ is obvious (because of the constraints in (P3)), $\mathbf{y}^{t}$ is obtained by solving \eqref{bcs10new} with respect to $\mathbf{y}$ with the other variables fixed. Since the right hand side of \eqref{bcs10new} is bounded (independently of $t$), and the fixed matrix pre-multiplying $\mathbf{y}$ in \eqref{bcs10new}  is positive-definite, the iterate $\mathbf{y}^{t}$ is bounded (in norm by a constant independent of $t$) too.

\textbf{In the remainder of this section}, we consider a (arbitrary) convergent subsequence $\left \{ \mathbf{C}^{q_t}, \mathbf{D}^{q_t},  \mathbf{y}^{q_t} \right \}$ of the iterate sequence with limit $(\mathbf{C}^{*}, \mathbf{D}^{*}, \mathbf{y}^{*})$. Recall the following definition of $g$:
\begin{align} 
\nonumber & g(\mathbf{C}, \mathbf{D}, \mathbf{y})  = \nu \left \| \mathbf{A}\mathbf{y}-\mathbf{z} \right \|_{2}^{2} + \begin{Vmatrix}
\mathbf{Y}- \sum_{j=1}^{J} \mathbf{d}_{j}\mathbf{c}_{j}^{H}
\end{Vmatrix}_{F}^{2}\\
& \;\; + \lambda^{2} \sum_{j=1}^{J} \left \| \mathbf{c}_{j} \right \|_{0}   + \sum_{j=1}^{J} \chi (\mathbf{d}_{j})  + \sum_{j=1}^{J} \psi(\mathbf{c}_{j}).  \label{eqop32bbnew}
\end{align}
Then, similar to  \eqref{uo1new} in Section \ref{app2a1new}, by considering $\lim_{t \to \infty} g\left ( \mathbf{C}^{q_t}, \mathbf{D}^{q_t}, \mathbf{y}^{q_t} \right )$ above and using  \eqref{vfevnew}, \eqref{ceyui2new} and the fact that $g^{q_t} \to g^{*}$ yields 
\begin{equation}
g^{*} = g\left ( \mathbf{C}^{*}, \mathbf{D}^{*}, \mathbf{y}^{*} \right ). \label{jkw13new}
\end{equation}

\textbf{Statement (iii):} The following optimality condition holds in the image update step of the Algorithm for all $\mathbf{y} \in \mathbb{C}^{p}$:
\begin{equation}
g\left ( \mathbf{C}^{q_t}, \mathbf{D}^{q_t}, \mathbf{y}^{q_t} \right ) \leq g\left ( \mathbf{C}^{q_t}, \mathbf{D}^{q_t}, \mathbf{y} \right ). \label{jkw12new}
\end{equation}
Taking $t \to \infty$ in \eqref{jkw12new} and using \eqref{vfevnew}, \eqref{ceyui2new}, and \eqref{jkw13new} yields 
\begin{equation}
g\left ( \mathbf{C}^{*}, \mathbf{D}^{*}, \mathbf{y}^{*} \right ) \leq g\left ( \mathbf{C}^{*}, \mathbf{D}^{*}, \mathbf{y} \right ) \label{jkw15new}
\end{equation}
for each $\mathbf{y} \in \mathbb{C}^{p}$, thereby establishing \eqref{cnbcs4anew}.

\textbf{Statement (iv):} Let $\left \{ \mathbf{y}^{q_{n_{t}}-1} \right \}$ be an arbitrary convergent subsequence of the bounded $\left \{ \mathbf{y}^{q_{t}-1} \right \}$, with limit $\mathbf{y}^{**}$. Then, using similar arguments as for \eqref{jkw13new}, we get $\lim_{t \to \infty} g\left ( \mathbf{C}^{q_{n_t}}, \mathbf{D}^{q_{n_t}}, \mathbf{y}^{q_{n_t}-1} \right ) $ $ = g\left ( \mathbf{C}^{*}, \mathbf{D}^{*}, \mathbf{y}^{**} \right )$. Moreover, $g\left ( \mathbf{C}^{*}, \mathbf{D}^{*}, \mathbf{y}^{**} \right ) = g\left ( \mathbf{C}^{*}, \mathbf{D}^{*}, \mathbf{y}^{*} \right )$ (equivalence of accumulation points), which together with \eqref{jkw15new} and the uniqueness of the minimizer in \eqref{bcs10new} implies that $\mathbf{y}^{**}=\mathbf{y}^{*}$ is the unique partial global minimizer of $g\left ( \mathbf{C}^{*}, \mathbf{D}^{*}, \mathbf{y} \right )$. Since $\mathbf{y}^{**}=\mathbf{y}^{*}$ holds for the limit of any arbitrary convergent subsequence of the bounded $\left \{ \mathbf{y}^{q_{t}-1} \right \}$, so $\mathbf{y}^{q_{t}-1} \to \mathbf{y}^{*}$.
Finally, using similar arguments as in Section \ref{app2bnew}, we can show that $0$ is the limit of any convergent subsequence of the bounded $\left \{ a^{t} \right \}$ with $a^{t} \triangleq \left \| \mathbf{y}^{t} - \mathbf{y}^{t-1} \right \|_{2}$. Thus, $\left \| \mathbf{y}^{t} - \mathbf{y}^{t-1} \right \|_{2} \to 0$ as $t \to \infty$.

\textbf{Statement (v):} For simplicity, first, we consider the case of (number of iterations of SOUP-DILLO in Fig. 2 of \citeSupp{sairakfes55:supp}) $K=1$.
Let $\left \{ \mathbf{C}^{q_{n_t}+1}, \mathbf{D}^{q_{n_t} +1} \right \}$ be a convergent subsequence of the bounded $\left \{ \mathbf{C}^{q_{t}+1}, \mathbf{D}^{q_{t} +1} \right \}$, with limit $(\mathbf{C}^{**}, \mathbf{D}^{**})$.

Let $\mathbf{E}_{j}^{t} \triangleq \mathbf{Y}^{t-1} - \sum_{k<j} \mathbf{d}_{k}^{t} \left ( \mathbf{c}_{k}^{t} \right )^{H}$ $ - \sum_{k>j} \mathbf{d}_{k}^{t-1} \left ( \mathbf{c}_{k}^{t-1} \right )^{H}$. For the accumulation point $(\mathbf{C}^{*}, \mathbf{D}^{*}, \mathbf{y}^{*})$, let $\mathbf{E}_{j}^{*} \triangleq \mathbf{Y}^{*} - \mathbf{D}^{*} \left ( \mathbf{C}^{*} \right )^{H} $ $ + \mathbf{d}_{j}^{*} \left ( \mathbf{c}_{j}^{*} \right )^{H}$.
Similar to \eqref{eqopp20new} and \eqref{eqopp20bbbnew} of Section~\ref{app2anew}, we denote the objectives in the $j$th inner sparse coding and dictionary atom update steps (only 1 iteration of SOUP-DILLO) of iteration $t$ in Fig. 2 of \citeSupp{sairakfes55:supp} as $g\left ( \mathbf{E}_{j}^{t},\mathbf{c}_{j}, \mathbf{d}_{j}^{t-1} \right )$ and
 $g\left ( \mathbf{E}_{j}^{t},\mathbf{c}_{j}^{t}, \mathbf{d}_{j} \right )$, respectively.
 The functions $g\left ( \mathbf{E}_{j}^{*},\mathbf{c}_{j}, \mathbf{d}_{j}^{*} \right )$ and $g\left ( \mathbf{E}_{j}^{*},\mathbf{c}_{j}^{*}, \mathbf{d}_{j} \right )$ are defined with respect to $(\mathbf{C}^{*}, \mathbf{D}^{*},\mathbf{y}^{*})$.

Then, we use the same series of arguments as in
Section~\ref{app2anew} but with respect to $g$ to show the partial global optimality of each column of $\mathbf{C}^{*}$ and $\mathbf{D}^{*}$ for the cost $g$ (and that $\mathbf{D}^{**} = \mathbf{D}^{*}$, $\mathbf{C}^{**} = \mathbf{C}^{*}$). Statement (iii) established that $\mathbf{y}^{*}$ is a global minimizer of $g\left ( \mathbf{C}^{*}, \mathbf{D}^{*}, \mathbf{y} \right )$. These results together imply (with Proposition 3 of \citeSupp{Attouchaa:supp}) that  $0 \in \partial g\left ( \mathbf{C}^{*}, \mathbf{D}^{*}, \mathbf{y}^{*} \right )
$ (critical point property).
Finally, by using similar arguments as in Section \ref{app2bnew}, we get $\begin{Vmatrix}
\mathbf{D}^{t} - \mathbf{D}^{t-1} 
\end{Vmatrix}_{F} \to 0$ and $\begin{Vmatrix}
\mathbf{C}^{t} - \mathbf{C}^{t-1} 
\end{Vmatrix}_{F} \to 0$. 

Although we considered $K=1$ above for simplicity, the $K>1$ case can be handled in a similar fashion by: 1) considering the set of (bounded) inner iterates (dictionaries and sparse coefficient matrices) generated during the $K$ iterations of SOUP-DILLO in the $t$th (outer) iteration in Fig. 2 of \citeSupp{sairakfes55:supp}, and 2) arranging these $K$ inner iterates as an ordered tuple, and 3) picking a convergent subsequence of such a sequence (over $t$) of tuples.
Then, the arguments in Section~\ref{app2anew} are repeated sequentially for each (inner) sparse coding and dictionary atom update step corresponding to each element in the above tuple (sequentially) to arrive at the same results as for $K=1$.
 $\;\;\; \blacksquare$

\section{Dictionary Learning Convergence Experiment} \label{sec5bnew}

To study the convergence behavior of the proposed SOUP-DILLO Algorithm for (P1), we use the same training data as in Section VI.B of our manuscript (i.e., we extract $3 \times 10^{4}$ patches of size $8 \times 8$ from randomly chosen locations in the $512 \times 512$ standard images Barbara, Boats, and Hill). The SOUP-DILLO Algorithm was used to learn a $64 \times 256$ overcomplete dictionary, with $\lambda = 69$.
The initial estimate for $\mathbf{C}$ was an all-zero matrix, and the initial estimate for $\mathbf{D}$ was the overcomplete DCT \citeSupp{elad2:supp, el2:supp}.
For comparison, we also test the convergence behavior of the recent OS-DL \citeSupp{sadeg33:supp} method for (P2). The data and parameter settings are the same as discussed above, but with $\mu = 615$. The $\mu$ was chosen to achieve the same eventual sparsity factor ($\sum_{j=1}^{J}\left ( \left \| \mathbf{c}_{j} \right \|_{0} / nN \right )$) for $\mathbf{C}$ in OS-DL as that achieved in SOUP-DILLO with $\lambda = 69$.

Fig. \ref{im2new} illustrates the convergence behaviors of the SOUP-DILLO and OS-DL methods for sparsity penalized dictionary learning. The objectives in these methods (Fig. \ref{im2new}(a)) converged monotonically and quickly over the iterations. Figs. \ref{im2new}(b) and \ref{im2new}(c) show the normalized sparse representation error and sparsity factor (for $\mathbf{C}$) respectively, with both quantities expressed as percentages. Both these quantities converged quickly for the SOUP-DILLO algorithm, and the NSRE improved by 1 dB beyond the first iteration, indicating the success of the SOUP-DILLO approach in representing data using a small number of non-zero coefficients (sparsity factor of $3.1\%$ at convergence).
For the same eventual (net) sparsity factor, SOUP-DILLO converged to a significantly lower NSRE than the $\ell_{1}$ norm-based OS-DL.

Finally, both the quantities $\left \| \mathbf{D}^{t} - \mathbf{D}^{t-1} \right \|_{F}$ (Fig. \ref{im2new}(d)) and $\left \| \mathbf{C}^{t} - \mathbf{C}^{t-1} \right \|_{F}$ (Fig. \ref{im2new}(e)) decrease towards $0$ for SOUP-DILLO and OS-DL, as predicted by Theorems \ref{theorem2} and \ref{theorem4}. However, the convergence is eventually quicker for the $\ell_{0}$ ``norm''-based SOUP-DILLO than for the $\ell_{1}$ norm-based OS-DL\footnote{The plots for SOUP-DILLO show some intermediate fluctuations in the changes between successive iterates at small values of these errors, which eventually vanish. This behavior is not surprising because Theorems \ref{theorem2} and \ref{theorem4} do not guarantee a strictly monotone decrease of $\left \| \mathbf{D}^{t} - \mathbf{D}^{t-1} \right \|_{F}$ or $\left \| \mathbf{C}^{t} - \mathbf{C}^{t-1} \right \|_{F}$, but rather their asymptotic convergence to $0$.}.
The above results are also indicative (are necessary but not sufficient conditions) of the convergence of the entire sequences $\left \{ \mathbf{D}^{t} \right \}$ and $\left \{ \mathbf{C}^{t} \right \}$ for the block coordinate descent methods (for both (P1) and (P2)) in practice.
In contrast, Bao et al. \citeSupp{bao1:supp} showed that the distance between successive iterates may not converge to zero for popular algorithms such as K-SVD.

\vspace{-0.05in}
\section{Adaptive Sparse Representation of Data: Additional Results}

\begin{figure}
\begin{center}
\begin{tabular}{c}
\includegraphics[height=1.46in]{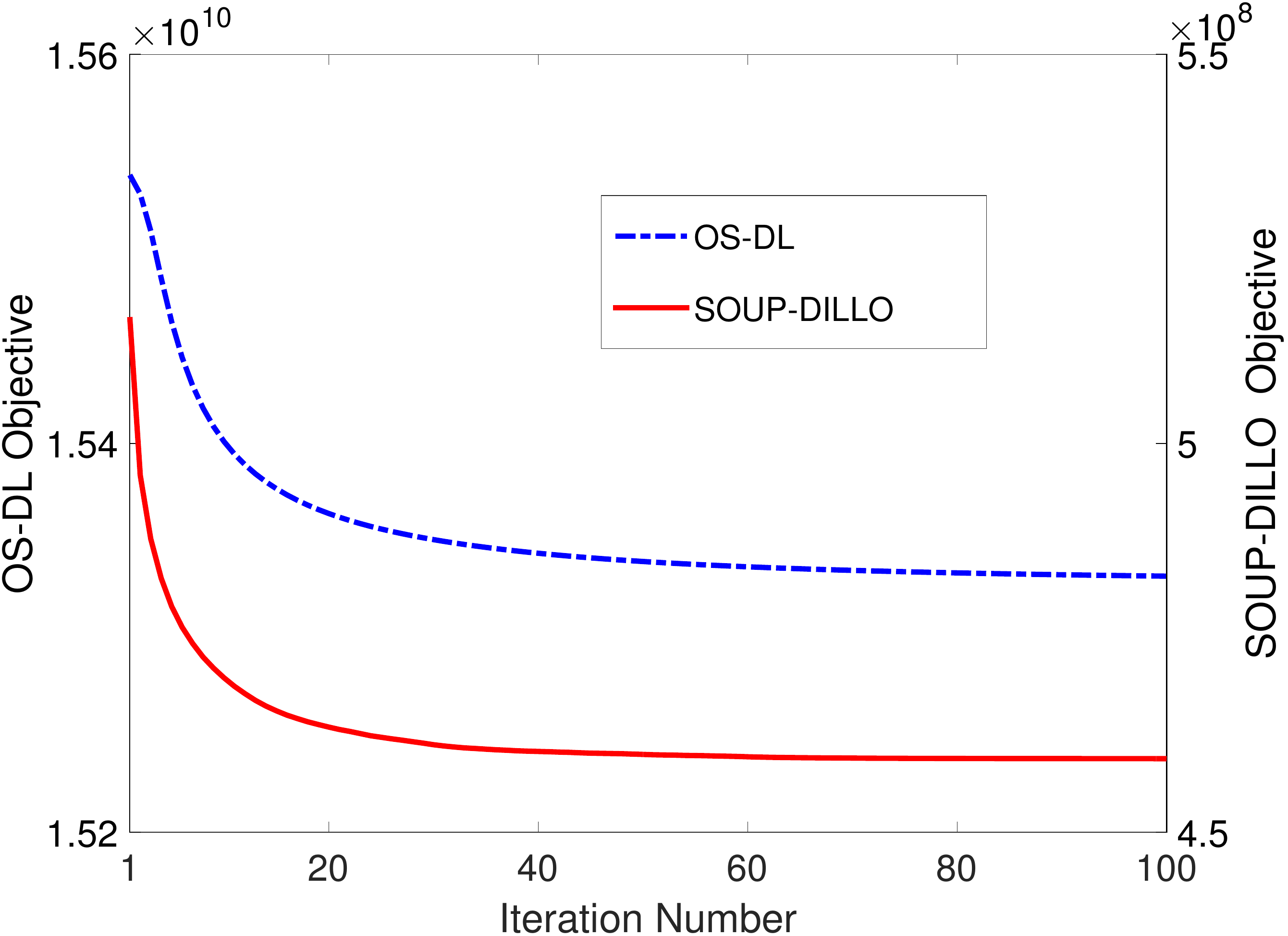}\\
 (a) \\
\includegraphics[height=1.43in]{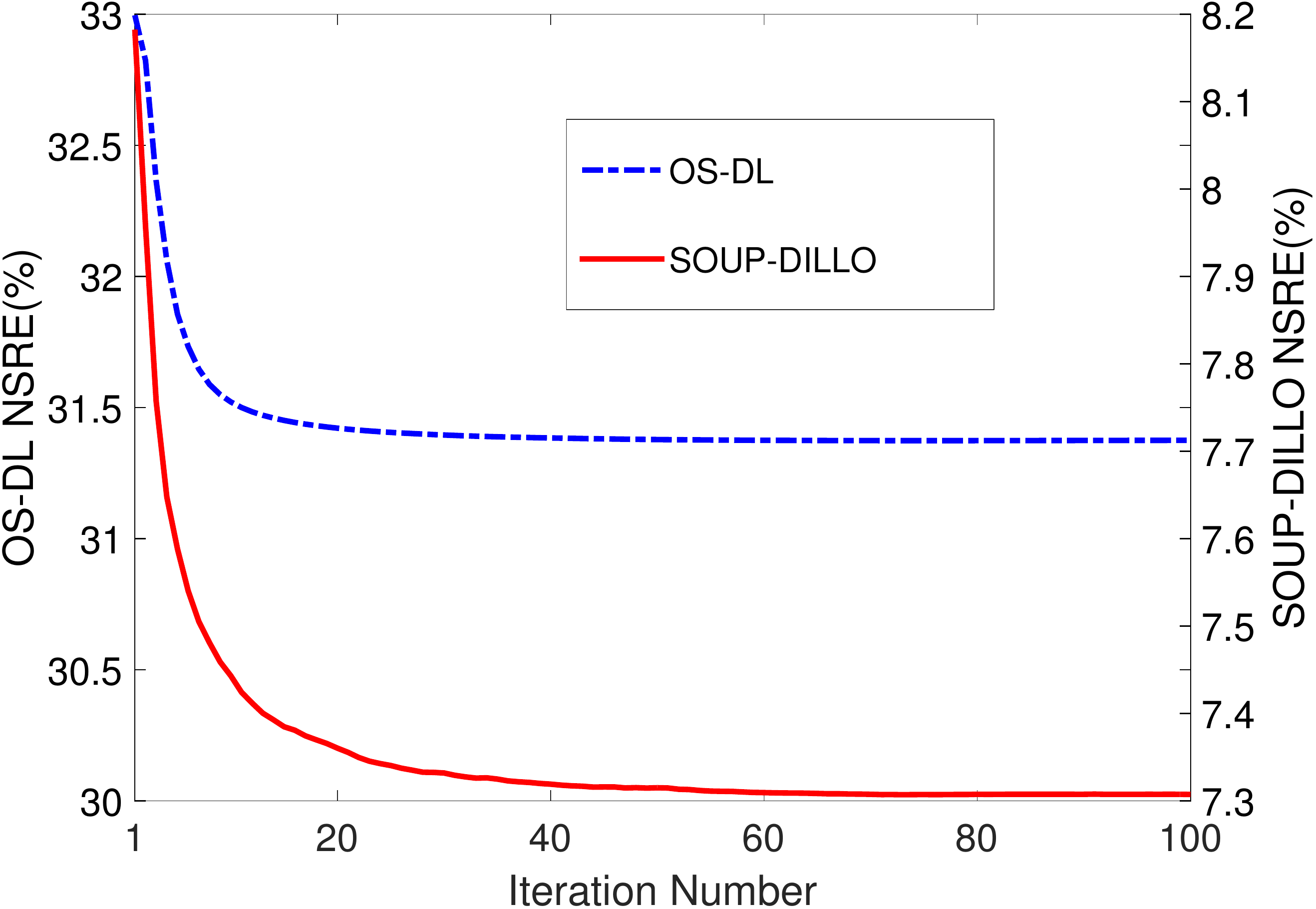}\\
(b) \\
\includegraphics[height=1.432in]{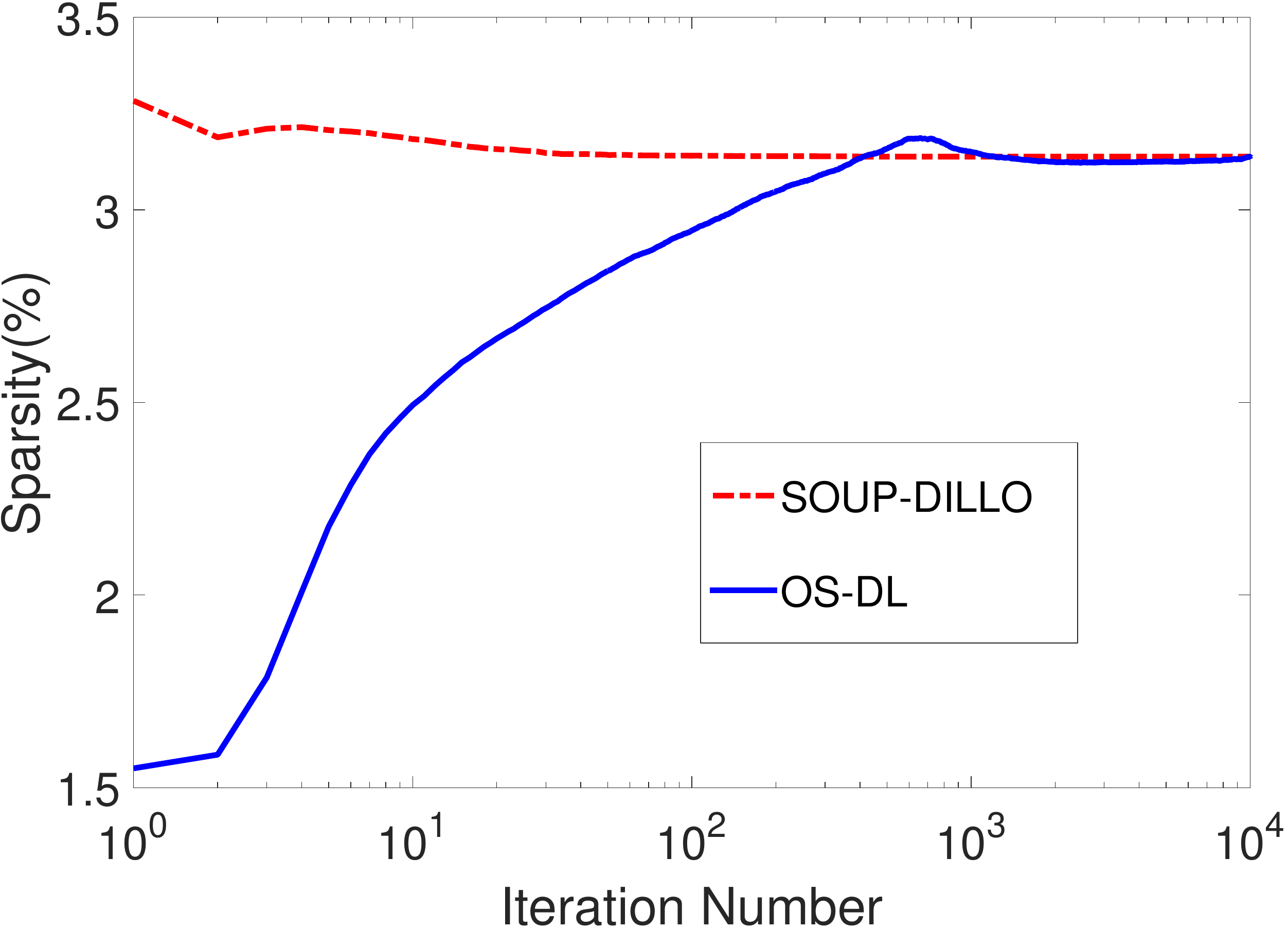}\\
 (c) \\
\includegraphics[height=1.42in]{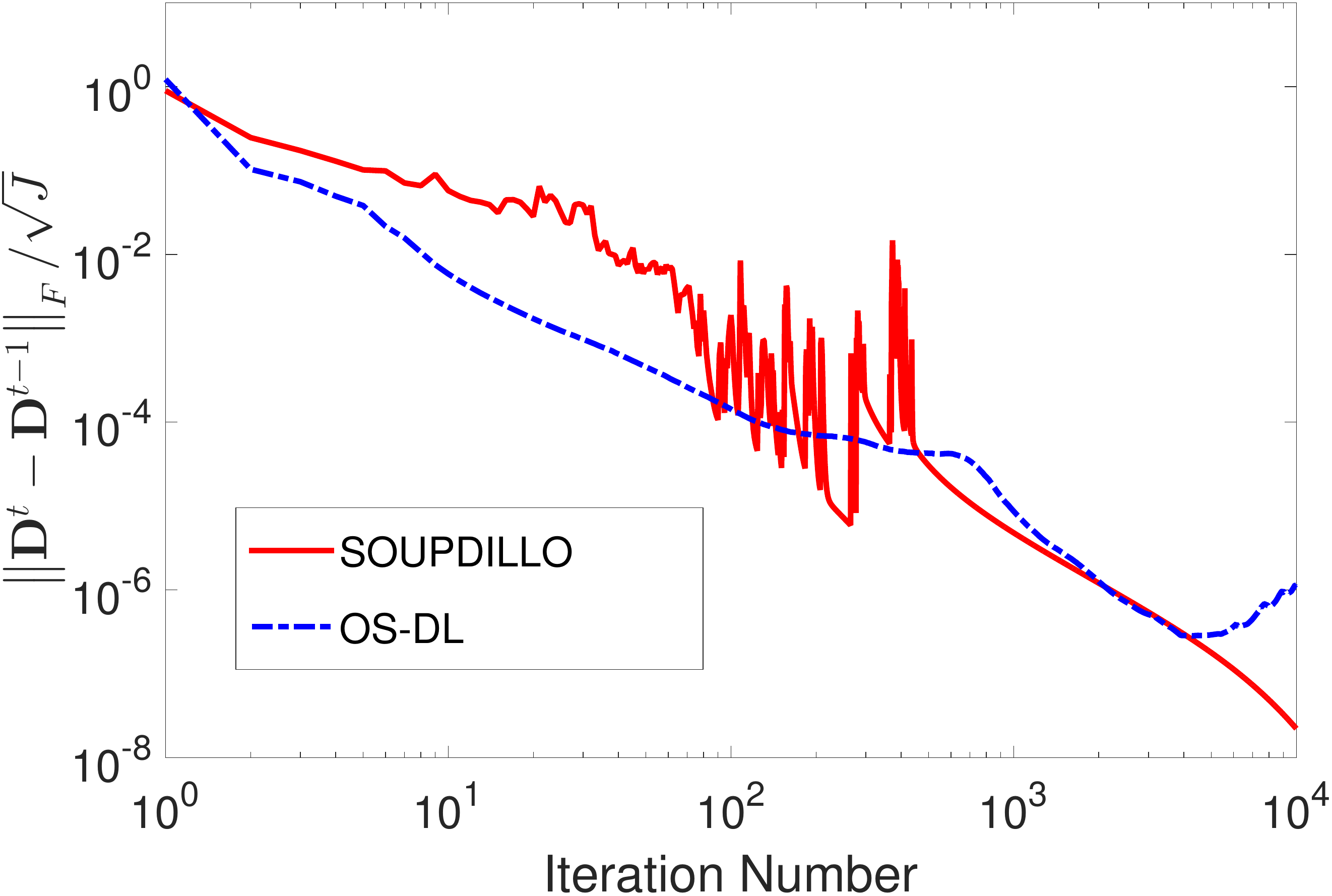}\\
(d) \\
\includegraphics[height=1.51in]{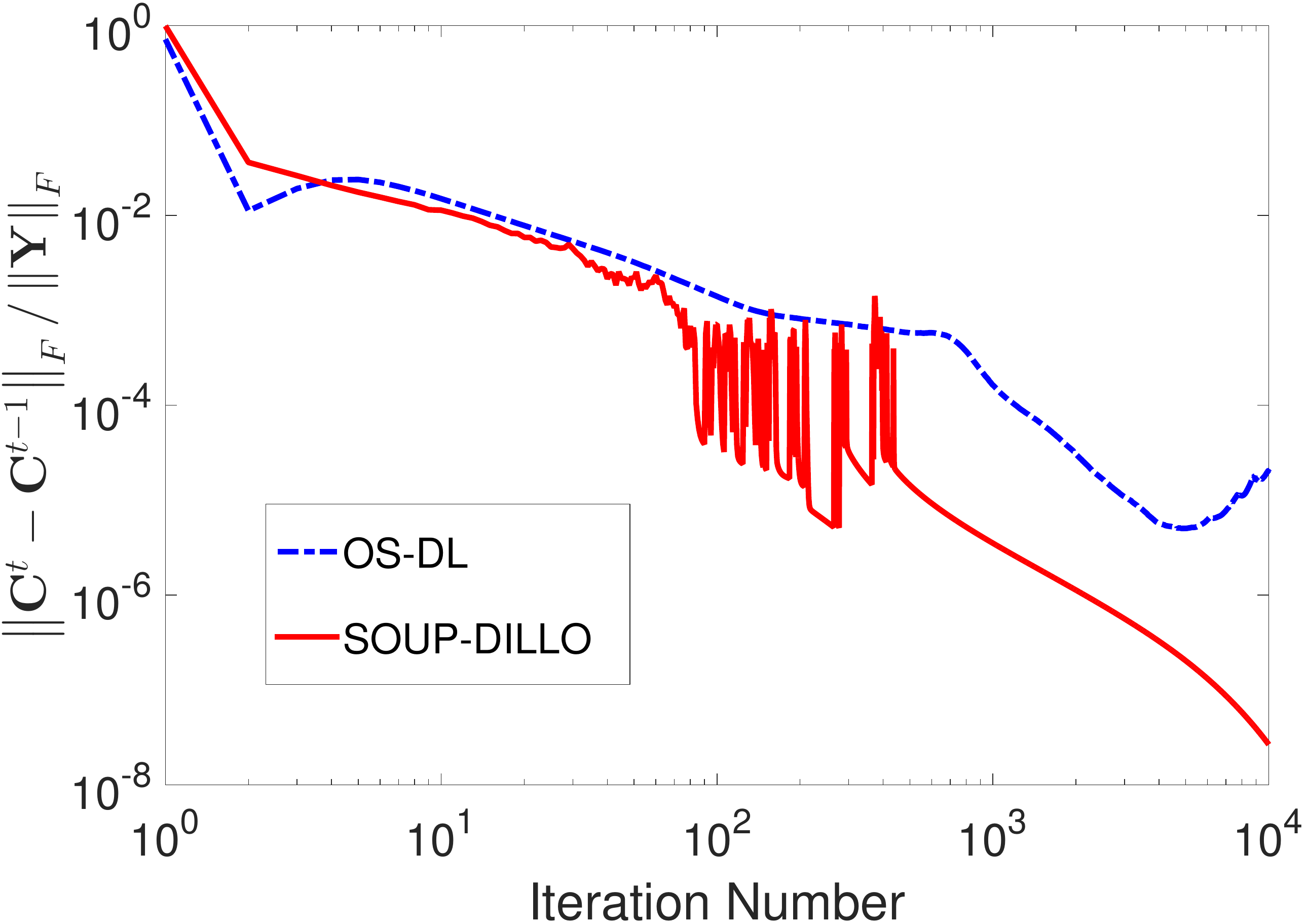}\\
(e)\\
\end{tabular}
\caption{Convergence behavior of the SOUP-DILLO and OS-DL Algorithms for (P1) and (P2): (a) Objective function; (b) Normalized sparse representation error (percentage); (c) sparsity factor of $\mathbf{C}$ ($\sum_{j=1}^{J}\left ( \left \| \mathbf{c}_{j} \right \|_{0} / nN \right )$ -- expressed as a percentage); (d) normalized changes between successive $\mathbf{D}$ iterates ($\left \| \mathbf{D}^{t} - \mathbf{D}^{t-1} \right \|_{F}/\sqrt{J}$); and (e) normalized changes between successive $\mathbf{C}$ iterates ($\left \| \mathbf{C}^{t} - \mathbf{C}^{t-1} \right \|_{F}/\left \| \mathbf{Y} \right \|_{F}$). Note that (a) and (b) have two vertical scales.}
\label{im2new}
\end{center}
\end{figure}

This section considers the same training data and learned dictionaries as in Section VI.B of our manuscript \citeSupp{sairakfes55:supp}. Recall that dictionaries of size $64 \times 256$ were learned for various choices of $\lambda$ in (P1) using the SOUP-DILLO and PADL  \citeSupp{bao1:supp, bao2:supp} methods, and dictionaries were also learned for various choices of $\mu$ in (P2) using the OS-DL \citeSupp{sadeg33:supp} method.
While Problems (P1) and (P2) are for aggregate sparsity penalized dictionary learning (using $\ell_{0}$ or $\ell_{1}$ penalties), here, we investigate the behavior of the learned dictionaries to sparse code the 
same 
data with sparsity constraints that are imposed in a conventional column-by-column (or signal-by-signal) manner (as in (P0) of  \citeSupp{sairakfes55:supp}).
We performed column-wise sparse coding using orthogonal matching pursuit (OMP) \citeSupp{pati:supp}. We used the OMP implementation in the K-SVD package \citeSupp{el2:supp} that uses a fixed cardinality (sparsity bound of $s$ per signal) and a negligible squared error threshold ($10^{-6}$).

For the comparisons here, dictionaries learned by SOUP-DILLO and OS-DL at a certain net sparsity factor in Fig. 4(e) of \citeSupp{sairakfes55:supp} were matched to corresponding column sparsities ($\left \| C \right \|_{0}/N$). For PADL, we picked the dictionaries learnt at identical $\lambda$ values as for SOUP-DILLO.
Fig. \ref{im35new}(a) shows the NSRE values for column-wise sparse coding (via OMP) using the dictionaries learned by PADL and SOUP-DILLO.
Dictionaries learned by SOUP-DILLO provide 0.7 dB better NSRE on the average (over the sparsities) than those learned by the alternating proximal scheme PADL.
At the same time, when the dictionaries learnt by (direct block coordinate descent methods) SOUP-DILLO and OS-DL were used to perform column-wise sparse coding, their performance was similar (a difference of only 0.07 dB in NSRE on average over sparsities).

Next, to compare formulations (P1) and (P0) in \citeSupp{sairakfes55:supp}, we learned dictionaries at the same column sparsities as in Fig. \ref{im35new}(a) using 30 iterations of K-SVD \citeSupp{elad:supp, el2:supp} (initialized the same way as the other methods). We first compare (Fig. \ref{im35new}(b)) the NSRE values obtained with K-SVD to those obtained with column-wise sparse coding using the learned dictionaries in SOUP-DILLO.
While the dictionaries learned by SOUP-DILLO provide significantly lower column-wise sparse coding residuals (errors) than K-SVD at low sparsities, K-SVD does better as the sparsity increases ($s \geq 3$), since it is adapted for the learning problem with column sparsity constraints.
Importantly, the NSRE values obtained by K-SVD are much higher (Fig. \ref{im35new}) than those obtained by SOUP-DILLO in the sparsity penalized dictionary learning Problem (P1), with the latter showing improvements of 14-15 dB at low sparsities and close to a decibel at some higher sparsities.
Thus, solving (P1) may offer potential benefits for adaptively representing data sets (e.g., patches of an image) using very few total non-zero coefficients.

\section{Dictionary-blind Image Reconstruction: Additional Results}

\begin{figure}[!t]
\begin{center}
\begin{tabular}{c}
\includegraphics[height=1.73in]{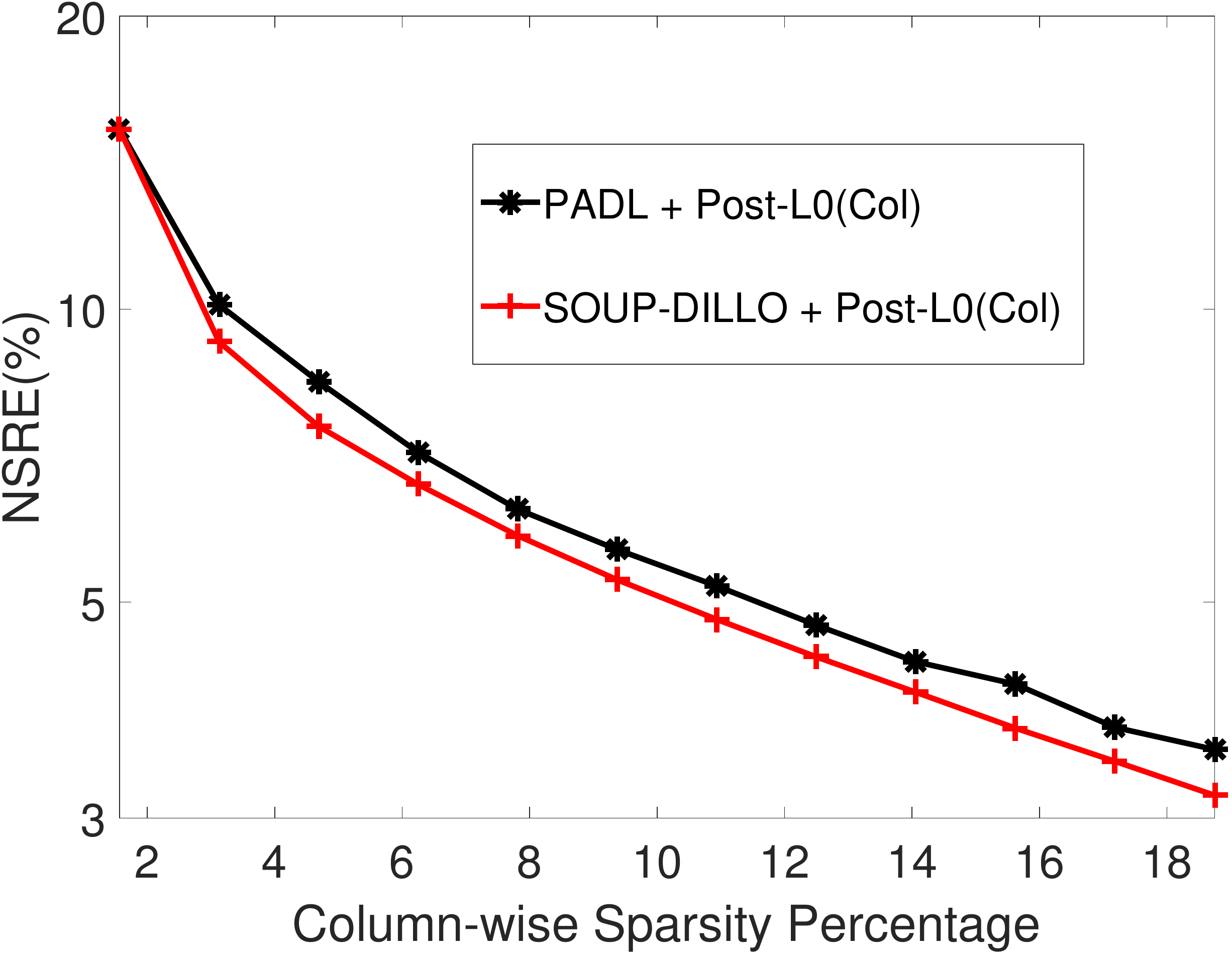}\\
(a) \\
\includegraphics[height=1.71in]{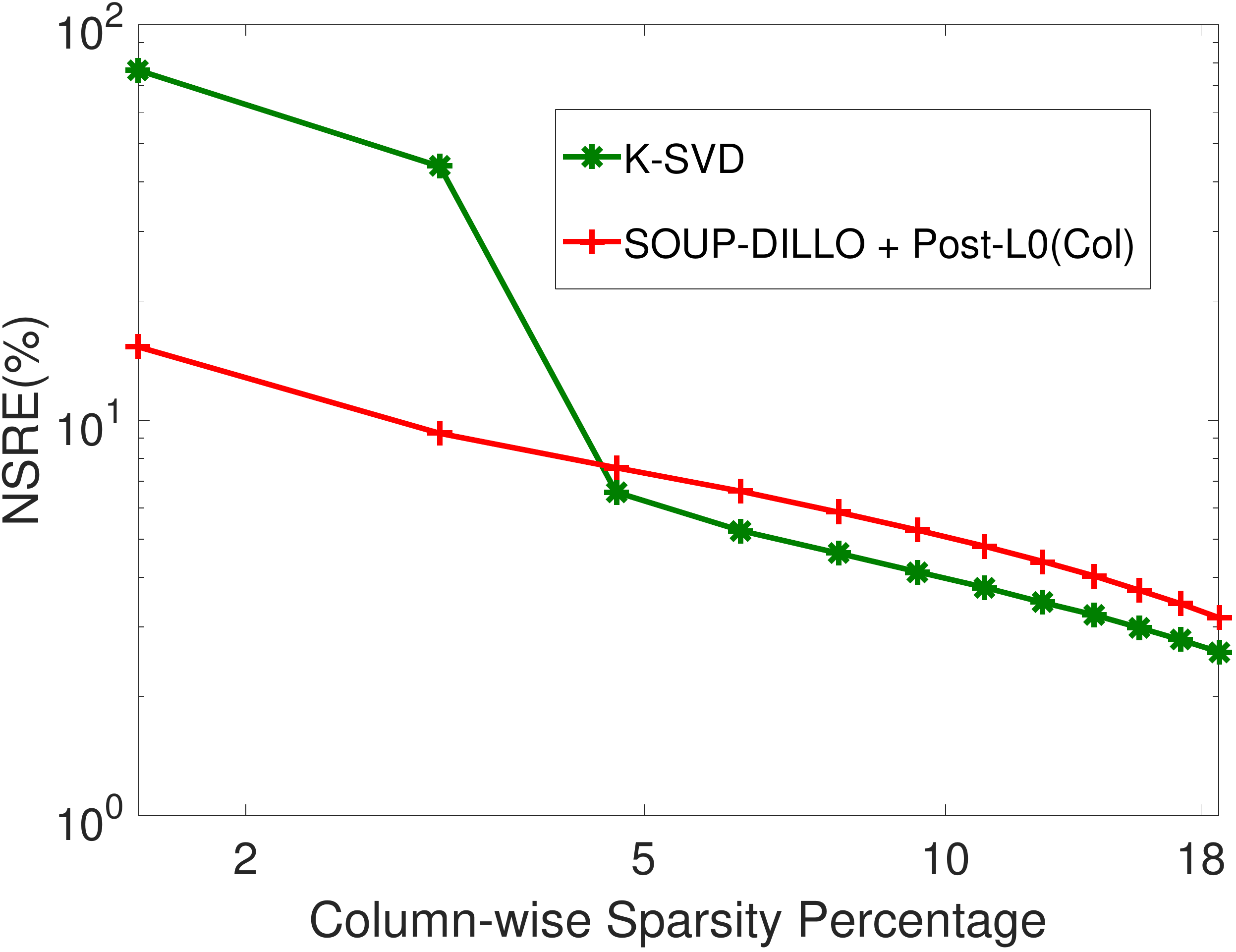}\\
(b) \\
\includegraphics[height=1.74in]{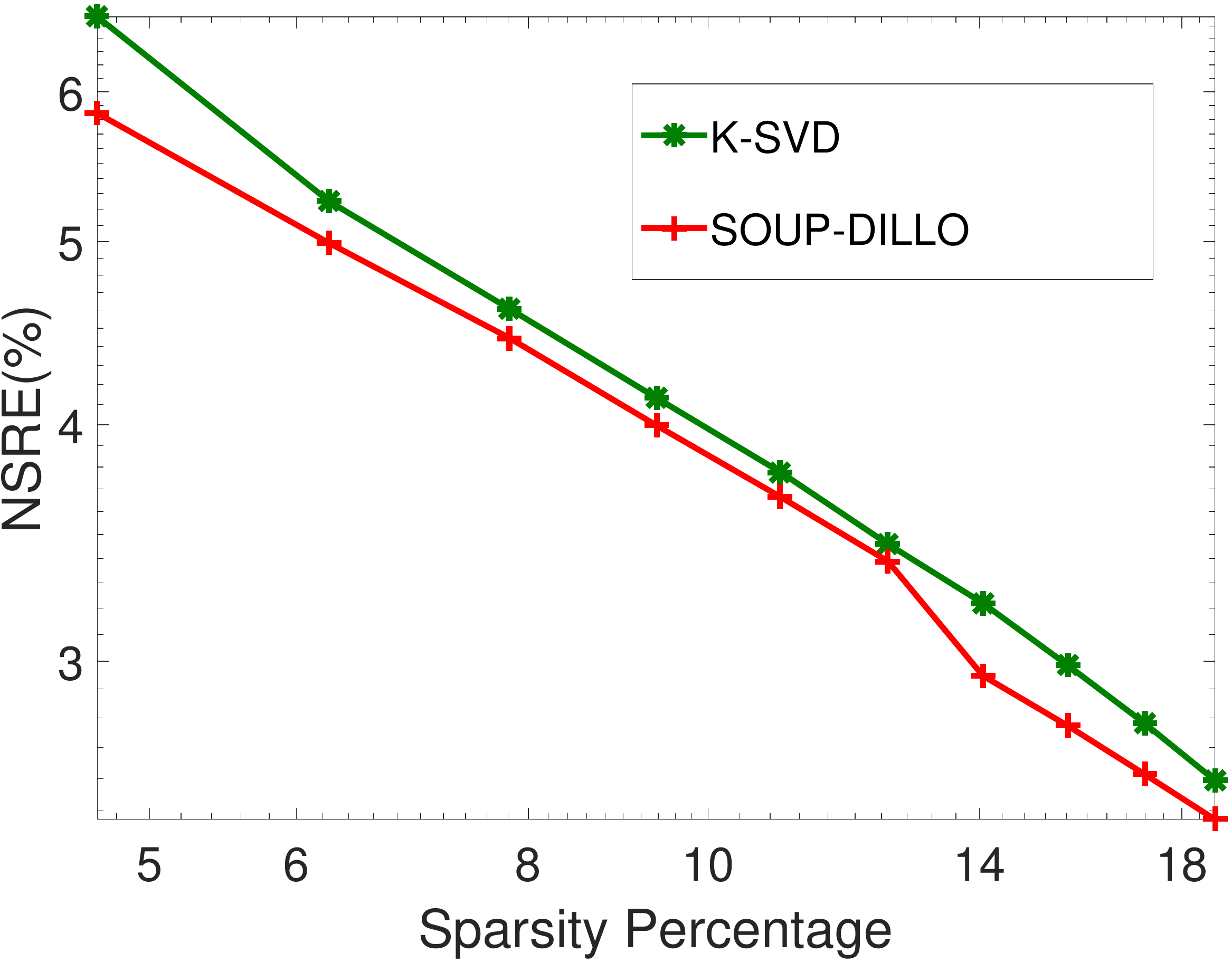} \\
(c)  \\
\end{tabular}
\caption{Comparison of dictionary learning methods for adaptive sparse representation (NSRE \protect\citeSupp{sairakfes55:supp} expressed as percentage): (a) NSRE values achieved with column-wise sparse coding (using orthogonal matching pursuit) using dictionaries learned with PADL \protect\citeSupp{bao1:supp, bao2:supp} and SOUP-DILLO; (b) NSRE values obtained using K-SVD \protect\citeSupp{elad:supp} compared to those with column-wise sparse coding (as in (a)) using learned SOUP-DILLO dictionaries; and (c) NSRE values obtained by SOUP-DILLO in Problem (P1) compared to those of K-SVD at matching net sparsity factors. The term `Post-L0(Col)' in the plot legends denotes (post) column-wise sparse coding using learned dictionaries.
}
\label{im35new}
\end{center}
\end{figure}

Fig. \ref{im15bcsbbnew} shows the reconstructions and reconstruction error maps (i.e., the magnitude of the difference between the magnitudes of the reconstructed and reference images) for several methods for an example in Table I of the manuscript \citeSupp{sairakfes55:supp}.
The reconstructed images and error maps for the $\ell_{0}$ ``norm''-based SOUP-DILLO MRI (with zero-filling initial reconstruction) show fewer artifacts or smaller distortions than for the K-SVD-based DLMRI \citeSupp{bresai:supp}, non-local patch similarity-based PANO \citeSupp{Qu2014843:supp}, or the $\ell_{1}$ norm-based SOUP-DILLI MRI.

\begin{figure}[!t]
\begin{center}
\begin{tabular}{cc}
\includegraphics[height=1.45in]{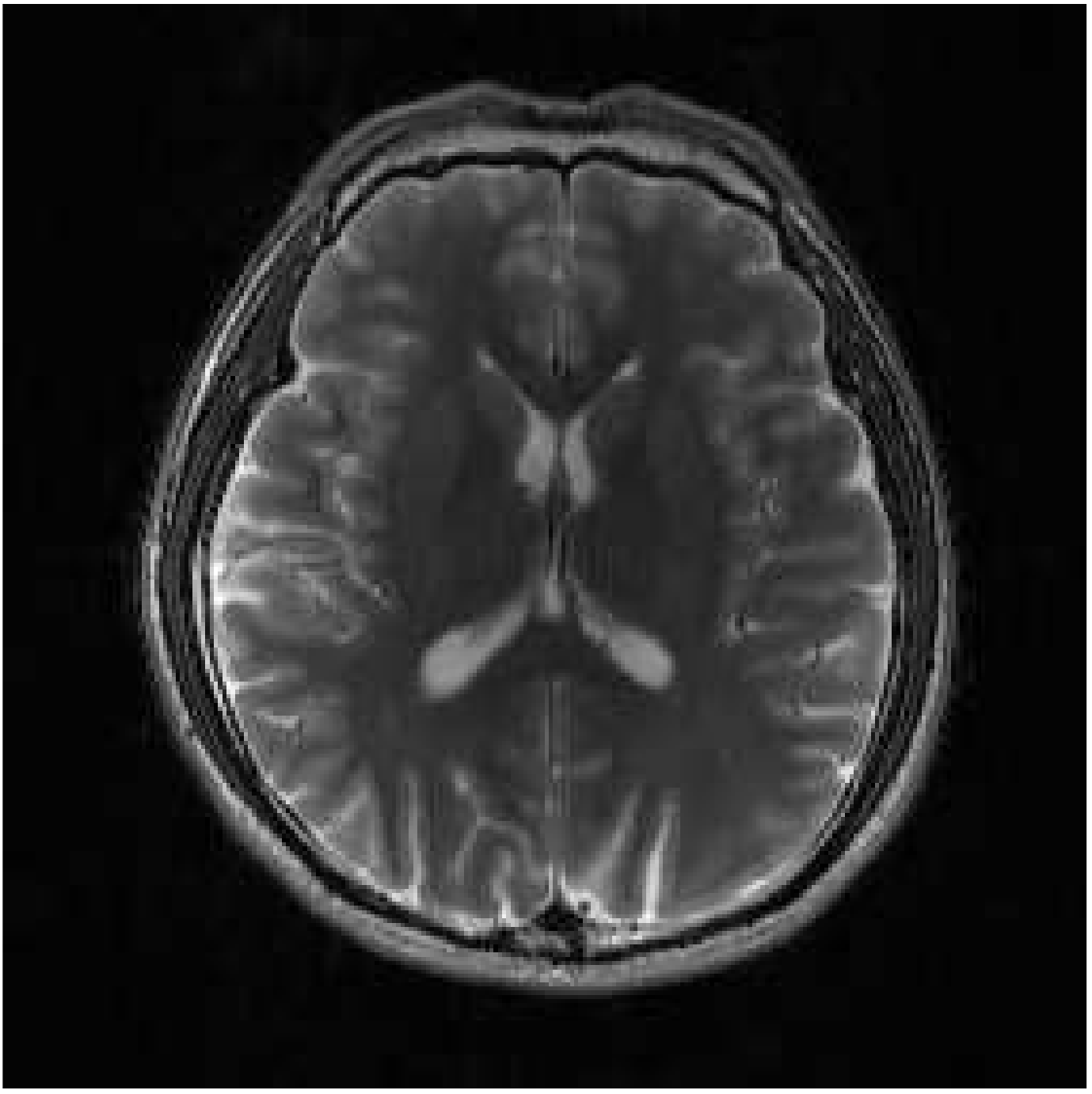}&
\includegraphics[height=1.47in]{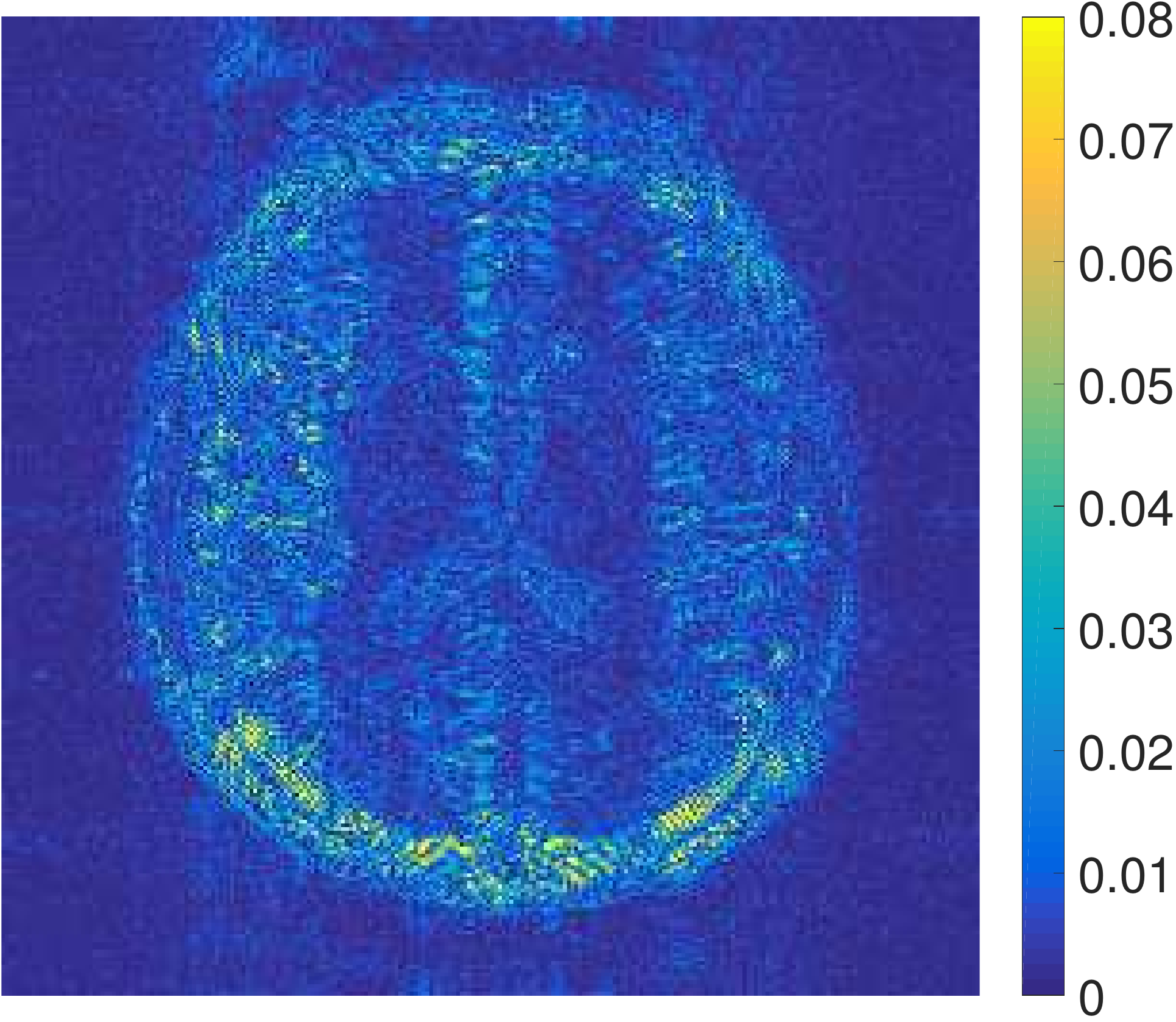}\\
(a) & (e) \\
\includegraphics[height=1.45in]{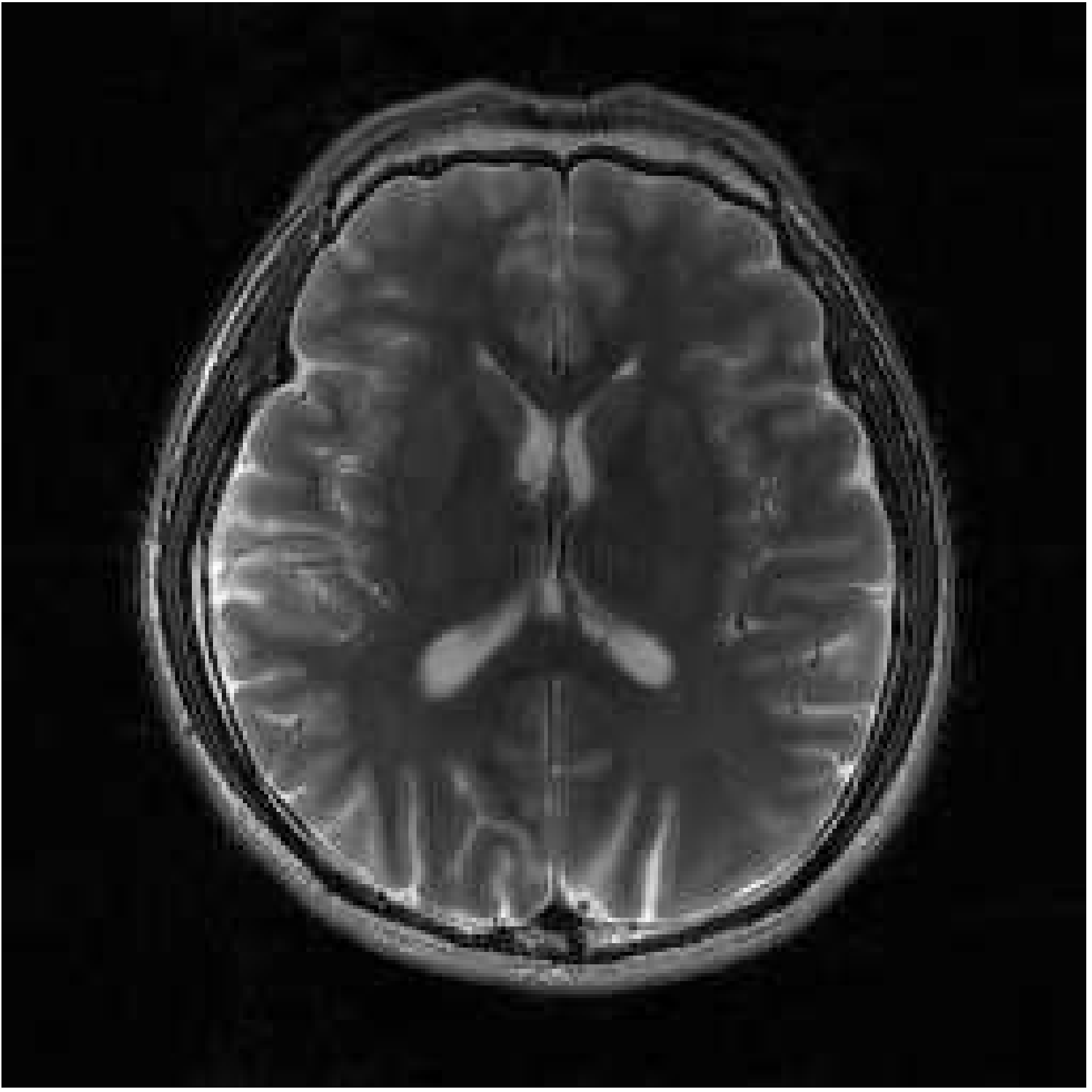}&
\includegraphics[height=1.47in]{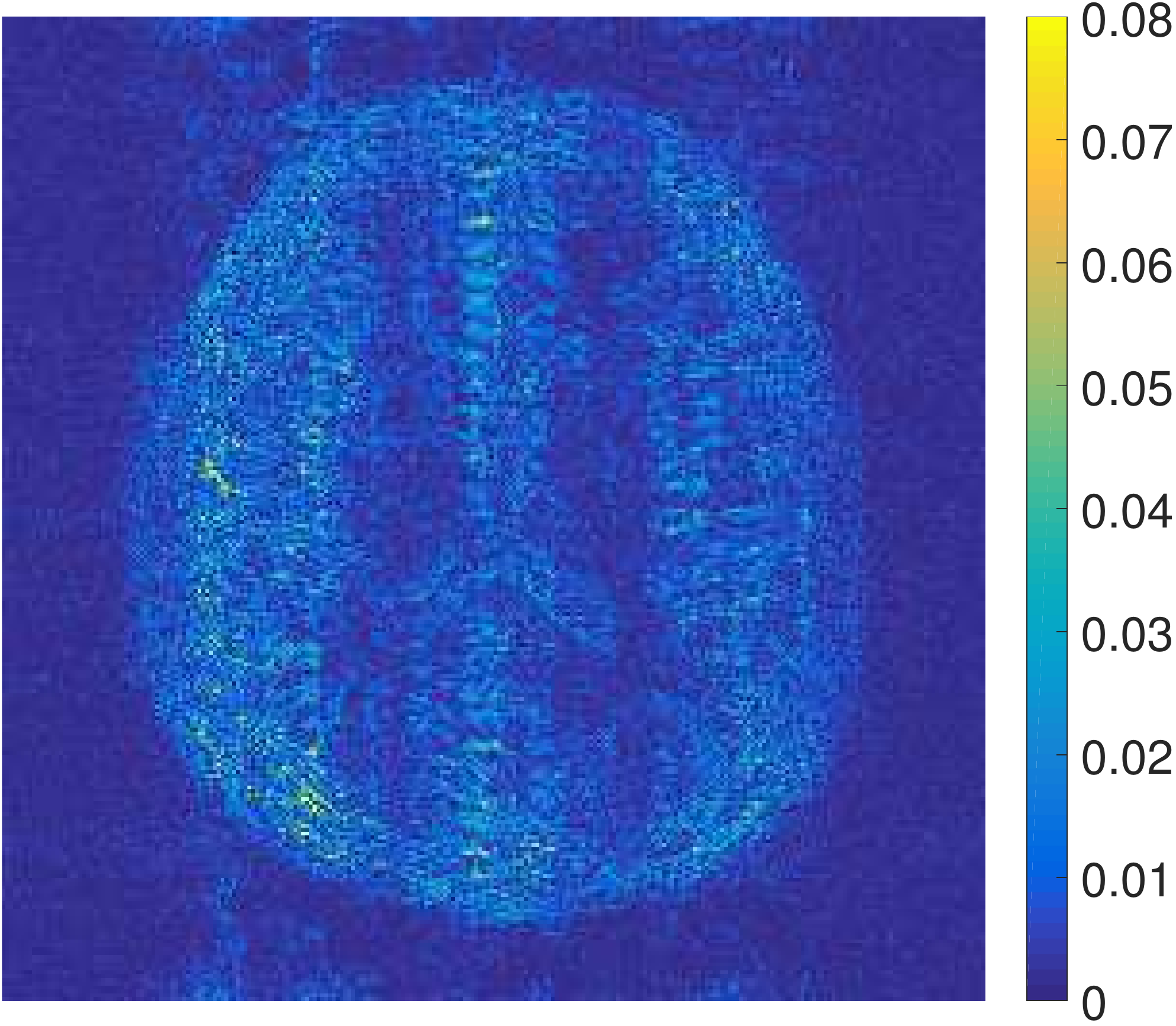}\\
(b) & (f) \\
\includegraphics[height=1.45in]{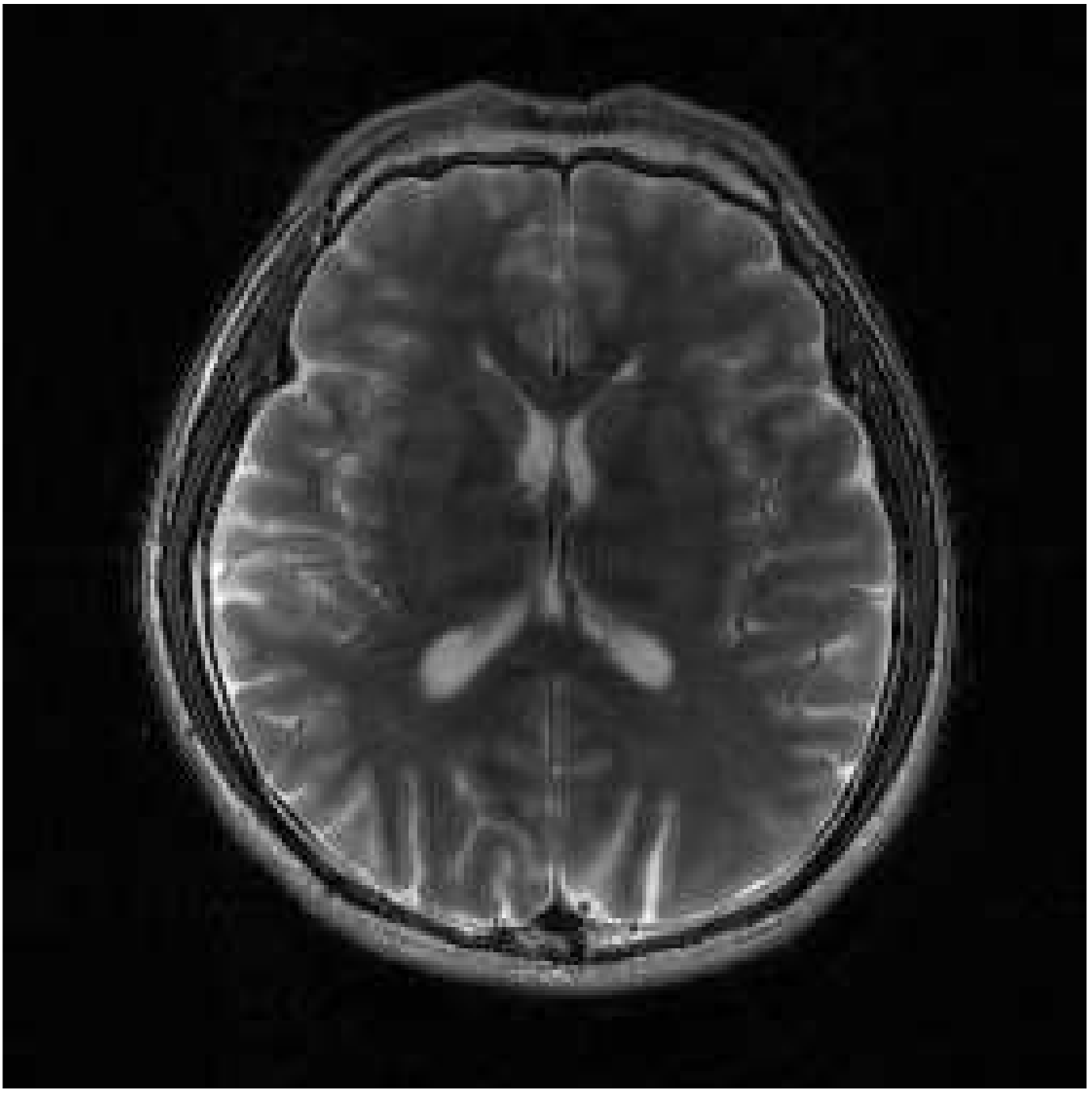}&
\includegraphics[height=1.47in]{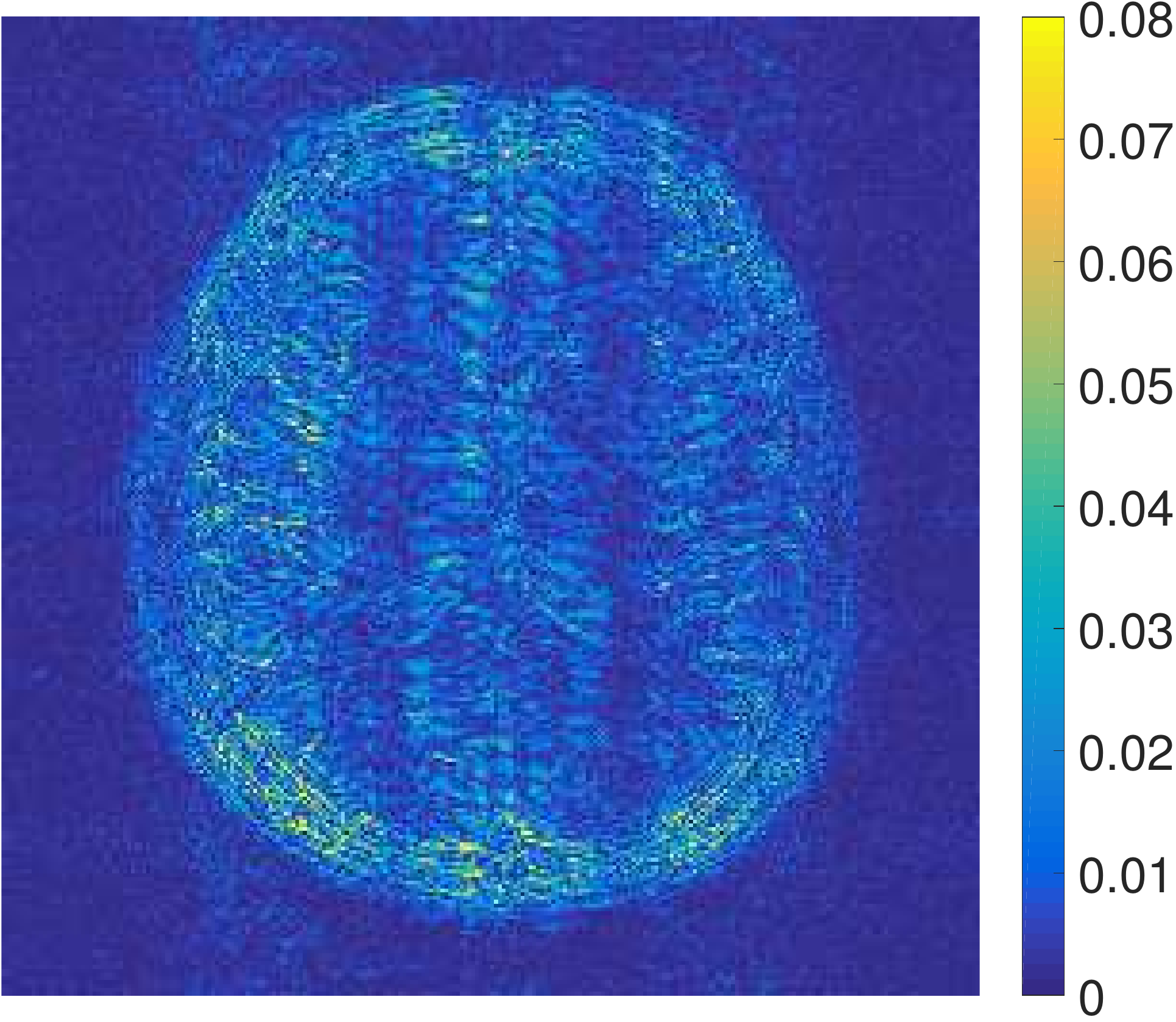}\\
 (c) & (g)\\
\includegraphics[height=1.45in]{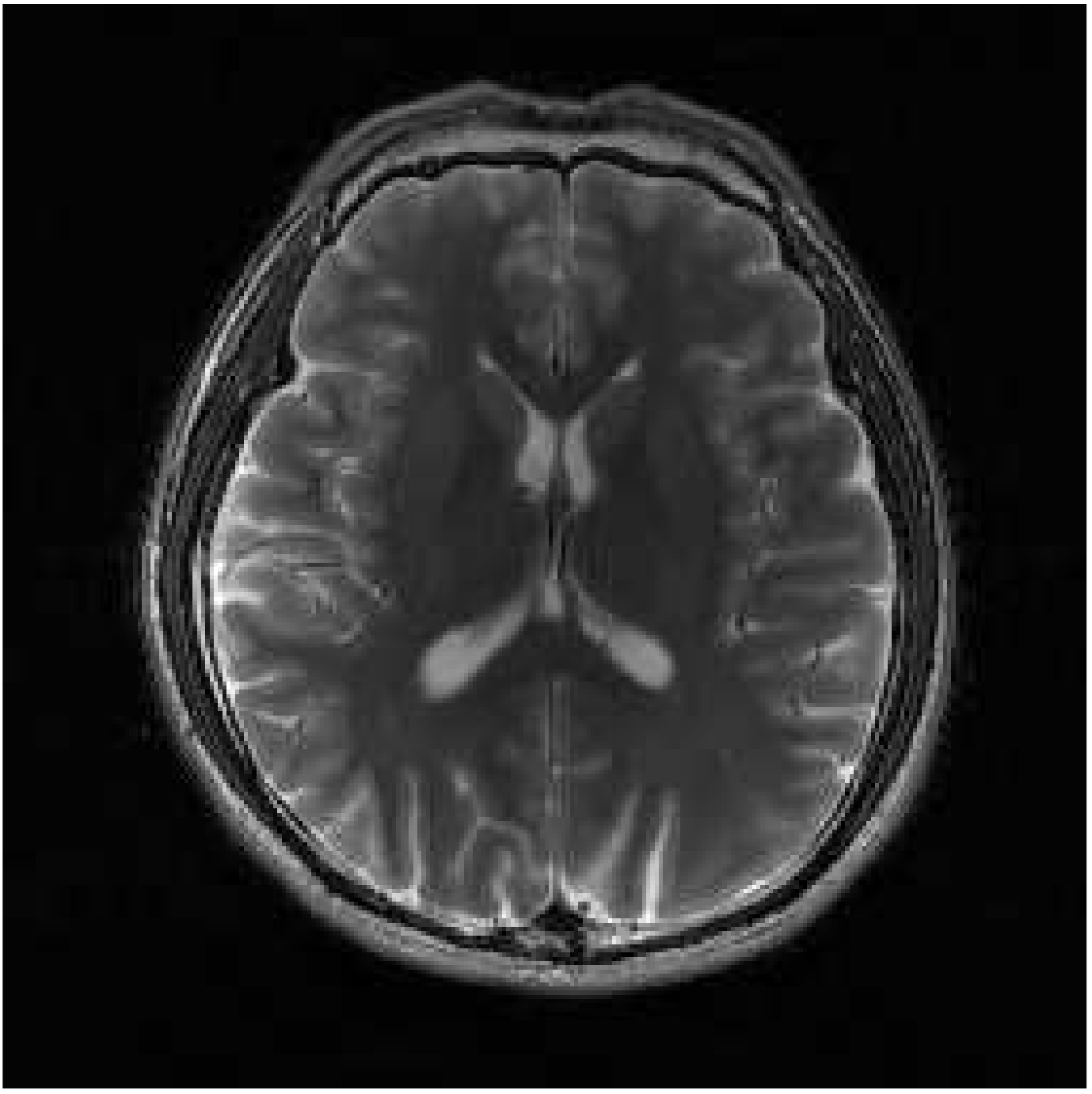}&
\includegraphics[height=1.47in]{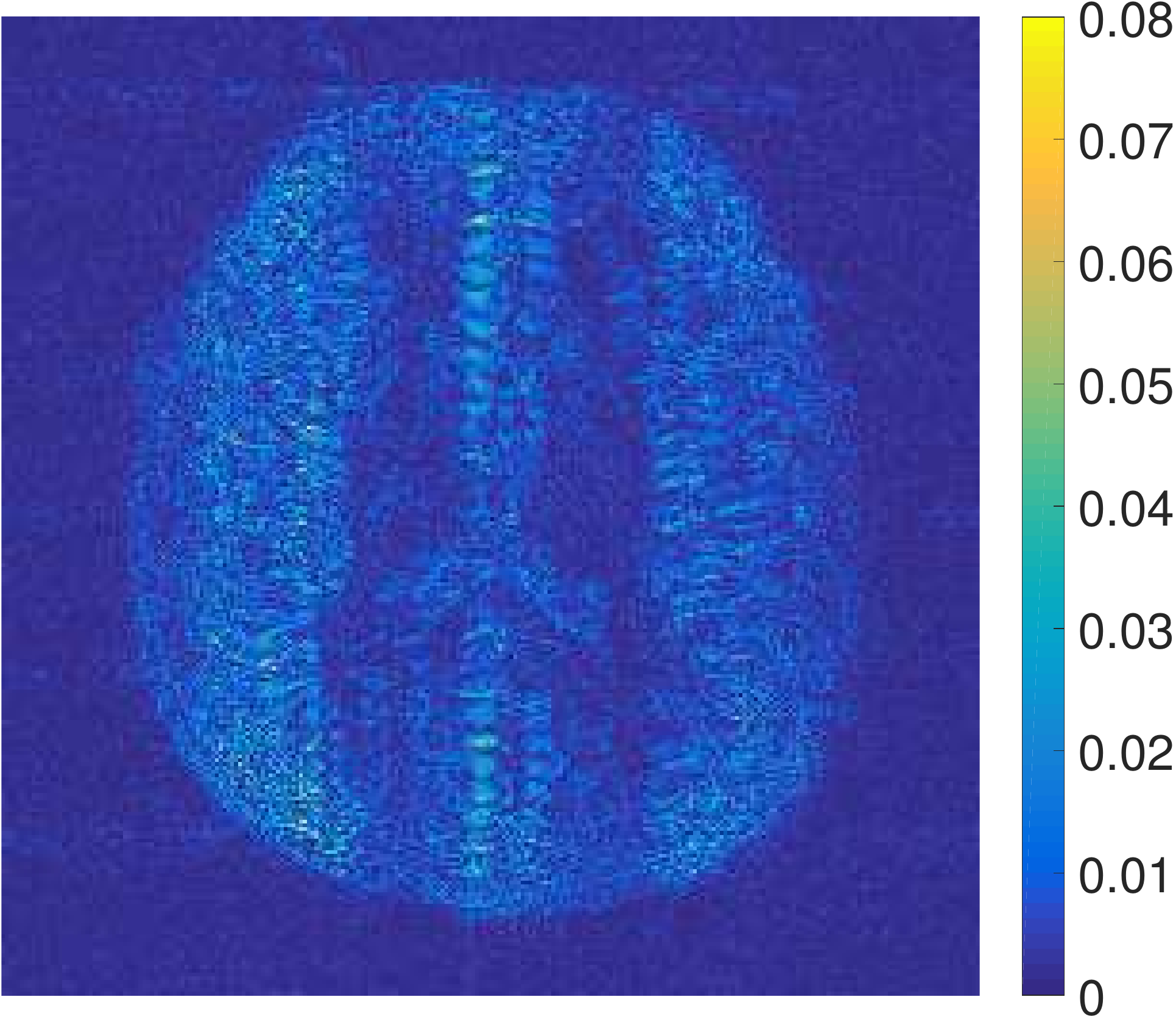}\\
(d) & (h) \\
\end{tabular}
\caption{Results for Image (e) in Fig. 3 (and Table I) of \protect\citeSupp{sairakfes55:supp}: Cartesian sampling with 2.5 fold undersampling. The sampling mask is shown in Fig. 5(a) of \protect\citeSupp{sairakfes55:supp}. Reconstructions (magnitudes): (a) DLMRI (PSNR = 38 dB) \protect\citeSupp{bresai:supp}; (b) PANO (PSNR = 40.0 dB) \protect\citeSupp{Qu2014843:supp}; (c) SOUP-DILLI MRI (PSNR = 37.9 dB); and (d) SOUP-DILLO MRI (PSNR = 41.5 dB). (e)-(h) are the reconstruction error maps for (a)-(d), respectively.}
\label{im15bcsbbnew}
\end{center}
\end{figure}

\appendices

\section{Sub-differential and Critical Points} \label{app56new}

\newtheorem{definition}{Definition}

This Appendix reviews the Fr\'{e}chet sub-differential of a function \citeSupp{vari1:supp, vari2:supp}. The norm and inner product notation in Definition~\ref{def1new} correspond to the euclidean $\ell_{2}$ settings.

\begin{definition} \label{def2new}
For a function $\phi: \mathbb{R}^{p} \mapsto (-\infty, + \infty]$, its domain is defined as $\mathrm{dom} \phi = \left \{ \mathbf{x} \in \mathbb{R}^{p} : \phi(\mathbf{x}) < + \infty \right \}$. Function $\phi$ is proper if $\mathrm{dom} \phi$ is nonempty.
\end{definition}

\begin{definition} \label{def1new}
Let $\phi: \mathbb{R}^{p} \mapsto (-\infty, + \infty]$ be a proper function and let $\mathbf{x} \in \mathrm{dom} \phi$. The Fr\'{e}chet sub-differential of the function $\phi$ at $\mathbf{x}$ is the following set denoted as $\hat{\partial }\phi(\mathbf{x}) $:
\begin{equation*}
 \begin{Bmatrix}
\mathbf{h} \in \mathbb{R}^{p} : \underset{\mathbf{b} \to \mathbf{x}, b \neq \mathbf{x}}{\lim \inf} \frac{1}{\left \| \mathbf{b}-\mathbf{x} \right \|}\left ( \phi(\mathbf{b}) - \phi(\mathbf{x}) - \left \langle \mathbf{b}-\mathbf{x}, \mathbf{h} \right \rangle \right ) \geq 0
\end{Bmatrix}.
\end{equation*}
If $\mathbf{x} \notin \mathrm{dom} \phi$, then $\hat{\partial }\phi(\mathbf{x}) \triangleq \emptyset$, the empty set.
The sub-differential of $\phi$ at $\mathbf{x}$ is the set $\partial \phi(\mathbf{x})$ defined as
\begin{equation*}
  \begin{Bmatrix}
\tilde{\mathbf{h}}  \in \mathbb{R}^{p} : \exists \mathbf{x}_{k} \to \mathbf{x}, \phi(\mathbf{x}_{k}) \to \phi(\mathbf{x}), \mathbf{h}_{k} \in \hat{\partial }\phi(\mathbf{x}_{k}) \to \tilde{\mathbf{h}} 
\end{Bmatrix}.
\end{equation*}
\end{definition}
A necessary condition for $\mathbf{x} \in \mathbb{R}^{p}$ to be a minimizer of the function $\phi$ is that $\mathbf{x}$ is a \emph{critical point} of $\phi$, i.e., $0 \in \partial \phi(\mathbf{x})$. 
Critical points are considered to be ``generalized stationary points" \citeSupp{vari1:supp}.